\documentclass[11pt]{article}

\usepackage{setspace}
\usepackage{cprotect}
\usepackage{fullpage}
\usepackage{color}
\usepackage{cite}
\usepackage{fixltx2e}
\usepackage[cmex10]{amsmath}
\usepackage{amssymb}
\usepackage{array}
\usepackage{wasysym}
\usepackage{dsfont}
\usepackage[latin1]{inputenc}                       
\usepackage[english]{babel}                         
\usepackage[T1]{fontenc}
\usepackage{mathtools}
\usepackage{amssymb} 
\usepackage{enumerate}
\usepackage{bbm}
\usepackage{epsfig,syntonly}
\usepackage{verbatim,times}
\usepackage{epstopdf}
\usepackage{graphicx}
\usepackage{latexsym,fancyhdr,bm}
\usepackage{subfig}
\usepackage{caption}
\usepackage{url}
\usepackage{bbm}
\usepackage{dsfont}
\usepackage{amsthm}
\usepackage{bigints}

\newcommand{\norm}[1]{\left\lVert#1\right\rVert}

\newtheoremstyle{custom}
{} 
{} 
{} 
{} 
{\bfseries} 
{:} 
{.25em} 
{} 
\theoremstyle{plain}

\newtheorem{theorem}{Theorem}
\newtheorem{lemma}{Lemma}

\newtheorem{remark}{Remark}
\newtheorem{corollary}{Corollary}

\newtheorem*{theorem*}{Theorem}
\newtheorem*{lemma*}{Lemma}
\newtheorem*{proposition*}{Proposition}
\newtheorem*{definition*}{Definition}
\newtheorem*{example*}{Example}
\newtheorem*{remark*}{Remark}
\newtheorem*{corollary*}{Corollary}

\def\bz{{\boldsymbol z}}
\def\sT{{\mathsf T}}

\def\det{{\rm det}}

\def\complex{{\mathbb C}}
\def\reals{{\mathbb R}}

\def\<{\langle}
\def\>{\rangle}

\def\reals{{\mathbb R}}

\def\baq{\bar{q}}
\def\ba{{\boldsymbol a}}
\def\bq{{\boldsymbol q}}
\def\be{{\boldsymbol e}}
\def\bv{{\boldsymbol v}}
\def\bx{{\boldsymbol x}}
\def\bw{{\boldsymbol w}}
\def\by{{\boldsymbol y}}
\def\bd{{\boldsymbol d}}
\def\bs{{\boldsymbol s}}
\def\br{{\boldsymbol r}}
\def\hbs{\hat{\boldsymbol s}}
\def\hbr{\hat{\boldsymbol r}}
\def\cF{{\mathcal F}}
\def\cH{{\mathcal H}}
\def\hh{\hat{h}}
\def\sB{{\sf B}}
\def\hsB{\hat{\sf B}}
\def\bM{{\boldsymbol M}}

\def\bH{{\boldsymbol H}}
\def\bA{{\boldsymbol A}}
\def\bB{{\boldsymbol B}}
\def\bC{{\boldsymbol C}}
\def\bY{{\boldsymbol Y}}
\def\bX{{\boldsymbol X}}

\def\bZ{{\boldsymbol Z}}
\def\bI{{\boldsymbol I}}
\def\bU{{\boldsymbol U}}
\def\bSigma{{\boldsymbol \Sigma}}
\def\bV{{\boldsymbol V}}
\def\bD{{\boldsymbol D}}
\def\bP{{\boldsymbol P}}

\def\bJ{{\boldsymbol J}}
\def\bL{{\boldsymbol L}}
\def\b0{{\boldsymbol 0}}
\def\bzero{{\boldsymbol 0}}
\def\bS{{\boldsymbol S}}

\def\bu{{\boldsymbol u}}
\def\bg{{\boldsymbol g}}
\def\bv{{\boldsymbol v}}

\def\sF{{\sf F}}
\def\sG{{\sf G}}
\def\bz{{\boldsymbol z}}
\def\hbz{\hat{\boldsymbol z}}

\def\hz{\hat{z}}

\def\sb{{\sf b}}

\def\normal{{\sf N}}
\def\cnormal{{\sf CN}}
\def\id{{\boldsymbol I}}
\def\E{{\mathbb E}}

\def\de{{\rm d}}

\def\Sphere{{\sf S}}
\def\MMSE{{\sf MMSE}}
\def\Unif{{\rm Unif}}

\def\complex{{\mathbb C}}
\def\cT{{\mathcal T}}

\makeatletter
\let\l@ENGLISH\l@english
\makeatother

\title{Fundamental Limits of Weak Recovery\\ with Applications to Phase Retrieval}

\author{Marco~Mondelli\thanks{Department of Electrical Engineering,
    Stanford University}
 \;\;\; and\;\;\; Andrea Montanari\thanks{Department of Electrical Engineering and Department of Statistics, Stanford University}}

\begin{document}

\maketitle
\begin{abstract}
\noindent In phase retrieval we want to recover an unknown signal $\bx\in\complex^d$ from $n$ quadratic measurements of the form 
$y_i = |\<\ba_i,\bx\>|^2+w_i$ where $\ba_i\in \complex^d$ are known sensing vectors and  $w_i$ is measurement noise. We ask the following
\emph{weak recovery} question: what is the minimum number of measurements $n$ needed to produce an estimator $\hat{\bx}(\by)$ that is positively correlated with 
the signal $\bx$?
We consider the case of Gaussian  vectors $\ba_i$. We prove that -- in the high-dimensional limit -- a sharp phase transition takes place, and we locate the threshold in the regime of vanishingly small noise.
For $n\le d-o(d)$ no estimator can do significantly better than random and achieve a strictly positive correlation.
For $n\ge d+o(d)$ a simple spectral estimator achieves a positive correlation. 
Surprisingly, numerical simulations with the same spectral estimator  demonstrate promising performance with realistic sensing matrices. 
Spectral methods are used to initialize non-convex optimization algorithms in phase retrieval, and our approach can boost the performance in this setting as well.

Our impossibility result is based on classical information-theory arguments. The spectral algorithm computes the leading eigenvector
of a weighted empirical covariance matrix. We obtain a sharp
characterization of the spectral properties of this random matrix  using tools from free probability and generalizing
a recent result by Lu and Li. 
Both the upper and lower bound generalize beyond phase retrieval to measurements $y_i$ produced according to a generalized linear model. 
As a byproduct of our analysis, we compare the threshold of the proposed spectral method with that of a message passing algorithm. 
\end{abstract}


\section{Introduction} \label{sec:intro}


In this work, we consider the problem of recovering a signal $\bx$ of dimension $d$, given $n$ \emph{generalized linear measurements}. More specifically, the measurements are drawn independently according to the conditional distribution
\begin{equation}\label{eq:model}
y_i\sim p(y\mid \langle\bx, \ba_i \rangle), \qquad i\in\{1, \ldots, n\},
\end{equation}
where $\langle\cdot, \cdot\rangle$ denotes the inner product, $\{\ba_i\}_{1\le i\le n}$ is a set of known sensing vector, and $p(\cdot \mid \langle\bx, \ba_i \rangle)$ is a known probability density function. This model appears in many problems in signal processing and statistical estimation, e.g., photon-limited imaging \cite{photonUnser, photonVetterli}, signal recovery from quantized measurements \cite{noiseRangan}, and phase retrieval \cite{phFienup, shechtman2015phase}. For the problem of \emph{phase retrieval}, the model \eqref{eq:model} is specialized to 
\begin{equation}\label{eq:phretr}
y_i = |\langle\bx, \ba_i\rangle|^2+w_i, \qquad i\in\{1, \ldots, n\}\, ,
\end{equation}
where $w_i$ is noise. 
Applications of phase retrieval arise in several areas of science and engineering, including X-ray crystallography \cite{M90, H93}, microscopy \cite{miao2008extending}, astronomy \cite{fienup1987phase}, optics \cite{walther1963question}, acoustics\cite{balan2006signal}, interferometry \cite{demanet2017convex}, and quantum mechanics \cite{corbett2006pauli}.
 
Popular methods to solve the phase retrieval problem are based on semi-definite programming relaxations \cite{candes2015phase,
  candes2015phase2, candes2013phaselift,  waldspurger2015phase}. However, these algorithms rapidly become prohibitive from a computational point of view when the dimension $d$
of the signal increases, which makes them impractical in most of the real-world applications. For this reason, several algorithms have been
developed in order to solve directly the non-convex least-squares problem, including the error reduction schemes dating back to
Gerchberg-Saxton and Fienup \cite{gerchberg1972practical, phFienup}, alternating minimization \cite{netrapalli2013phase}, 
approximate message passing (AMP) \cite{schniter2015compressive}, Wirtinger Flow \cite{candes2015wirt}, iterative projections \cite{li2015phase}, the Kaczmarz method
\cite{wei2015solving}, and a number of other approaches \cite{chen2017solving, zhang2016reshaped, cai2016optimal, wang2016solving, wang2016solvingnips, soltanolkotabi2017structured,duchi2017solving, wang2017solving}. Furthermore, recently a convex relaxation that operates in the natural domain of the signal was independently proposed by two groups of authors \cite{goldstein2016phasemax, bahmani17a}. All these techniques require an initialization step, whose goal is to provide a solution $\hat{\bx}$ that is positively correlated with the unknown signal $\bx$. To do so, spectral methods are widely employed: the estimate $\hat{\bx}$ is given by the principal eigenvector of a suitable matrix constructed from the data. 
A similar stategy (initialization step followed by an iterative algorithm) has proved successful for many other estimation problems, e.g., matrix completion \cite{keshavan2010matrix, jain2013low}, blind deconvolution \cite{lee2017blind, li2016rapid}, sparse coding \cite{arora2015simple} and joint alignment from pairwise noisy observations \cite{chen2016projected}. 

We focus on a regime in which both the number of measurement $n$ and the dimension of the signal $d$ tend to infinity, but their ratio $n/d$ 
tends to a positive constant $\delta$. The \emph{weak recovery} problem requires to provide an estimate $\hat{\bx}(\by)$ that has a positive correlation with the unknown vector $\bx$:
\begin{equation}\label{eq:poscorr}
\liminf_{n\to\infty}\E\bigg\{\frac{|\langle\hat{\bx}(\by), \bx\rangle|}{\norm{\hat{\bx}(\by)}_2 \norm{\bx}_2} \bigg\}> \epsilon,
\end{equation}
for some $\epsilon >0$. 

In this paper, we consider either $\bx\in \reals^d$ or $\bx\in\complex^d$ and assume that the measurement vectors
$\ba_i$ are standard Gaussian (either real or complex). In the general setting of model (\ref{eq:model}), we present two types of results:
\begin{enumerate}
\item We develop an \emph{information-theoretic lower bound} $\delta_{\ell}$: for $\delta < \delta_{\ell}$, no estimator can output non-trivial estimates.
In other words, the weak recovery problem cannot be solved. 
\item We establish an  \emph{upper bound} $\delta_{\rm u}$ based on a \emph{spectral algorithm}:
for $\delta > \delta_{\rm u}$, we can achieve weak recovery (see~\eqref{eq:poscorr}) by letting $\hat{\bx}$ be the principal
eigenvector of a matrix suitably constructed from the data. We also show that $\delta_{\rm u}$ is the optimal threshold for spectral methods.
\end{enumerate}
The values of the thresholds $\delta_{\ell}$ and $\delta_{\rm u}$ depend on the conditional distribution $p(\cdot \mid \langle\bx, \ba_i \rangle)$. 
For the special case of phase retrieval (see \eqref{eq:phretr}), we evaluate these bounds and we show that they coincide
in the limit of vanishing noise.
\begin{theorem*} 
Let $\bx$ be uniformly distributed on the $d$-dimensional complex sphere with radius $\sqrt{d}$ and assume that $\{\ba_i\}_{1\le i\le n}\sim_{i.i.d.}\cnormal(\b0_d,\id_d/d)$.
Let $\by\in \mathbb R^n$ be given by \eqref{eq:phretr}, with $\{w_i\}_{1\le i\le n}\sim\normal(0,\sigma^2)$, and $n, d\to \infty$ with $n/d\to \delta \in (0, +\infty)$. Then,
\begin{itemize}
\item For $\delta <1$, no algorithm can provide non-trivial estimates on $\bx$;

\item For $\delta > 1$, there exists $\sigma_0(\delta)>0$ and  a spectral algorithm that returns an estimate $\hat{\bx}$ satisfying \eqref{eq:poscorr}, for any $\sigma\in [0,\sigma_0(\delta)]$. 
\end{itemize}
\end{theorem*}
The assumption that $\bx$ is uniform on the sphere can be dropped for the upper bound part. We also show that $\sigma_0(\delta)$ scales as $\sqrt{\delta-1}$ when $\delta$ is close to $1$. In the `real case' $\bx \in \mathbb R^d$ with $\norm{\bx}_2^2=d$ and $\{\ba_i\}_{1\le i\le n}\sim_{i.i.d.}\normal(\b0_d,\id_d/d)$, we prove that an analogous result holds and that the threshold moves from $1$ to $1/2$. This is reminiscent of how the injectivity thresholds
are $\delta=4$ and $\delta=2$ in the complex and the real case, respectively \cite{balan2006signal, bandeira2014saving, conca2015algebraic}. A possible intuition for this halving phenomenon comes from the fact that the complex problem has twice as many variables but the same amount of equations of the real problem. Hence, it is reasonable that the complex case requires twice the amount of data with respect to the real case.

Let us emphasize that we are considering the problem of weak recovery. Therefore, we may need less than $n$ samples in order to obtain positive correlation on $n$ unknowns.
For instance, in the linear case $y_i = \<\ba_i,\bx\>+w_i$, weak recovery is possible for any $\delta >0$. Consequently, it is not surprising that for phase retrieval in the real case 
weak recovery can be achieved for $\delta$ below one.

Our information-theoretic lower bound is proved by estimating the conditional entropy via the second moment method. 
In general, this might not match the spectral upper bound. We provide an example in which there is a strictly positive gap between $\delta_\ell$ and $\delta_{\rm u}$ in Remark \ref{rmk:gap} at the end of Section \ref{sec:main_real}.

As in earlier work (see Section \ref{sec:Related}), we consider spectral algorithms that computesthe eigenvector  corresponding to the largest eigenvalue of a matrix of the form:
\begin{equation}
\bD_n = \frac{1}{n}\sum_{i=1}^n \mathcal T (y_i) \ba_i \ba_i^*\, ,\label{eq:D_Matrix}
\end{equation}
where $\mathcal T: \mathbb R\to \mathbb R$ is a pre-processing function. 
For $\delta$ large  enough (and a suitable choice of $\mathcal T$), we expect the resulting eigenvector $\hat{\bx}(\by)$ to be positively 
correlated with the true signal $\bx$. 
The recent paper \cite{lulispectral_arxiv}  computed exactly the threshold value $\delta_{\rm u}$, 
under the assumption that the measurement vectors are real Gaussian, and $\cT$ is non-negative.

\begin{figure}[p]
    \centering
    \subfloat[Original image.\label{fig:ven}]{\includegraphics[width=.4\columnwidth]{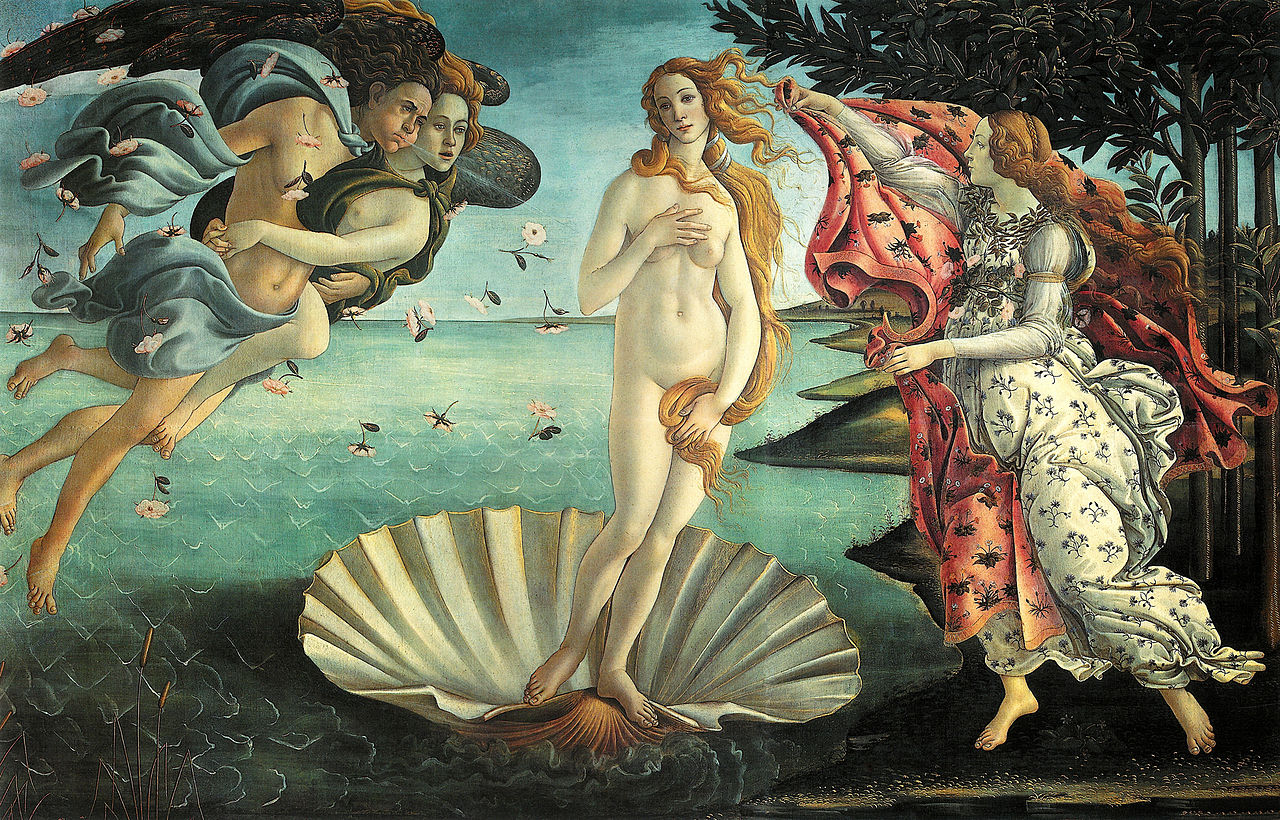}}\\
    \subfloat[proposed -- $\delta=4$.\label{fig:venopt4}]{\includegraphics[width=.4\columnwidth]{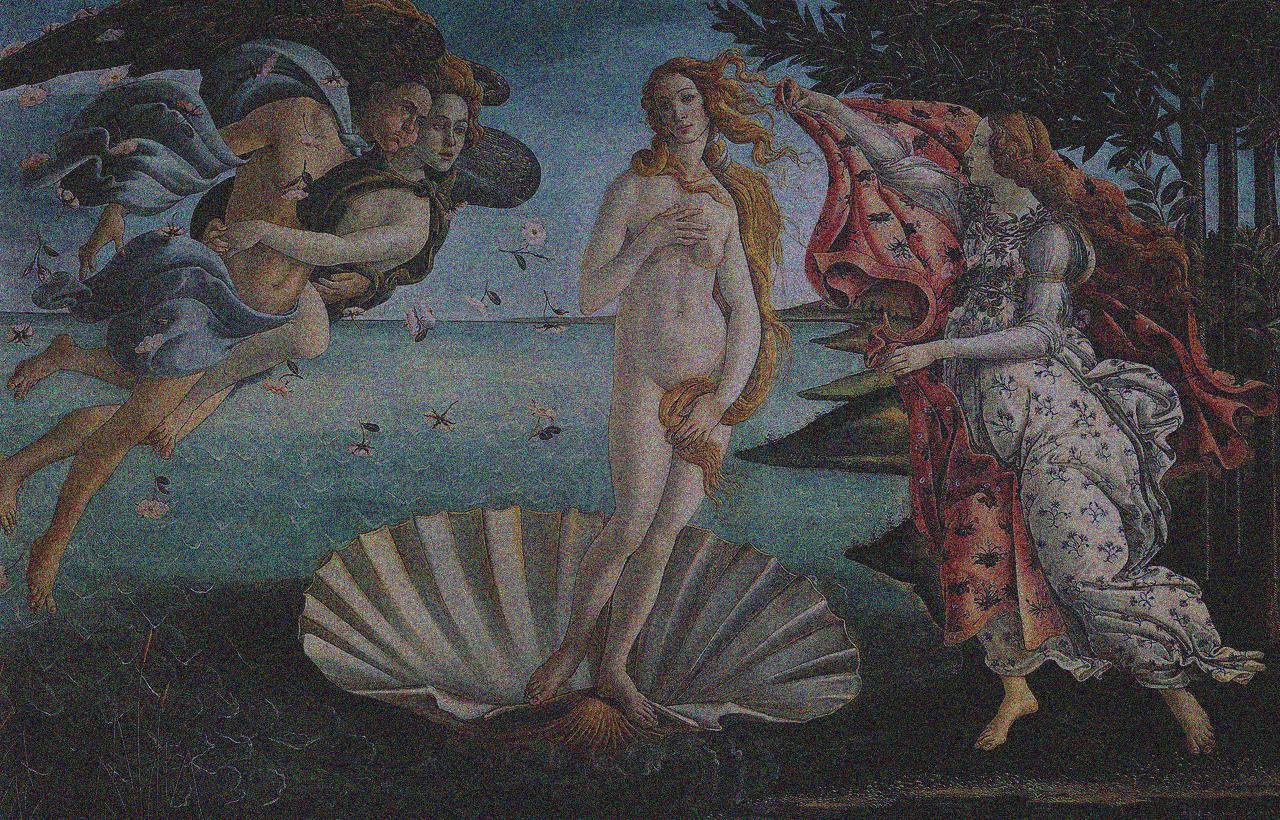}} \hspace{2em}
    \subfloat[truncated -- $\delta=4$.\label{fig:venmahdi4}]{\includegraphics[width=.4\columnwidth]{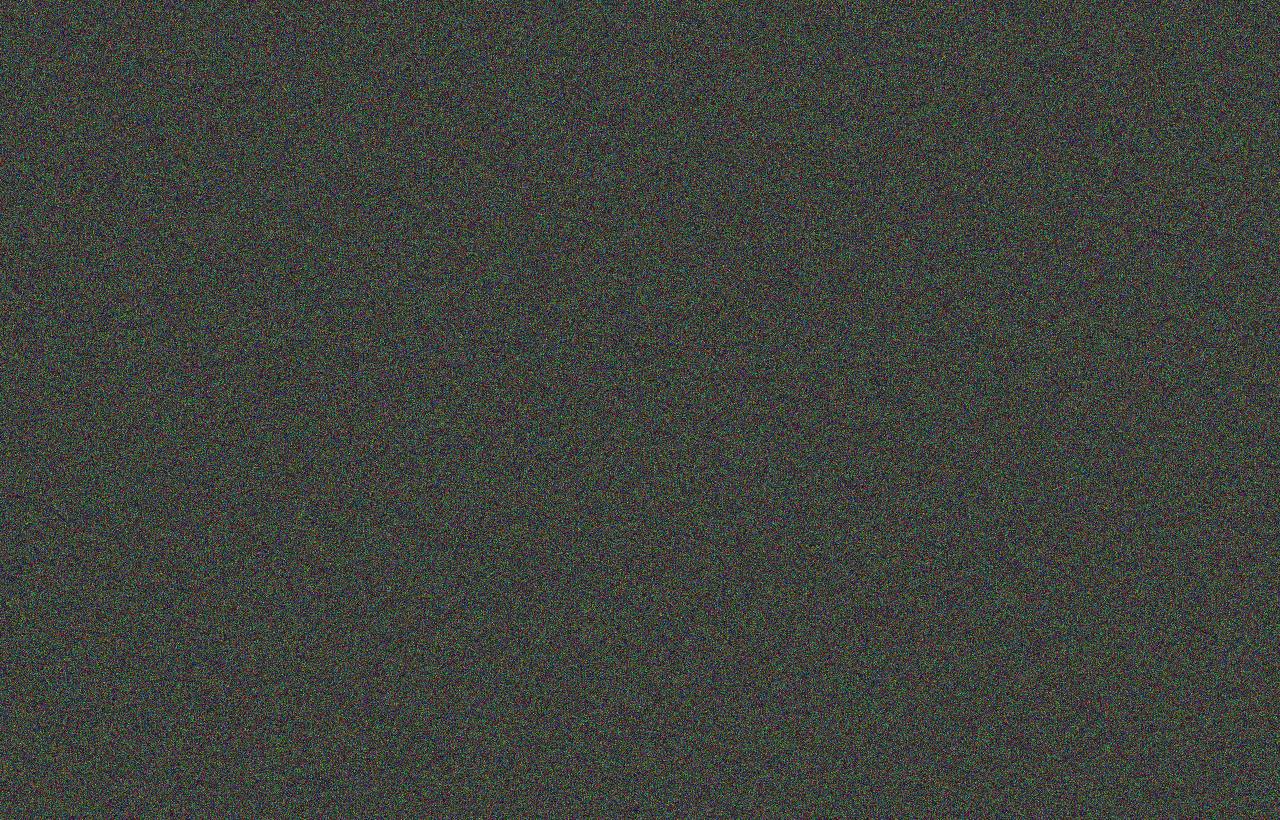}}\\
    \subfloat[proposed -- $\delta=6$.\label{fig:venopt6}]{\includegraphics[width=.4\columnwidth]{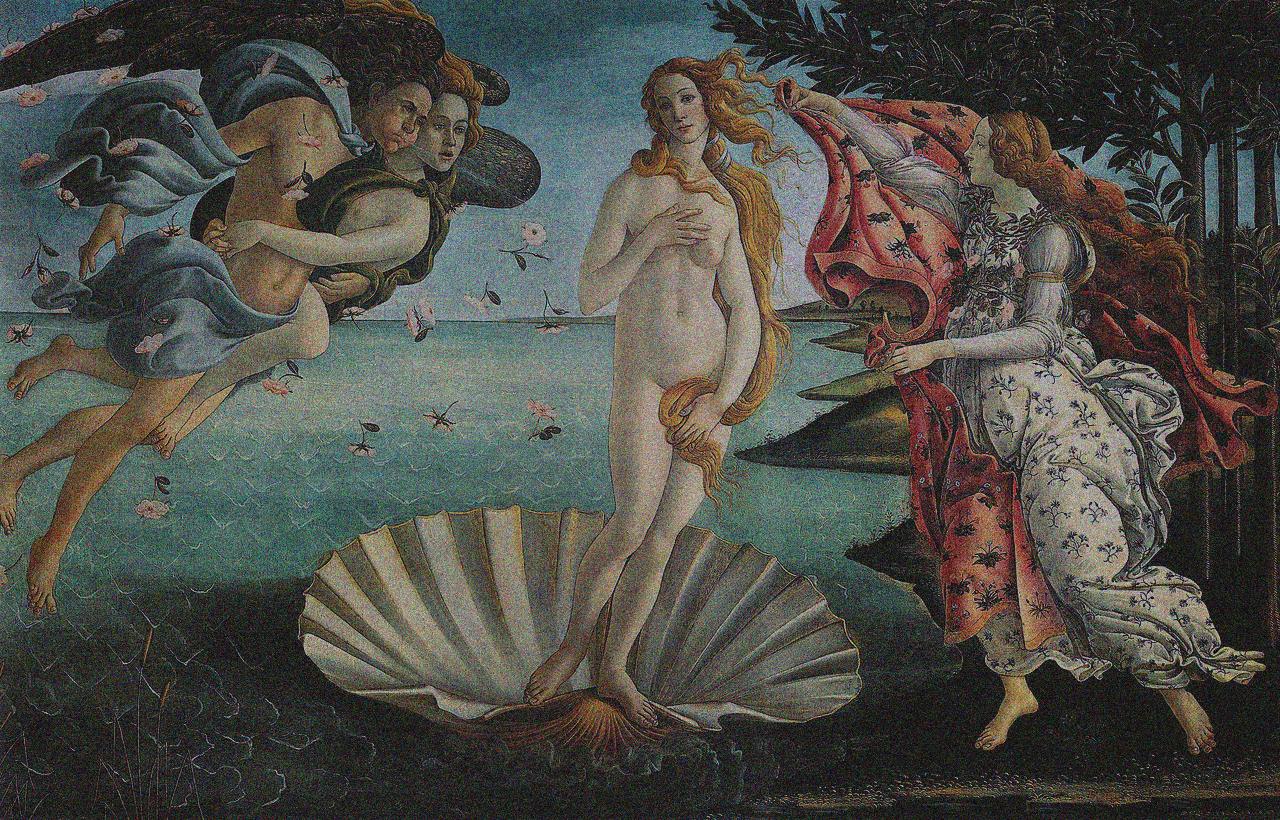}} \hspace{2em}
    \subfloat[truncated -- $\delta=6$.\label{fig:venmahdi6}]{\includegraphics[width=.4\columnwidth]{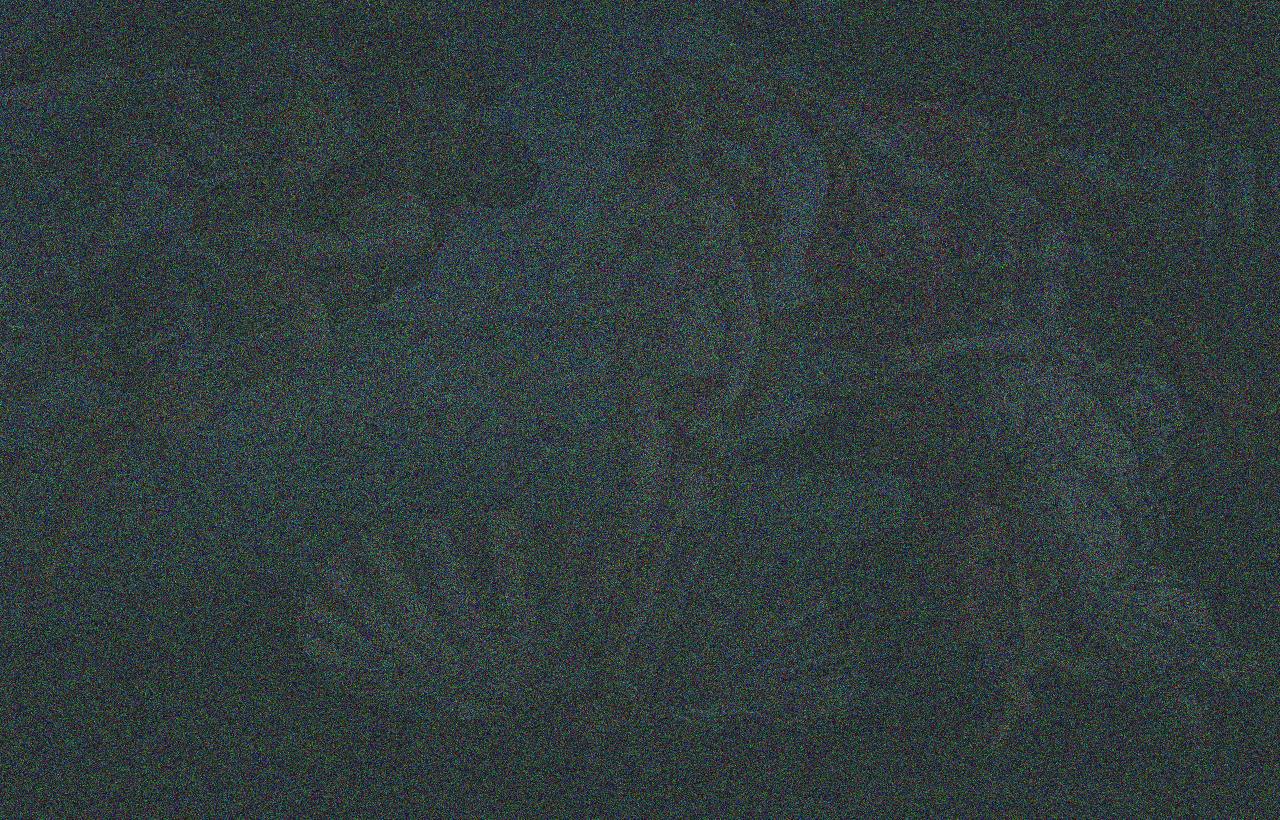}}\\
    \subfloat[proposed -- $\delta=12$.\label{fig:venopt12}]{\includegraphics[width=.4\columnwidth]{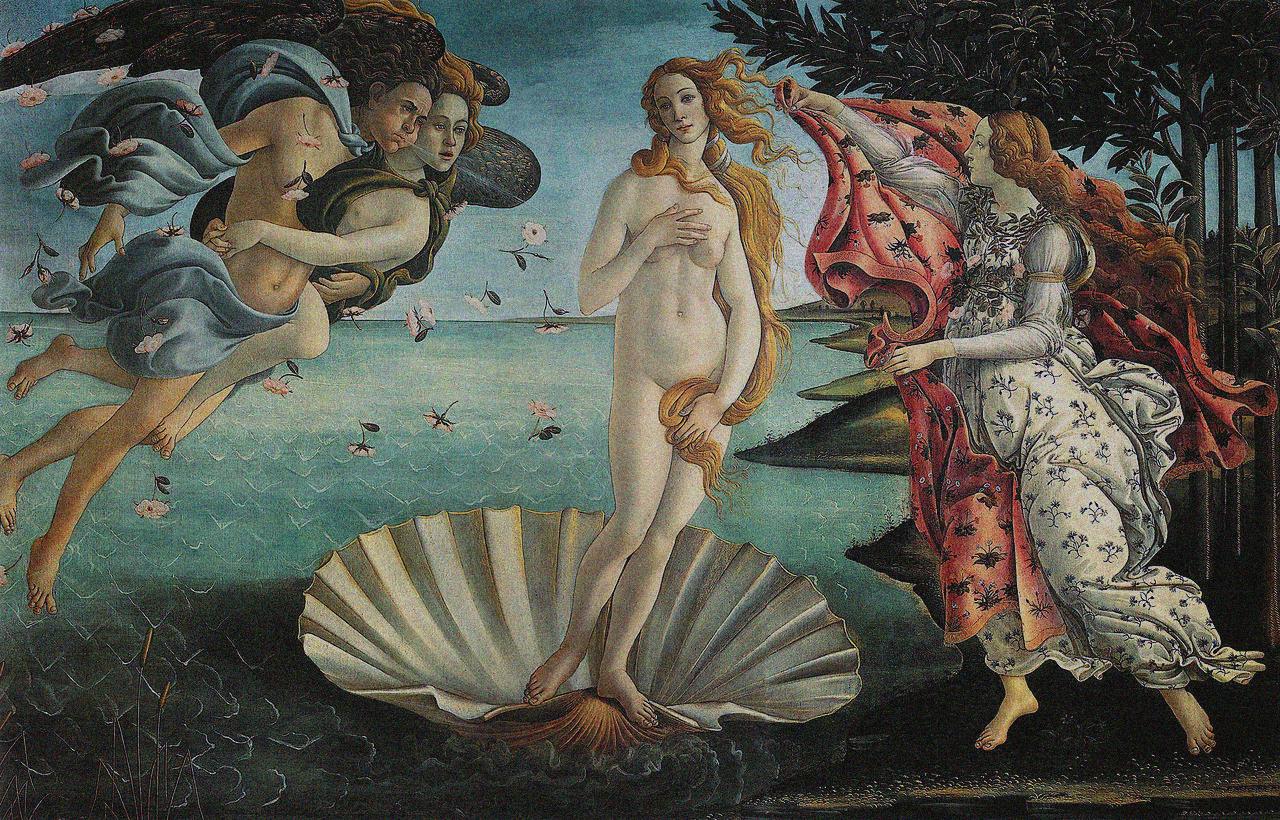}} \hspace{2em}
    \subfloat[truncated -- $\delta=12$.\label{fig:venmahdi12}]{\includegraphics[width=.4\columnwidth]{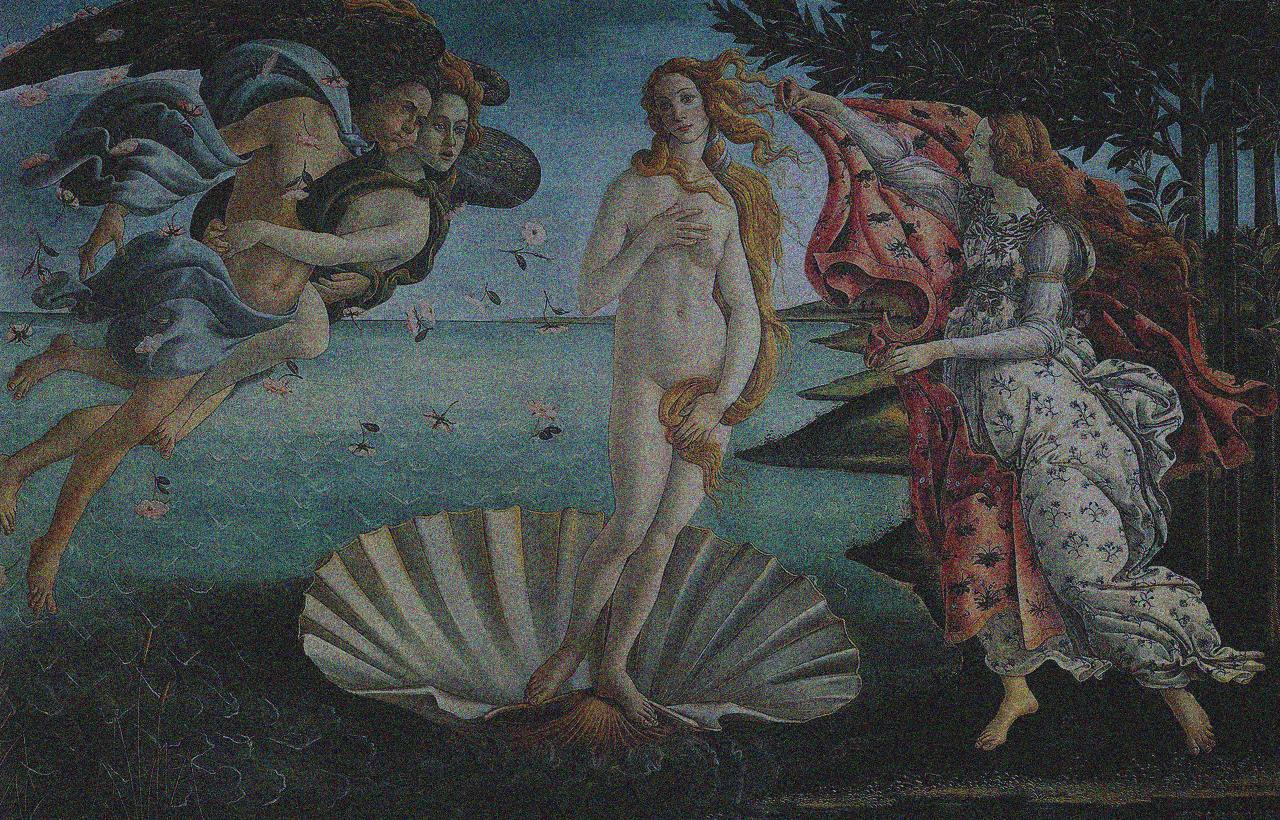}}
\caption{Performance comparison between the proposed spectral method and the truncated spectral initialization of \cite{chen2017solving} for the recovery of a digital photograph from coded diffraction patterns.}
\label{fig:venrec}
\end{figure}

Here, we generalize the result of \cite{lulispectral_arxiv} by removing the assumption that  $\mathcal T(y)\ge 0$ and by considering the complex case. The main technical lemma of this generalization consists in the computation of the largest eigenvalue of a matrix of the form $\bU \bM_n \bU^*$, where the entries of $\bU$ are $\sim_{i.i.d.}\cnormal (0, 1)$ and $\bM_n$ is independent of $\bU$ and has known empirical spectral distribution. The case in which $\bM_n$ is PSD is handled in \cite{baiyaospike2012}. In this paper, by using tools from free probability, we solve the case in which $\bM_n$ is not necessarily PSD. To do so, it is not sufficient to compute the weak limit of the empirical spectral distribution of $\bU \bM_n \bU^*$, but we also need to compute the almost sure limit of its principal eigenvalue.
Armed with this result, we compute the optimal 
pre-processing function $\cT^*_{\delta}(y)$ for the general model (\ref{eq:model}). 
This pre-processing function is optimal in the sense that it provides the smallest possible weak recovery threshold for the spectral method. Our upper bound
$\delta_u$ is the phase transition location for this optimal spectral method.  
In the  case of phase retrieval (as $\sigma\to 0$),  the optimal pre-processing function is given by
\begin{equation}
\mathcal T^*_{\delta}(y) = \frac{y-1}{y+\sqrt{\delta}-1}, \label{eq:T-Phase-Retrieval}
\end{equation}
and achieves weak recovery for any $\delta>\delta_u=1$. 
In the limit $\delta\downarrow 1$, this converges to the limiting function $\cT^*(y) = 1-(1/y)$.

While the expression (\ref{eq:T-Phase-Retrieval}) is remarkably simple, it is somewhat counter-intuitive. Earlier methods 
\cite{candes2015wirt,chen2015solving,lulispectral_arxiv}  use $\cT(y)\ge 0$ and try to extract information from the large values of $y_i$. The function (\ref{eq:T-Phase-Retrieval})
has a large negative part for small $y$, in particular when $\delta$ is close to $1$. Furthermore, it extracts useful information from data
points with $y_i$ small. One possible interpretation is that the points in which the measurement vector is basically orthogonal to the unknown signal are not informative, hence we penalize them. 




Our analysis applies to Gaussian measurement matrices. However, the proposed spectral method works well also on real images and realistic
measurement matrices. To illustrate this fact, in Figure \ref{fig:venrec} we test our algorithm on a digital photograph of the painting ``The birth of Venus'' by Sandro Botticelli. We consider a type of measurements that falls under the category of coded diffraction patterns (CDP) \cite{candes2015phase2, chen2017solving}: the measurement matrix is given by the product of $\delta$ copies of a Fourier matrix and a diagonal matrix with entries i.i.d. and uniform in $\{1, -1, i, -i\}$, where $i$ denotes the imaginary unit. We compare our method with the truncated spectral initialization proposed in \cite{chen2017solving}, which consists in discarding the measurements larger than an assigned threshold and leaving the others untouched. The proposed choice of the pre-processing function allows to recover a good estimate of the original image already when $\delta = 4$, while the truncated spectral initialization of \cite{chen2017solving} requires $\delta=12$ to obtain similar results. 

In general, our proposed spectral method can be thought of as a first step of the following two-round algorithm: first, use spectral initialization to perform weak recovery and then improve the solution with an iterative algorithm, e.g., AMP or Wirtinger Flow. By using optimal truncation methods, the weak recovery threshold is smaller, which means that less measurements are required in order to successfully complete the first step of the algorithm. If a different truncation
is used, the resulting performances are limited by the corresponding weak recovery threshold.

Note that the pre-processing function \eqref{eq:T-Phase-Retrieval} is optimal in the sense that it minimizes the weak recovery threshold associated to the spectral method. Hence, for a given correlation $\bar{\epsilon}\in (0, 1)$, the exact expression of the optimal pre-processing function that allows to obtain a correlation $\bar{\epsilon}$ between $\hat{\bx}(\by)$ and $\bx$ might be different and it might depend on $\bar{\epsilon}$. However, we observe that \eqref{eq:T-Phase-Retrieval} provides excellent empirical performance and outperforms state-of-the-art methods for a wide range of target correlations (see the simulation results of Section \ref{sec:num}).

The rest of the paper is organized as follows. In Section \ref{sec:main_complex}, after introducing the necessary notation, we define formally the problem. We then state our general 
information-theoretic lower bound and our spectral upper bound for the case of complex signal $\bx$ and complex measurement vectors $\ba_i$. The main results for the real case are stated in Section \ref{sec:main_real}.
 In Sections \ref{sec:mainproof} and \ref{sec:spectralproof}, we present the proof of the information-theoretic lower bound and of the spectral upper bound, respectively. In Section \ref{sec:amp}, 
we compare the spectral approach to a message passing algorithm. In particular we show that the latter cannot have a better threshold than $\delta_{\rm u}$ and that $\delta_{\rm u}$ is the threshold for a linearized version of message passing. In Section \ref{sec:num}, we present some numerical simulations that illustrate the behavior of the proposed spectral method for the phase retrieval problem. The proofs of several results are deferred to the various appendices.

\subsection{Related Work}\label{sec:Related}

Precise asymptotic information on high-dimensional regression problems has been obtained by several groups in recent years \cite{DMM-NSPT-11, BayatiMontanariLASSO, squaredOymak, polytopeBLM, highdimDonoho, asyKar, slopeSu, LK16, plan2016generalized, thrampoulidis2015lasso, neykov2016agnostic}. In particular, information-theoretically optimal estimation was considered for compressed sensing \cite{spatialcomprDonoho}, and random linear estimation 
\cite{replicaReeves, MIBarbier}. Minimax optimal estimation is considered, among others, in \cite{DMM-NSPT-11, slopeSu, ramjilower}.

The performance of the spectral methods for phase retrieval was first considered in \cite{netrapalli2013phase}. In the present notation,  \cite{netrapalli2013phase} uses
$\cT(y) = y$ and proves that there exists a constant $c_1$ such that weak recovery  can be achieved  for $n > c_1 \cdot d \cdot\log^3 d$. The same paper also gives an iterative procedure to improve over 
the spectral method, but the bottleneck is in the spectral step.
The sample complexity of weak recovery using  spectral methods was improved to $n > c_2\cdot d \cdot\log d$ in \cite{candes2015wirt} and then to $n > c_3\cdot d$ in \cite{chen2017solving}, 
for some constants $c_2$ and $c_3$. Both of these papers also prove guarantees for exact recovery by suitable descent algorithms. 
The guarantees on the spectral initialization are proved by matrix concentration inequalities, a technique that typically does not return exact threshold values.

In \cite{goldstein2016phasemax}, the authors introduce the PhaseMax relaxation and prove an exact recovery result for phase retrieval, which depends on the correlation between the true signal and the initial estimate given to the algorithm. The same idea was independently proposed in \cite{bahmani17a}. Furthermore, the analysis in \cite{bahmani17a} allows to use the same set of measurements for both initialization and convex programming, whereas the analysis in \cite{goldstein2016phasemax} requires fresh extra measurements for convex programming. By using our spectral method to obtain the initial estimate, it should be possible to improve the existing upper bounds on the number of samples needed for exact recovery. 

 As previously mentioned, our analysis of spectral methods builds on the recent work of Lu and Li \cite{lulispectral_arxiv}
that compute the exact spectral threshold for a matrix of the form (\ref{eq:D_Matrix}) with $\cT(y)\ge 0$.  Here we generalize this result to signed pre-processing functions 
$\cT(y)$, and construct a function of this type that achieves the information-theoretic threshold for phase retrieval. Our proof indeed implies that non-negative 
pre-processing functions lead to an unavoidable gap with respect to the ideal threshold.

Finally, while this paper was under completion, two works appeared that address related problems. In \cite{barbier2017phase}, the authors characterize the information-theoretically optimal estimation error for
a broad class of models of the form (\ref{eq:model}). However, note that this analysis does not prove -- in general -- the existence of an efficient estimation algorithm (for instance in the 
case of phase retrieval). The paper \cite{dhifallah2017fundamental} studies the PhaseMax approach \cite{goldstein2016phasemax, bahmani17a} to phase retrieval and uses the non-rigorous replica method from statistical physics to derive exact thresholds 
for this algorithm. The rigorous performance analysis of PhaseMax under Gaussian measurements in the large system limit is provided in \cite{dhifallah2017phasebis}.

\section{Main Results: Complex Case} \label{sec:main_complex}

\subsection{Notation and System Model} \label{subsec:sys}

We use $[n]$ as a shortcut for $\{1, \ldots, n\}$. We use upper-case letters (e.g., $X, Y, Z, \ldots$) to denote random variables when we are taking operators such as expectation, variance or mutual information. We denote by $\b0_n$ the vector consisting of $n$ 0s. Given a vector $\bx$, we denote by $\norm{\bx}_2$ its $\ell_2$ norm. Given a matrix $\bA$, we denote by $\norm{\bA}_F$ its Frobenius norm, by $\norm{\bA}_{\rm op}$ its operator norm, by $\bA^{\sT}$ its transpose, and by $\bA^*$ its conjugate transpose. Given two vectors $\bx, \by \in {\mathbb C}^d$, we denote by $\langle \bx, \by \rangle = \sum_{i=1}^d x_i y_i^*$ their scalar product. We take logarithms in the natural basis and we measure entropies in nats. Given $c\in \mathbb C$, we denote by $\Re{(c)}$ and $\Im{(c)}$ its real and imaginary part, respectively. We use $\stackrel{\mathclap{\mathcal P}}{\longrightarrow}$ and $\stackrel{\mathclap{\mbox{\footnotesize a.s.}}}{\longrightarrow}$ to denote the convergence in probability and the almost sure convergence, respectively. 

Let $\bx\in {\mathbb C}^d$ be chosen uniformly at random on the $d$-dimensional complex sphere with radius $\sqrt{d}$, i.e.,
\begin{equation}\label{eq:defx}
\bx \sim \Unif(\sqrt{d}\Sphere_{\mathbb C}^{d-1}).
\end{equation}

Let the sensing vectors $\{\ba_i\}_{1\le i\le n}$, with $\ba_i \in {\mathbb C}^d$, be independent and identically distributed according to a circularly-symmetric complex normal distribution with variance $1/d$, i.e.,
\begin{equation}\label{eq:defa}
\{\ba_i\}_{1\le i\le n}\sim_{i.i.d.}\cnormal({\bm 0}_d,\id_d/d).
\end{equation}

Given $g_i = \langle \bx, \ba_i\rangle$, the vector of measurements $\by\in {\mathbb R}^n$ is obtained by drawing each component independently according to the following distribution:
\begin{equation}\label{eq:defy}
y_i\sim p(y\mid |g_i|), \qquad i\in [n].
\end{equation}
For the special case of \emph{phase retrieval}, the measurements are given by the squared scalar product corrupted by additive Gaussian noise with variance $\sigma^2$, i.e.,
\begin{equation}\label{eq:defypr}
p_{\rm PR}(y\mid |g_i|)  =\frac{1}{\sigma\sqrt{2\pi}}\exp\left(-\frac{(y-|g_i|^2)^2}{2\sigma^2}\right).
\end{equation}

Let $\delta_n = n/d$ and assume that, as $n\to \infty$, $\delta_n\to \delta$ for some $\delta \in(0, \infty)$.

\subsection{Information-Theoretic Lower Bound} \label{subsec:lower}

The main result of this section establishes the following: there is a critical value $\delta_{\ell}$ such that, for any $\delta<\delta_{\ell}$, the optimal estimator has the same performance as a trivial estimator that does not have access to any measurement. The value of $\delta_{\ell}$ depends on the distribution \eqref{eq:defy} of the measurements and we provide an expression to compute it. 

In order to state formally the result, we need to introduce a few definitions. Consider the function $f : [0, 1]\to {\mathbb R}$, given by 
\begin{equation}\label{eq:fnorm}
f(m) = \bigintssss_{\mathbb R}\frac{{\mathbb E}_{G_1, G_2}\left\{p(y\mid |G_1|)p(y\mid |G_2|)\right\}}{{\mathbb E}_{G}\left\{p(y\mid |G|)\right\}} \,{\rm d}y,
\end{equation}
with 
\begin{equation}
G\sim \cnormal(0, 1),\qquad (G_1, G_2)\sim\cnormal\left(\b0_2, \left[\begin{array}{ll}
1 & c \\ c^* & 1\\ 
\end{array}\right]\right),
\end{equation}
and $m=|c|^2$. Note that the RHS of \eqref{eq:fnorm} depends only on $m=|c|^2$. Indeed, by applying the transformation $(G_1, G_2)\to (e^{i\theta_1}G_1, e^{i\theta_2}G_2)$, $f(m)$ does not change, but the correlation coefficient $c$ is mapped into $ce^{i(\theta_1-\theta_2)}$. A more explicit formula for $f(m)$ is provided by Lemma \ref{lemma:expl} in Appendix \ref{app:noiseless}. The function $f(m)$ is related to the conditional entropy $H(Y_1, \ldots, Y_n \mid \bA_1, \ldots, \bA_n)$, as clarified in the proof of Lemma \ref{lemma:ent} in Section \ref{subsec:proofcompl}. Furthermore, set
\begin{equation}\label{eq:defF}
F_{\delta}(m) = \delta \log f(m) + \log (1-m).
\end{equation}
Note that, when $m=0$, $G_1$ and $G_2$ are independent. Hence, $f(0)=1$, which implies that $F_{\delta}(0)=0$ for any $\delta >0$. We define the information-theoretic threshold $\delta_{\ell}$ as the largest value of $\delta$ such that the maximum of $F_{\delta}(m)$ is attained at $m=0$, i.e., 
\begin{equation}\label{eq:defdelta}
\delta_{\ell} = \sup \{\delta \mid F_{\delta}(m) < 0 \mbox{ for }m\in (0, 1]\}.
\end{equation}

Let us now define the error metric. The setting is the following: we observe the vector of $n$ measurements $\by$ and, given a new sensing vector $\ba_{n+1}$, we want to estimate some function $\phi(|\langle \bx, \ba_{n+1}\rangle|)$ given by
\begin{equation}\label{eq:defphi}
\phi(|\langle \bx, \ba_{n+1}\rangle|) = \int_{\mathbb R} \varphi(y) p(y\mid |\langle \bx, \ba_{n+1}\rangle|)\,{\rm d}y.
\end{equation}
Then, the minimum mean square error is defined as
\begin{equation}
\MMSE(\delta_n) = {\mathbb E}\biggl\{\Bigl(\phi(|\langle \bX, \bA_{n+1}\rangle|) - {\mathbb E}\bigl\{\phi(|\langle \bX, \bA_{n+1}\rangle|)\,\bigl\rvert \,\bY, \{\bA_i\}_{1\le i\le n}\bigr\}\Bigr)^2\biggr\},
\end{equation}
where ${\mathbb E}\left\{\phi(|\langle \bX, \bA_{n+1}\rangle|)\mid \bY, \{\bA_i\}_{1\le i\le n}\right\}$ represents the optimal estimator of the quantity $\phi(|\langle \bx, \ba_{n+1}\rangle|)$ and the expectation of the square error is to be intended over all the randomness of the system, i.e., over $\bX$, $\bA_{n+1}$, $\bY$, and $\{\bA_i\}_{1\le i\le n}$. Note that this error metric depends on the choice of the function $\phi$. Furthermore, observe that, if we do not have access to the vector of measurements $\bY$, the trivial estimator ${\mathbb E}\left\{\phi(|\langle \bX, \bA_{n+1}\rangle|)\right\}$ has a mean square error given by
\begin{equation}
{\mathbb E}\biggl\{\Bigl(\phi(|\langle \bX, \bA_{n+1}\rangle|) - {\mathbb E}\bigl\{\phi(|\langle \bX, \bA_{n+1}\rangle|)\bigr\}\Bigr)^2\biggr\} = {\rm Var}\bigl\{\phi(|\langle \bX, \bA_{n+1}\rangle|)\bigr\}.
\end{equation}

At this point we are ready to state our main result, which is proved in Section \ref{subsec:proofcompl}.

\begin{theorem}[Information-Theoretic Lower Bound for General Complex Sensing Model]\label{th:lower}
Let $\bx$, $\{\ba_i\}_{1\le i\le n+1}$, and $\by$ be distributed according to \eqref{eq:defx}, \eqref{eq:defa}, and \eqref{eq:defy}, respectively. Let $n/d\to \delta$ and define $\delta_{\ell}$ as in \eqref{eq:defdelta}. Furthermore, assume that the function $\varphi$ that appears in \eqref{eq:defphi} is bounded. Then, for any $\delta < \delta_{\ell}$, we have that
\begin{equation}
\lim_{n\to \infty}\MMSE(\delta_n) = {\rm Var}\bigl\{\phi(|\langle \bX, \bA_{n+1}\rangle|)\bigr\}. 
\end{equation}
\end{theorem}

Let us point out that the requirement that the function $\varphi$ is bounded can be relaxed when the tails of the distribution of $\bY$ are sufficiently light (e.g., sub-Gaussian). Indeed, this is what happens for the special case of phase retrieval, which is considered immediately below. 

For the special case of \emph{phase retrieval}, a more explicit error metric is given by the matrix minimum mean square error, defined as 
\begin{equation}\label{eq:defMMSEphretr}
\MMSE_{\rm PR}(\delta_n) = \frac{1}{d^2} {\mathbb E}\left\{\Bigl\lVert\bX \bX^* - {\mathbb E}\bigl\{\bX \bX^*\mid \bY, \{\bA_i\}_{1\le i\le n}\bigr\}\Bigl\lVert_{F}^2\right\}.
\end{equation}
Indeed, the vector $\bx$ can be recovered only up to a sign change, since we observe a function of the scalar products $|\langle \bx, \ba_i\rangle|$. Clearly, $\MMSE(\delta_n)\in [0, 1]$ and $\MMSE(\delta_n) = 1$ implies that the optimal estimator coincides with the trivial estimator that outputs the all-0 vector.

The corollary below provides the exact value of $\delta_{\ell}$ for the case of phase retrieval and it is proved in Appendix \ref{app:noiseless}. 

\begin{corollary}[Information-Theoretic Lower Bound for Phase Retrieval]\label{cor:lower}
Let $\bx$, $\{\ba_i\}_{1\le i\le n}$, and $\by$ be distributed according to \eqref{eq:defx}, \eqref{eq:defa}, and \eqref{eq:defypr}, respectively. Let $n/d\to \delta$. Then, for any $\delta < 1$, we have that
\begin{equation}
\lim_{\sigma\to 0}\lim_{n\to \infty}\MMSE_{\rm PR}(\delta_n) = 1.
\end{equation}
\end{corollary}

\subsection{Upper Bound via Spectral Method} \label{subsec:upper}

The main result of this section establishes the following: there is a critical value $\delta_{\rm u}$ such that, for any $\delta>\delta_{\rm u}$, the principal eigenvector of a suitably constructed matrix, call it $\bD_n$, provides an estimate $\hat{\bx}$ that satisfies \eqref{eq:poscorr}. 
3
The threshold $\delta_{\rm u}$ is defined as
\begin{equation}\label{eq:defdeltau}
\delta_{\rm u} = \frac{1}{\displaystyle\bigintssss_{\mathbb R}\frac{\left({\mathbb E}_{G}\left\{p(y\mid |G|)(|G|^2-1)\right\}\right)^2}{{\mathbb E}_{G}\left\{p(y\mid |G|)\right\}} \,{\rm d}y},
\end{equation}
with $G\sim \cnormal(0, 1)$. Given the measurements $\{y_i\}_{1\le i\le n}$, we construct the matrix $\bD_n$ as  
\begin{equation}\label{eq:defD}
\bD_n = \frac{1}{n}\sum_{i=1}^n \mathcal T (y_i) \ba_i \ba_i^*,
\end{equation}
where $\mathcal T: \mathbb R\to \mathbb R$ is a pre-processing function.

At this point we are ready to state our main result, which is proved in Section \ref{sec:spectralproof}. 

\begin{theorem}[Spectral Upper Bound for Complex General Sensing Model]\label{th:upper}
Let $\bx$, $\{\ba_i\}_{1\le i\le n}$, and $\by$ be distributed according to \eqref{eq:defx}, \eqref{eq:defa}, and \eqref{eq:defy}, respectively. Let $n/d\to \delta$ and define $\delta_{\rm u}$ as in \eqref{eq:defdeltau}. Let $\hat{\bx}$ be the principal eigenvector of the matrix $\bD_n$ defined in \eqref{eq:defD}. For any $\delta > \delta_{\rm u}$, set the pre-processing function $\mathcal T$ to the function $\mathcal T_\delta^*$ given by
\begin{equation}\label{eq:deftydelta}
\mathcal T_\delta^*(y) = \frac{\sqrt{\delta_{\rm u}}\cdot\mathcal T^*(y)}{\sqrt{\delta}-(\sqrt{\delta}-\sqrt{\delta_{\rm u}})\mathcal T^*(y)},
\end{equation}
where
\begin{equation}\label{eq:defty}
\mathcal T^*(y) = 1-\frac{{\mathbb E}_{G}\left\{p(y\mid |G|)\right\}}{{\mathbb E}_{G}\left\{p(y\mid |G|)\cdot |G|^2\right\}}.
\end{equation}
Then, we have that, almost surely,
\begin{equation}\label{eq:corrub}
\lim_{n\to \infty}\frac{|\langle\hat{\bx}, \bx\rangle|}{\norm{\hat{\bx}}_2 \, \norm{\bx}_2} > \epsilon,
\end{equation}
for some $\epsilon >0$. Furthermore, for any $\delta \leq \delta_{\rm u}$, there is no pre-processing function $\mathcal T$ such that, almost surely, \eqref{eq:corrub} holds. 
\end{theorem}

Let us highlight that the pre-processing function \eqref{eq:deftydelta} provides the optimal threshold among spectral methods thatuse matrices of the form (\ref{eq:D_Matrix})  in the sense that it achieves weak recovery for $\delta > \delta_{\rm u}$ and no function achieves  weak recovery for $\delta \leq \delta_{\rm u}$. Note also that the assumption that $\bx$ is uniform on the sphere can be dropped (see the beginning of the proof of Lemma \ref{lemma:condub} in Section \ref{sec:spectralproof}).

As a byproduct of our analysis, we also give guarantees on the value of $\delta$ sufficient to achieve an assigned correlation with the ground truth, using the spectral method, see \eqref{eq:numpred} in the statement of Lemma \ref{lemma:condub} in Section \ref{sec:spectralproof}. Hence, we can combine our upper bound with existing nonconvex optimization algorithms, in order to obtain provable performance guarantees.

The corollary below provides the exact value of $\delta_{\rm u}$ and an explicit expression for $\mathcal T_\delta^*(y)$ for the case of phase retrieval. Its proof is contained in Appendix \ref{app:noiseless2}. Note that, for phase retrieval, $\delta_{\rm u}=\delta_{\ell}=1$, i.e., the spectral upper bound matches the information-theoretic lower bound. 
\begin{corollary}[Spectral Upper Bound for Phase Retrieval]\label{cor:upper}
Let $\bx$, $\{\ba_i\}_{1\le i\le n}$, and $\by$ be distributed according to \eqref{eq:defx}, \eqref{eq:defa}, and \eqref{eq:defypr}, respectively. Let $n/d\to \delta$. Let $\hat{\bx}$ be the principal eigenvector of the matrix $\bD_n$ defined in \eqref{eq:defD}. For any $\delta > 1$, set the pre-processing function $\mathcal T$ to the function $\mathcal T_\delta^*$ given by (with $y_+\equiv \max(0,y)$):
\begin{equation}\label{eq:Tstardeltapr}
\mathcal T^*_{\delta}(y) = \frac{y_+-1}{y_++\sqrt{\delta}-1}.
\end{equation}
Then, we have that, almost surely,
\begin{equation}\label{eq:limscaldelta}
\lim_{\sigma\to 0}\lim_{n\to \infty}\frac{|\langle\hat{\bx}, \bx\rangle|}{\norm{\hat{\bx}}_2 \, \norm{\bx}_2} > \epsilon,
\end{equation}
for some $\epsilon >0$.
\end{corollary}

Notice that this statement is stronger than the claim that $\delta_{{\rm u}}(\sigma^2)\to 1$ as $\sigma^2\to 0$, where $\delta_{{\rm u}}(\sigma^2)$ is the spectral threshold
at noise level $\sigma^2$.  Indeed it requires proving that the scalar product $|\langle\hat{\bx}, \bx\rangle|$ stays bounded away from $0$, as $\sigma^2\to 0$.
Furthermore, this is achieved with the pre-processing function (\ref{eq:Tstardeltapr}) that does not require to estimate $\sigma$, which can be challenging with real data.

We also characterize the scaling between $\delta_{{\rm u}}$ and $\sigma^2$ when $\sigma^2$ is close to $0$: $\delta_{{\rm u}}(\sigma^2)= 1+\sigma^2+o(\sigma^2)$ (see Lemma \ref{lemma:deltaulnoiseless} in Appendix \ref{app:noiseless2}).

\section{Main Results: Real Case} \label{sec:main_real}

Let us now briefly discuss what happens in the real case. Let $\bx\in {\mathbb R}^d$ be chosen uniformly at random on the $d$-dimensional real sphere with radius $\sqrt{d}$, i.e.,
\begin{equation}\label{eq:defxr}
\bx \sim \Unif(\sqrt{d}\Sphere_{\mathbb R}^{d-1}).
\end{equation}
Let the sensing vectors $\{\ba_i\}_{1\le i\le n}$, with $\ba_i \in {\mathbb R}^d$ being independent and identically distributed according to a normal distribution with zero mean and variance $1/d$, i.e.,
\begin{equation}\label{eq:defar}
\{\ba_i\}_{1\le i\le n}\sim_{i.i.d.}\normal({\bm 0}_d,\id_d/d).
\end{equation}
Given $g_i = \langle \bx, \ba_i\rangle$, the vector of measurements $\by\in {\mathbb R}^n$ is obtained by drawing each component independently according to the following distribution:
\begin{equation}\label{eq:defyr}
y_i\sim p(y\mid g_i), \qquad i\in [n].
\end{equation}
We can define the \emph{``real'' phase retrieval} model, whereby the measurements are given by the squared scalar product corrupted by additive Gaussian noise with variance $\sigma^2$, i.e.,
\begin{equation}\label{eq:defyprr}
p_{\rm PR}(y\mid g_i)  =\frac{1}{\sigma\sqrt{2\pi}}\exp\left(-\frac{(y-g_i^2)^2}{2\sigma^2}\right).
\end{equation}

We first present the information-theoretic lower bound. Consider the function $f : [-1, 1]\to {\mathbb R}$, given by 
\begin{equation}\label{eq:fnormr}
f(m) = \bigintssss_{\mathbb R}\frac{{\mathbb E}_{G_1, G_2}\left\{p(y\mid G_1)p(y\mid G_2)\right\}}{{\mathbb E}_{G}\left\{p(y\mid G)\right\}} \,{\rm d}y,
\end{equation}
with 
\begin{equation}
G\sim \normal(0, 1),\qquad (G_1, G_2)\sim\normal\left(\b0_2, \left[\begin{array}{ll}
1 & m \\ m & 1\\ 
\end{array}\right]\right).
\end{equation}
Furthermore, set
\begin{equation}\label{eq:defFr}
F_{\delta}(m) = \delta \log f(m) + \frac{1}{2}\log (1-m^2).
\end{equation}
Again, $F_{\delta}(0)=0$ for any $\delta >0$. We define the information-theoretic threshold $\delta_{\ell}$ as the largest value of $\delta$ such that the maximum of $F_{\delta}(m)$ is attained at $m=0$, i.e., 
\begin{equation}\label{eq:defdeltar}
\delta_{\ell} = \sup \{\delta \mid F_{\delta}(m) < 0 \mbox{ for }m\in [-1, 1]\setminus\{0\}\}.
\end{equation}

As for the error metric, we observe the vector of $n$ measurements $\by$ and, given a new sensing vector $\ba_{n+1}$, we want to estimate some function $\phi(\langle \bx, \ba_{n+1}\rangle)$ given by
\begin{equation}\label{eq:defphir}
\phi(\langle \bx, \ba_{n+1}\rangle) = \int_{\mathbb R} \varphi(y) p(y\mid \langle \bx, \ba_{n+1}\rangle){\rm d}y.
\end{equation}
Then, the minimum mean square error is defined as
\begin{equation}
\MMSE(\delta_n) = {\mathbb E}\biggl\{\Bigl(\phi(\langle \bX, \bA_{n+1}\rangle) - {\mathbb E}\bigl\{\phi(\langle \bX, \bA_{n+1}\rangle)\, \bigl\rvert\, \bY, \{\bA_i\}_{1\le i\le n}\bigr\}\Bigr)^2\biggr\}.
\end{equation}
Recall that, if we do not have access to the vector of measurements $\by$, the trivial estimator ${\mathbb E}\left\{\phi(\langle \bX, \bA_{n+1}\rangle)\right\}$ has a mean square error given by
\begin{equation}
{\mathbb E}\biggl\{\Bigl(\phi(\langle \bX, \bA_{n+1}\rangle) - {\mathbb E}\bigl\{\phi(\langle \bX, \bA_{n+1}\rangle)\bigr\}\Bigr)^2\biggr\} = {\rm Var}\bigl\{\phi(\langle \bX, \bA_{n+1}\rangle)\bigr\}.
\end{equation}

At this point we are ready to state the information-theoretic lower bound, which is proved in Section \ref{subsec:proofreal}.

\begin{theorem}[Information-Theoretic Lower Bound for Real General Sensing Model]\label{th:lowerr}
Let $\bx$, $\{\ba_i\}_{1\le i\le n+1}$, and $\by$ be distributed according to \eqref{eq:defxr}, \eqref{eq:defar}, and \eqref{eq:defyr}, respectively. Let $n/d\to \delta$ and define $\delta_{\ell}$ as in \eqref{eq:defdeltar}. Furthermore, assume that the function $\varphi$ that appears in \eqref{eq:defphir} is bounded. Then, for any $\delta < \delta_{\ell}$, we have that
\begin{equation}
\lim_{n\to \infty}\MMSE(\delta_n) = {\rm Var}\bigl\{\phi(\langle \bX, \bA_{n+1}\rangle)\bigr\}.
\end{equation}
\end{theorem}

\begin{remark}[Information-Theoretic Lower Bound for Real Phase Retrieval]
For the special case of \emph{phase retrieval}, a more explicit error metric is given by the matrix minimum mean square error, defined as
\begin{equation}
\MMSE_{\rm PR}(\delta_n) = \frac{1}{d^2} {\mathbb E}\left\{\Bigl\lVert\bX \bX^{\sT} - {\mathbb E}\bigl\{\bX \bX^{\sT}\mid \bY, \{\bA_i\}_{1\le i\le n}\bigr\}\Bigl\lVert_{F}^2\right\}.
\end{equation}
By calculations similar to those in Lemma \ref{lemma:deltalnoiseless} contained in Appendix  \ref{app:noiseless}, one can prove that, if the distribution $p(\cdot\mid G)$ appearing in \eqref{eq:fnormr} is given by \eqref{eq:defyprr}, then
\begin{equation}
\lim_{\sigma\to 0}\delta_{\ell}(\sigma^2)=1/2.
\end{equation}
Consequently, by following a proof analogous to that of Corollary \ref{cor:lower} in Appendix \ref{app:noiseless}, we conclude that, for any $\delta < 1/2$,
\begin{equation}
\lim_{\sigma\to 0}\lim_{n\to \infty}\MMSE_{\rm PR}(\delta_n) = 1. 
\end{equation}
\end{remark}

Let us now move to the spectral upper bound. The threshold $\delta_{\rm u}$ is defined as
\begin{equation}\label{eq:defdeltaur}
\delta_{\rm u} = \frac{1}{\displaystyle\bigintssss_{\mathbb R}\frac{\left({\mathbb E}_{G}\left\{p(y\mid G)(G^2-1)\right\}\right)^2}{{\mathbb E}_{G}\left\{p(y\mid G)\right\}} \,{\rm d}y},
\end{equation}
with $G\sim \normal(0, 1)$. Given the measurements $\{y_i\}_{1\le i\le n}$, we construct the matrix $\bD_n$ as  
\begin{equation}\label{eq:defDr}
\bD_n = \frac{1}{n}\sum_{i=1}^n \mathcal T (y_i) \ba_i \ba_i^{\sT},
\end{equation}
where $\mathcal T:\mathbb R\to \mathbb R$ is a pre-processing function.

The proof of the following spectral upper bound is discussed in Remark \ref{rmk:realub} at the end of Section \ref{sec:spectralproof}. 

\begin{theorem}[Spectral Upper Bound for Real General Sensing Model]\label{th:upperr}
Let $\bx$, $\{\ba_i\}_{1\le i\le n}$, and $\by$ be distributed according to \eqref{eq:defxr}, \eqref{eq:defar}, and \eqref{eq:defyr}, respectively. Let $n/d\to \delta$ and define $\delta_{\rm u}$ as in \eqref{eq:defdeltaur}. Let $\hat{\bx}$ be the principal eigenvector of the matrix $\bD_n$ defined in \eqref{eq:defDr}. For any $\delta > \delta_{\rm u}$, set the pre-processing function $\mathcal T$ to the function $\mathcal T_\delta^*$ given by
\begin{equation}\label{eq:deftydeltar}
\mathcal T_\delta^*(y) = \frac{\sqrt{\delta_{\rm u}}\cdot\mathcal T^*(y)}{\sqrt{\delta}-(\sqrt{\delta}-\sqrt{\delta_{\rm u}})\mathcal T^*(y)},
\end{equation}
where
\begin{equation}\label{eq:deftyr}
\mathcal T^*(y) = 1-\frac{{\mathbb E}_{G}\left\{p(y\mid G)\right\}}{{\mathbb E}_{G}\left\{p(y\mid G)\cdot G^2\right\}}.
\end{equation}
Then, we have that, almost surely,
\begin{equation}\label{eq:upperreps}
\lim_{n\to \infty}\frac{|\langle\hat{\bx}, \bx\rangle|}{\norm{\hat{\bx}}_2 \, \norm{\bx}_2} > \epsilon,
\end{equation}
for some $\epsilon >0$.  Furthermore, for any $\delta \leq \delta_{\rm u}$, there is no pre-processing function $\mathcal T$ such that, almost surely, \eqref{eq:upperreps} holds. 
\end{theorem}

\begin{remark}[Spectral Upper Bound for Real Phase Retrieval]
By calculations similar to those in Lemma \ref{lemma:deltaulnoiseless} contained in Appendix \ref{app:noiseless2}, one can prove that, if the distribution $p(\cdot\mid G)$ appearing in \eqref{eq:fnormr} is given by \eqref{eq:defyprr}, then
\begin{equation}
\lim_{\sigma\to 0}\delta_{\rm u}(\sigma^2)=1/2.
\end{equation}
Furthermore, by following a proof analogous to that of Corollary \ref{cor:upper} in Appendix \ref{app:noiseless2}, one can prove the following result. For any $\delta > 1/2$, set the pre-processing function $\mathcal T$ to the function $\mathcal T_\delta^*$ given by  (with $y_+=\max(y,0)$)
\begin{equation}
\mathcal T^*_{\delta}(y) = \frac{y_+-1}{y_++\sqrt{2\delta}-1}.
\end{equation}
Then, we have that, almost surely,
\begin{equation}
\lim_{\sigma\to 0}\lim_{n\to \infty}\frac{|\langle\hat{\bx}, \bx\rangle|}{\norm{\hat{\bx}}_2 \, \norm{\bx}_2} > \epsilon,
\end{equation}
for some $\epsilon >0$. Note that, for real phase retrieval, the spectral upper bound matches the information-theoretic lower bound.
\end{remark}

In the following remark, we provide an example in which there is a strictly positive gap between $\delta_\ell$ and $\delta_{\rm u}$.

\begin{remark}[Gap between $\delta_\ell$ and $\delta_{\rm u}$]\label{rmk:gap}
Let us define
\begin{equation}
H(a) = {\mathbb E}_{G}\left\{\tanh^2(a\, G)\,(G^2-1)\right\},
\end{equation}
where $G\sim \normal(0, 1)$. Note that $H(0)=0$ and $\lim_{a\to\infty}H(a)=0$. Hence, there exists $a_2>a_1$ such that $H(a_1)=H(a_2)$.

Consider the following distribution for the components of the vector of measurements $\by$:
\begin{equation}
p(y\mid g) = \left\{\begin{array}{ll}
\tanh^2(a_2\, g)-\tanh^2(a_1\, g), & \mbox{ for }y\in [1, 2],\\
1-(\tanh^2(a_2\, g)-\tanh^2(a_1\, g)), & \mbox{ for }y\in [-2, -1].\\
\end{array}\right.
\end{equation}
Then, we have that, for any $y\in \mathbb R$,
\begin{equation}\label{eq:deltauinf}
{\mathbb E}_{G}\left\{p(y\mid G)\,(G^2-1)\right\} = 0,
\end{equation} 
which, by definition \eqref{eq:defdeltaur}, immediately implies that $\delta_{\rm u}=\infty$. Note that this argument works when we substitute $\tanh^2(x)$ with any function which is even, increasing for $x\ge 0$ and bounded between $0$ and $1$.

Let us now show that $\delta_\ell$ is finite. Consider the function $f(m)$ defined in \eqref{eq:fnormr}. As previously mentioned, $f(0)=1$. Furthermore,
\begin{equation}
\begin{split}
f(1) &= \bigintssss_{\mathbb R}\frac{{\mathbb E}_{G}\left\{\left(p(y\mid G)\right)^2\right\}}{{\mathbb E}_{G}\left\{p(y\mid G)\right\}} \,{\rm d}y \\
&= \bigintssss_{\mathbb R}\frac{\left({\mathbb E}_{G}\left\{p(y\mid G)\right\}\right)^2+{\rm Var}\left\{p(y\mid G)\right\}}{{\mathbb E}_{G}\left\{p(y\mid G)\right\}} \,{\rm d}y\\
&= 1+\bigintssss_{\mathbb R}\frac{{\rm Var}\left\{p(y\mid G)\right\}}{{\mathbb E}_{G}\left\{p(y\mid G)\right\}} \,{\rm d}y>1.
\end{split}
\end{equation}  
Consequently, there exists $m_*\in (0, 1)$ such that $f(m_*)>1$. Set 
\begin{equation}
\delta^* = -\frac{\log(1-m_*^2)}{2\log f(m_*)}+1.
\end{equation}
Then, we have that, for any $\delta \ge \delta^*$, 
\begin{equation}
F_{\delta}(m_*) \ge F_{\delta_*}(m_*) = 1 > 0.
\end{equation}
Hence, by definition \eqref{eq:defdeltar}, we conclude that $\delta_{\ell}<\delta^*$, which implies that $\delta_\ell$ is
finite. Note that this upper bound on $\delta_{\ell}$  applies to any $p(y\mid G)$ which is not constant in $G$ on a set of positive measure. 
As a result, there is a strictly positive gap between $\delta_\ell$ and $\delta_{\rm u}$.\footnote{This gap is not due to the looseness of our lower bound. Indeed, by using the result of \cite{barbier2017phase}, one can show that the actual information-theoretic threshold is finite.}
\end{remark}

\section{Proof of Theorems \ref{th:lower} and \ref{th:lowerr}: Information-Theoretic Lower Bound} \label{sec:mainproof}

\subsection{Complex Case} \label{subsec:proofcompl}

The crucial point of the proof consists in the computation of the conditional entropy $H(\bY\mid \bA)$, which is contained in Lemma \ref{lemma:ent}. Then, we use this result to compute the mutual information for the considered model. Finally, we provide the proof of Theorem \ref{th:lower}. 

\begin{lemma}[Conditional Entropy]\label{lemma:ent}
Let $\bx \sim \Unif(\sqrt{d}\Sphere_{\mathbb C}^{d-1})$, $\bA = (\ba_1, \ldots, \ba_n)$ with $\{\ba_i\}_{1\le i\le n}\sim_{i.i.d.}\cnormal({\bm 0}_d,\id_d/d)$, and $\by = (y_1, \ldots, y_n)$ with $y_i\sim p(\cdot \mid |g_i|)$ and $g_i=\langle \bx, \ba_i \rangle$. Let $n/d\to \delta$ and define $\delta_{\ell}$ as in \eqref{eq:defdelta}. Then, for any $\delta < \delta_{\ell}$, we have that
\begin{equation}
\lim_{n\to \infty}\frac{1}{n}H(\bY\mid \bA) = H(Y_1).
\end{equation}
\end{lemma}

\begin{proof}
We divide the proof into two steps. The \emph{first step} consists in showing that  
\begin{equation}\label{eq:firststep}
-\frac{1}{n}\left(\bigintsss_{\!\mathbb R^d}\frac{{\mathbb E}_{\bA}\left\{\left(p(\by\mid \bA)\right)^2\right\}}{{\mathbb E}_{\bA}\left\{p(\by\mid \bA)\right\}} \,{\rm d}\by-1\right)\le \frac{1}{n}H(\bY\mid \bA)-H(Y_1) \le 0,
\end{equation}
which holds for all $n\in \mathbb N$ and for all $\delta > 0$. The proof of \eqref{eq:firststep} does not require any assumption on the distribution of $\bx$ and on the distribution of $\{\ba_i\}_{1\le i\le n}$ (as long as the vectors $\{\ba_i\}_{1\le i\le n}$ are independent).

The \emph{second step} consists in showing that 
\begin{equation}\label{eq:secondstep}
\lim_{n\to +\infty}\frac{1}{n}\left(\bigintsss_{\!\mathbb R^d}\frac{{\mathbb E}_{\bA}\left\{\left(p(\by\mid \bA)\right)^2\right\}}{{\mathbb E}_{\bA}\left\{p(\by\mid \bA)\right\}} \,{\rm d}\by-1\right)=0.
\end{equation}
It is clear that \eqref{eq:firststep} and \eqref{eq:secondstep} imply the thesis. 

\paragraph*{First step.}
By definition of conditional entropy, we have that
\begin{equation}\label{eq:condentr}
\frac{1}{n}H(\bY\mid \bA) = \frac{1}{n} \int {\mathbb E}_{\bA}\left\{-p(\by\mid \bA)\log p(\by\mid \bA)\right\}\,{\rm d}\by.
\end{equation}
By using the definition of $y_i$ and the fact that they are independent, we can rewrite ${\mathbb E}_{\bA}\left\{p(\by\mid \bA)\right\}$ as follows:
\begin{equation}\label{eq:Ea}
\begin{split}
{\mathbb E}_{\bA}\left\{p(\by\mid \bA)\right\} &= {\mathbb E}_{\bA, \bX}\left\{p(\by\mid \bA, \bX)\right\} \\
&= {\mathbb E}_{\bA, \bX}\left\{\prod_{i=1}^n p(y_i\mid |\langle\bX, \bA_i\rangle|)\right\}\\
&=\prod_{i=1}^n {\mathbb E}_{G_i}\left\{ p(y_i\mid |G_i|)\right\},
\end{split}
\end{equation}
where we set $G_{i} = \langle\bX, \bA_i\rangle$.

Let us now give an upper bound on the RHS of \eqref{eq:condentr}:
\begin{equation*}
\begin{split}
\frac{1}{n}\int_{\mathbb R^d} {\mathbb E}_{\bA}\left\{-p(\by\mid \bA)\log p(\by\mid \bA)\right\}\,{\rm d}\by
& \stackrel{\mathclap{\mbox{\footnotesize(a)}}}{\le} \frac{1}{n}\int_{\mathbb R^d} -{\mathbb E}_{\bA}\left\{p(\by\mid \bA)\right\} \log{\mathbb E}_{\bA}\left\{p(\by\mid \bA)\right\}\,{\rm d}\by\\
&\stackrel{\mathclap{\mbox{\footnotesize(b)}}}{=} \frac{1}{n}\int_{\mathbb R^d} -\prod_{i=1}^n {\mathbb E}_{G_i}\left\{ p(y_i\mid |G_i|)\right\} \sum_{j=1}^n\log{\mathbb E}_{G_j}\left\{ p(y_j\mid G_j)\right\}\,{\rm d}\by\\
&= \frac{1}{n}\sum_{i=1}^n\int_{\mathbb R}  -{\mathbb E}_{G_i}\left\{ p(y_i\mid |G_i|)\right\} \log{\mathbb E}_{G_i}\left\{ p(y_i\mid |G_i|)\right\}\,{\rm d}y_i\\
&=H(Y_1),\\
\end{split}
\end{equation*}
where in (a) we apply Jensen's inequality as the function $g(x)=-x\log x$ is concave, and in (b) we use \eqref{eq:Ea}. This immediately implies that 
\begin{equation}\label{eq:ub}
\frac{1}{n}H(\bY\mid \bA) - H(Y_1)\le 0.
\end{equation}

Note that the upper bound \eqref{eq:ub} is based on the inequality $${\mathbb E}_{\bA}\left\{-p(\by\mid \bA)\log p(\by\mid \bA)\right\} - \left( -{\mathbb E}_{\bA}\left\{p(\by\mid \bA)\right\} \log{\mathbb E}_{\bA}\left\{p(\by\mid \bA)\right\}\right)\le 0.$$Let us now find a lower bound to this quantity:
\begin{equation}\label{eq:aux1}
\begin{split}
{\mathbb E}_{\bA}  \left\{-p(\by\mid \bA)\log p(\by\mid \bA)\right\} & - \left( -{\mathbb E}_{\bA}\left\{p(\by\mid \bA)\right\} \log{\mathbb E}_{\bA}\left\{p(\by\mid \bA)\right\}\right) \\
&\hspace{7em}={ \mathbb E}_{\bA}\left\{-p(\by\mid \bA)\log \frac{p(\by\mid \bA)}{{\mathbb E}_{\bA}\left\{p(\by\mid \bA)\right\}}\right\}\\
&\hspace{7em}\stackrel{\mathclap{\mbox{\footnotesize(a)}}}{=}{\mathbb E}_{\bA}\left\{p(\by\mid \bA)\right\} {\mathbb E}_{\bZ} \left\{-\bZ\log \bZ\right\}\\
&\hspace{7em}\stackrel{\mathclap{\mbox{\footnotesize(b)}}}{=}{\mathbb E}_{\bA}\left\{p(\by\mid \bA)\right\} {\mathbb E}_{\bZ} \left\{-\bZ\log \bZ+\bZ-1\right\}\\
&\hspace{7em}\stackrel{\mathclap{\mbox{\footnotesize(c)}}}{\ge}-{\mathbb E}_{\bA}\left\{p(\by\mid \bA)\right\} {\mathbb E}_{\bZ} \left\{(\bZ-1)^2\right\}\\
&\hspace{7em}\stackrel{\mathclap{\mbox{\footnotesize(d)}}}{=}-{\mathbb E}_{\bA}\left\{p(\by\mid \bA)\right\} \left({\mathbb E}_{\bZ} \left\{\bZ^2\right\}-1\right)\\
&\hspace{7em}=-\left(\frac{{\mathbb E}_{\bA}\left\{\left(p(\by\mid \bA)\right)^2\right\}}{{\mathbb E}_{\bA}\left\{p(\by\mid \bA)\right\}}-{\mathbb E}_{\bA}\left\{p(\by\mid \bA)\right\} \right),\\
\end{split}
\end{equation}
where in (a) we set $\bZ = p(\by\mid \bA)/{\mathbb E}_{\bA}\left\{p(\by\mid \bA)\right\}$, in (b) we use that ${\mathbb E}_{\bZ} \left\{\bZ\right\}=1$, in (c) we use that $-z\log z+z -1\ge -(z-1)^2$ for any $z\ge 0$, and in (d) we use again that ${\mathbb E}_{\bZ} \left\{\bZ\right\}=1$. Therefore,
\begin{equation*}
\begin{split}
\frac{1}{n}H(\bY\mid \bA) - H(Y_1) &= \frac{1}{n}\int ( {\mathbb E}_{\bA}\left\{-p(\by\mid \bA)\log p(\by\mid \bA)\right\} - \left( -{\mathbb E}_{\bA}\left\{p(\by\mid \bA)\right\} \log{\mathbb E}_{\bA}\left\{p(\by\mid \bA)\right\}\right))\,{\rm d}\by\\
&\stackrel{\mathclap{\mbox{\footnotesize(a)}}}{\ge} -\frac{1}{n}\bigintsss_{\!\mathbb R^d} \left(\frac{{\mathbb E}_{\bA}\left\{\left(p(\by\mid \bA)\right)^2\right\}}{{\mathbb E}_{\bA}\left\{p(\by\mid \bA)\right\}}-{\mathbb E}_{\bA}\left\{p(\by\mid \bA)\right\} \right) \,{\rm d}\by,\\
&\stackrel{\mathclap{\mbox{\footnotesize(b)}}}{=} -\frac{1}{n}\left(\bigintsss_{\!\mathbb R^d}\frac{{\mathbb E}_{\bA}\left\{\left(p(\by\mid \bA)\right)^2\right\}}{{\mathbb E}_{\bA}\left\{p(\by\mid \bA)\right\}} \,{\rm d}\by-1\right),
\end{split}
\end{equation*}
where in (a) we use \eqref{eq:aux1} and in (b) we use that the integral of $p(\by\mid \bA)$ is $1$. This concludes the proof of \eqref{eq:firststep}.

\paragraph*{Second step.}
As $\bX\sim\Unif(\sqrt{d}\Sphere_{\mathbb C}^{d-1})$ and $\bA_i\sim\cnormal({\bm 0}_d,\id_d/d)$, we have that
\begin{equation*}
\{G_i\}_{1\le i\le n}\sim_{i.i.d.} \cnormal\left(0, 1\right).
\end{equation*}

Let us rewrite the quantity ${\mathbb E}_{\bA}\left\{\left(p(\by\mid \bA)\right)^2\right\}$ as follows:
\begin{equation}\label{eq:expAsquared}
\begin{split}
{\mathbb E}_{\bA}\left\{\left(p(\by\mid \bA)\right)^2\right\} &= {\mathbb E}_{\bA}\left\{\left({\mathbb E}_{\bX}\left\{\prod_{i=1}^n p(y_i\mid |\langle\bX, \bA_i\rangle|)\right\}\right)^2\right\}\\
&\stackrel{\mathclap{\mbox{\footnotesize(a)}}}{=}  {\mathbb E}_{\bA}\left\{{\mathbb E}_{\bX_1, \bX_2}\left\{\prod_{i=1}^n p(y_i\mid |\langle\bX_1, \bA_i\rangle|)\cdot p(y_i\mid |\langle\bX_2, \bA_i\rangle|)\right\}\right\}\\
&\stackrel{\mathclap{\mbox{\footnotesize(b)}}}{=} {\mathbb E}_{C}\left\{\prod_{i=1}^n {\mathbb E}_{G_{i, 1}, G_{i, 2}}\left\{ p(y_i\mid |G_{i, 1}|)\cdot p(y_i\mid |G_{i, 2}|)\right\}\right\},\\
\end{split}
\end{equation}
where in (a) $\bX_1$ and $\bX_2$ are independent, and in (b) we set $G_{i, 1} = \langle\bX_1, \bA_i\rangle$, $G_{i, 2} = \langle\bX_2, \bA_i\rangle$, and  
\begin{equation*}
C = \frac{\langle \bX_1, \bX_2\rangle}{\norm{\bX_1}_2\norm{\bX_2}_2}.
\end{equation*}
Then, given $C=c$, as $\bX_1, \bX_2\sim_{i.i.d.}\Unif(\sqrt{d}\Sphere_{\mathbb C}^{d-1})$ and $\bA_i\sim\cnormal(\b0_d,\id_d/d)$, we have that
\begin{equation*}
\{(G_{i, 1}, G_{i, 2})\}_{1\le i\le n}\sim_{i.i.d.} \cnormal\left(\b0_2, \left[\begin{array}{ll}
1 & c \\ c^* & 1\\ 
\end{array}\right]\right).
\end{equation*}
Hence,
\begin{equation*}
\begin{split}
\frac{1}{n}\bigintsss_{\!\mathbb R^d}\frac{{\mathbb E}_{\bA}\left\{\left(p(\by\mid \bA)\right)^2\right\}}{{\mathbb E}_{\bA}\left\{p(\by\mid \bA)\right\}} \,{\rm d}\by& \stackrel{\mathclap{\mbox{\footnotesize(a)}}}{=} \frac{1}{n}\bigintssss_{\mathbb R^d}{\mathbb E}_{C}\left\{\prod_{i=1}^n\frac{{\mathbb E}_{G_{i, 1}, G_{i, 2}}\left\{p(y\mid |G_{i, 1}|)p(y\mid |G_{i, 2}|)\right\}}{{\mathbb E}_{G_i}\left\{p(y\mid |G_i|)\right\}}\right\} \,{\rm d}\by\\
&=\frac{1}{n}{\mathbb E}_C\left\{\prod_{i=1}^n\bigintssss_{\mathbb R}\frac{{\mathbb E}_{G_{i, 1}, G_{i, 2}}\left\{p(y\mid |G_{i, 1}|)p(y\mid |G_{i, 2}|)\right\}}{{\mathbb E}_{G_i}\left\{p(y\mid |G_i|)\right\}} \,{\rm d}y_i\right\}\\
&\stackrel{\mathclap{\mbox{\footnotesize(b)}}}{=}\frac{1}{n}{\mathbb E}_M \left\{(f(M))^n\right\}\\
&\stackrel{\mathclap{\mbox{\footnotesize(c)}}}{=}\frac{d-1}{n}\displaystyle\int_0^1 (f(m))^n(1-m)^{d-2}\,{\rm d}m,\\
\end{split}
\end{equation*} 
where in (a) we use \eqref{eq:Ea} and \eqref{eq:expAsquared}, in (b) we use the fact that $f$ depends only on $m=|c|^2$, which is clear from the explicit expression provided by Lemma \ref{lemma:expl} contained in Appendix \ref{app:noiseless}, and in (c) we use that $M\sim {\rm Beta}(1, d-1)$ by Lemma \ref{lemma:beta} contained in Appendix \ref{app:distribution}. 

Set $d'=d-2$ and $\delta_n'=n/d'$. Thus,
\begin{equation}\label{eq:integral}
\int_0^1 (f(m))^n(1-m)^{d-2}\,{\rm d}m = \int_0^1 \exp\left(n\cdot F_{\delta_n'}(m)\right)\,{\rm d}m,
\end{equation}
where $F_{\delta_n'}(m)$ is given by \eqref{eq:defF}. Define
\begin{equation}
\tilde{F}_{\delta}(m) = \delta \max(\log f(m), 0) + \log(1-m).
\end{equation}
As $\delta < \delta_{\ell}$ and $n/d'\to \delta$, there exists $\delta_*\in(\delta,\delta_{\ell})$ such that $\delta_n'<\delta_*$ for $n$ sufficiently large. As $F_{\delta_n'}(m)\le \tilde{F}_{\delta_n'}(m)$ and $\tilde{F}_{\delta}(m)$ is non-decreasing in $\delta$, we have that 
\begin{equation}\label{eq:integral2}
\int_0^1 \exp\left(n\cdot F_{\delta_n'}(m)\right)\,{\rm d}m \le \int_0^1 \exp\left(n\cdot \tilde{F}_{\delta_*}(m)\right)\,{\rm d}m.
\end{equation}
Note that $\tilde{F}_{\delta_*}(m)<0$ if and only if $F_{\delta_*}(m)<0$. Thus, by definition of $\delta_{\ell}$, we have that $\tilde{F}_{\delta_*}(m)<0$ for $m\in (0, 1]$ when $n$ is sufficiently large. Furthermore, $\tilde{F}_{\delta_*}(0)=0$ and $\tilde{F}_{\delta_*}$ is a continuous function. As a result, by Lemma \ref{lemma:laplace}, the integral in \eqref{eq:integral2} tends to $0$ as $n\to\infty$ and the claim immediately follows. 
\end{proof}

\begin{remark}[Mutual Information]
An immediate consequence of Lemma \ref{lemma:ent} is that one can compute the mutual information $I(\bX; \bY, \bA)$ for any $\delta < \delta_{\ell}$:
\begin{equation}
\lim_{n\to +\infty}\frac{1}{n}I(\bX; \bY, \bA) = H\left({\mathbb E}_G\left\{p(\cdot\mid |G|)\right\}\right)-{\mathbb E}_G \left\{H(p(\cdot\mid |G|))\right\},
\end{equation}
where $G\sim \cnormal(0, 1)$.
\end{remark}

\begin{proof}[Proof of Theorem \ref{th:lower}]
Define $\by_{1:n}=(y_1, \ldots, y_n)$ and $\ba_{1:n} = (\ba_1, \ldots, \ba_n)$. We divide the proof into two steps. The \emph{first step} consists in showing that the mutual information between the next observation $y_{n+1}$ and the previous observations $\by_{1:n}$ tends to $0$. More formally, we will prove that 
\begin{equation}\label{eq:1step}
I(Y_{n+1}; \bY_{1:n}, \bA_{1:n}\mid \bA_{n+1}) =  o_n(1).
\end{equation}

The \emph{second step} consists in showing that the estimate obtained on $\phi(|\langle \bx, \ba_{n+1}\rangle|)$ given the observations $\by_{1:n}$ is similar to the estimate on $\phi(|\langle \bx, \ba_{n+1}\rangle|)$ when no observation is available. This means that the observations $\by_{1:n}$ do not provide any help. More formally, we will prove that
\begin{equation}\label{eq:2step}
{\mathbb E}_{\bY_{1:n}, \bA_{1:n+1}}\biggl\{\Bigl({\mathbb E}\bigl\{\phi(|\langle \bX, \bA_{n+1}\rangle|)\bigr\} - {\mathbb E}\bigl\{\phi(|\langle \bX, \bA_{n+1}\rangle|)\,\bigl\lvert \,\bY_{1:n}, \bA_{1:n}\bigr\}\Bigr)^2\biggr\}=o_n(1),
\end{equation}
where $\phi$ is defined in \eqref{eq:defphi}.

Furthermore, we have that 
\begin{equation}\label{eq:3step}
\begin{split}
\mathbb E \biggl\{\Bigl(\phi(|\langle \bX, \bA_{n+1}\rangle|) &- {\mathbb E}\bigl\{\phi(|\langle \bX, \bA_{n+1}\rangle|)\,\bigl\lvert\, \bY_{1:n}, \bA_{1:n}\bigr\}\Bigr)^2\biggr\} \\
&\hspace{-2em}- \Bigl({\mathbb E}\bigl\{\phi(|\langle \bX, \bA_{n+1}\rangle|)\bigr\} - {\mathbb E}\bigl\{\phi(|\langle \bX, \bA_{n+1}\rangle|)\,\bigl\lvert\, \bY_{1:n}, \bA_{1:n}\bigr\}\Bigr)^2  \\
&\hspace{2cm}= \mathbb E \biggl\{\Bigl(\phi(|\langle \bX, \bA_{n+1}\rangle|) - {\mathbb E}\bigl\{\phi(|\langle \bX, \bA_{n+1}\rangle|)\bigr\} \Bigr)^2\biggr\}\\
&\hspace{2cm}={\rm Var}\bigl\{\phi(|\langle \bX, \bA_{n+1}\rangle|)\bigr\}.
\end{split}
\end{equation}
By applying \eqref{eq:2step} and \eqref{eq:3step}, the proof of Theorem \ref{th:lower} follows. 

\paragraph*{First step.} By using the chain rule of entropy and that $y_i$ is independent from $\ba_{i+1:n+1}$, we obtain that 
\begin{equation*}
\begin{split}
\frac{1}{n+1}H(\bY_{1:n+1}\mid \bA_{1:n+1}) &= \frac{1}{n+1}\sum_{i=1}^{n+1} H(Y_i\mid \bY_{1:i-1}, \bA_{1:n+1})\\
&= \frac{1}{n+1}\sum_{i=1}^{n+1} H(Y_i\mid \bY_{1:i-1}, \bA_{1:i}).
\end{split}
\end{equation*}
The sequence $s_n = H(Y_n\mid \bY_{1:n-1}, \bA_{1:n})$ is decreasing, as conditioning reduces entropy. Hence $s_n$ has a limit, and this limit must be equal to $H(Y_1)$ by Lemma \ref{lemma:ent}. Since the $Y_i$ are i.i.d., we obtain that
\begin{equation*}
H(Y_{n+1}\mid \bY_{1:n}, \bA_{1:n+1}) = H(Y_{n+1}) + o_n(1).
\end{equation*} 
By using again that conditioning reduces entropy, we also obtain that 
\begin{equation*}
H(Y_{n+1}\mid \bA_{n+1}) = H(Y_{n+1}) + o_n(1).
\end{equation*} 
By putting these last two equations together, we deduce that \eqref{eq:1step} holds. 

\paragraph*{Second step.} Given two probability distributions $p$ and $q$, let $D_{\rm KL}(p || q)$ and $\norm{p- q}_{\rm TV}$ denote their Kullback-Leibler divergence and their total variation distance, respectively. Then,
\begin{equation}\label{eq:distances}
\begin{split}
I(Y_{n+1};&  \bY_{1:n}, \bA_{1:n}\mid \bA_{n+1}) \\
&= {\mathbb E}_{\bY_{1:n}, \bA_{1:n+1}} \left\{D_{\rm KL}(p(y_{n+1}\mid \bY_{1:n}, \bA_{1:n+1})||p(y_{n+1}\mid \bA_{n+1}))\right\}\\
&\stackrel{\mathclap{\mbox{\footnotesize(a)}}}{\ge} \frac{1}{2}\cdot{\mathbb E}_{\bY_{1:n}, \bA_{1:n+1}}\left\{ \left(\norm{p(y_{n+1}\mid \bY_{1:n}, \bA_{1:n+1})- p(y_{n+1}\mid \bA_{n+1}))}_{\rm TV}\right)^2\right\}\\
&\stackrel{\mathclap{\mbox{\footnotesize(b)}}}{\ge} \frac{1}{2K^2}\cdot{\mathbb E}_{\bY_{1:n}, \bA_{1:n+1}}\Biggl\{ \biggl(\int_{\mathbb R}p(y_{n+1}\mid \bY_{1:n}, \bA_{1:n+1})\varphi(y_{n+1})\,{\rm d}y_{n+1}\\
&\hspace{9.5em}-\int_{\mathbb R}p(y_{n+1}\mid \bA_{n+1})\varphi(y_{n+1})\,{\rm d}y_{n+1}\biggr)^2\Biggr\}\\
&\stackrel{\mathclap{\mbox{\footnotesize(c)}}}{=} \frac{1}{2K^2}\cdot{\mathbb E}_{\bY_{1:n}, \bA_{1:n+1}}\Biggl\{ \biggl(\int_{\mathbb C^d}p(\bx\mid \bY_{1:n}, \bA_{1:n})\int_{\mathbb R}p(y_{n+1}\mid \bx, \bY_{1:n}, \bA_{1:n+1})\varphi(y_{n+1})\,{\rm d}y_{n+1}\,{\rm d}\bx\\
&\hspace{3.4cm}-\int_{\mathbb C^d}p(\bx)\int_{\mathbb R}p(y_{n+1}\mid \bx, \bA_{n+1})\varphi(y_{n+1})\,{\rm d}y_{n+1}\,{\rm d}\bx\biggr)^2\Biggr\}\\
&\stackrel{\mathclap{\mbox{\footnotesize(d)}}}{=} \frac{1}{2K^2}\cdot{\mathbb E}_{\bY_{1:n}, \bA_{1:n+1}}\left\{\left({\mathbb E}\left\{\phi(|\langle \bX, \bA_{n+1}\rangle|)\right\} - {\mathbb E}\left\{\phi(|\langle \bX, \bA_{n+1}\rangle|)\mid \bY_{1:n}, \bA_{1:n}\right\}\right)^2\right\}.\\
\end{split}
\end{equation}
where in (a) we use Pinsker's inequality, in (b) we use that $\varphi$ is bounded and we set $\norm{\varphi}_{\infty}=K$, in (c) we use that $\bX$ and $\bA_{n+1}$ are independent, and in (d) we use the definition \eqref{eq:defphi}. By combining \eqref{eq:1step} and \eqref{eq:distances}, \eqref{eq:2step} immediately follows. 
\end{proof}

\subsection{Real Case} \label{subsec:proofreal}

The proof is very similar to the one provided in Section \ref{subsec:proofcompl} for the complex case. In particular, the crucial point consists in showing that 
\begin{equation}\label{eq:lemmaHr}
\lim_{n\to \infty}\frac{1}{n}H(\bY\mid \bA) = H(Y_1),
\end{equation}
where $\bx \sim \Unif(\sqrt{d}\Sphere_{\mathbb R}^{d-1})$, $\ba = (\ba_1, \ldots, \ba_n)$ with $\{\ba_i\}_{1\le i\le n}\sim_{i.i.d.}\normal({\bm 0}_d,\id_d/d)$, and $\by = (y_1, \ldots, y_n)$ with $y_i\sim p(\cdot \mid g_i)$ and $g_i=\langle \bx, \ba_i \rangle$.
Then, the proof of Theorem \ref{th:lowerr} follows similar passages as the proof of Theorem \ref{th:lower}. 

In order to prove \eqref{eq:lemmaHr}, we show that \eqref{eq:firststep} and \eqref{eq:secondstep} hold. The proof of \eqref{eq:firststep} follows the same passages as the first step of the proof of Lemma \ref{lemma:ent}, hence it is omitted. The proof of \eqref{eq:secondstep} is slightly different and we detail what changes in the remaining part of this section.

Similarly to \eqref{eq:Ea}, we have that 
\begin{equation*}
{\mathbb E}_{\bA}\left\{p(\by\mid \bA)\right\} = \prod_{i=1}^n {\mathbb E}_{G_i}\left\{ p(y_i\mid G_i)\right\},
\end{equation*}
where $G_{i} = \langle\bX, \bA_i\rangle\sim \normal(0, 1)$. Furthermore, similarly to \eqref{eq:expAsquared}, we also have that 
\begin{equation*}
{\mathbb E}_{\bA}\left\{\left(p(\by\mid \bA)\right)^2\right\} ={\mathbb E}_{M}\left\{\prod_{i=1}^n {\mathbb E}_{G_{i, 1}, G_{i, 2}}\left\{ p(y_i\mid G_{i, 1})\cdot p(y_i\mid G_{i, 2})\right\}\right\},\\
\end{equation*}
where $G_{i, 1} = \langle\bX_1, \bA_i\rangle$, $G_{i, 2} = \langle\bX_2, \bA_i\rangle$, and we define
\begin{equation*}
M = \frac{\langle \bX_1, \bX_2\rangle}{\norm{\bX_1}_2\norm{\bX_2}_2}.
\end{equation*}
Then, given $M=m$, as $\bX_1, \bX_2\sim_{i.i.d.}\Unif(\sqrt{d}\Sphere_{\mathbb R}^{d-1})$ and $\bA_i\sim\normal(\b0_d,\id_d/d)$, we have that
\begin{equation*}
\{(G_{i, 1}, G_{i, 2})\}_{1\le i\le n}\sim_{i.i.d.} \normal\left(\b0_2, \left[\begin{array}{ll}
1 & m \\ m & 1\\ 
\end{array}\right]\right).
\end{equation*}
Hence,
\begin{equation}\label{eq:intrproof}
\begin{split}
\frac{1}{n}\bigintsss_{\!\mathbb R^d}\frac{{\mathbb E}_{\bA}\left\{\left(p(\by\mid \bA)\right)^2\right\}}{{\mathbb E}_{\bA}\left\{p(\by\mid \bA)\right\}} \,{\rm d}\by &\stackrel{\mathclap{\mbox{\footnotesize(a)}}}{=}\frac{1}{n}{\mathbb E}_M \left\{(f(M))^n\right\}\\
&\stackrel{\mathclap{\mbox{\footnotesize(b)}}}{=}\frac{1}{n}\frac{\Gamma(\frac{d}{2})}{\sqrt{\pi}\Gamma(\frac{d-1}{2})}\displaystyle\int_{-1}^1 (f(m))^n(1-m^2)^{\frac{d-3}{2}}\,{\rm d}m,\\
\end{split}
\end{equation} 
where in (a) we use the definition \eqref{eq:defFr} of $f$ and in (b) we plug in the distribution of $M$ obtained from Lemma \ref{lemma:betar} contained in Appendix \ref{app:distribution}. Note that
\begin{equation*}
\lim_{d\to \infty}\frac{\Gamma(\frac{d}{2})}{\frac{d}{2}\cdot \Gamma(\frac{d-1}{2})}=1.
\end{equation*}
Therefore, by showing that the integral in the RHS of \eqref{eq:intrproof} tends to $0$, the claim immediately follows. 

Set $d'=d-3$ and $\delta_n'=n/d'$. Thus,
\begin{equation}\label{eq:integralr}
\int_{-1}^1 (f(m))^n(1-m^2)^{\frac{d-3}{2}}\,{\rm d}m = \int_{-1}^1 \exp\left(n\cdot F_{\delta_n'}(m)\right)\,{\rm d}m,
\end{equation}
where $F_{\delta_n'}(m)$ is defined in \eqref{eq:defFr}. Define
\begin{equation}
\tilde{F}_{\delta}(m) = \delta \max(\log f(m), 0) +\frac{1}{2} \log(1-m^2).
\end{equation}
As $\delta < \delta_{\ell}$ and $n/d'\to \delta$, there exists $\delta_*\in(\delta,\delta_{\ell})$ such that $\delta_n'<\delta_*$ for $n$ sufficiently large. As $F_{\delta_n'}(m)\le \tilde{F}_{\delta_n'}(m)$ and $\tilde{F}_{\delta}(m)$ is non-decreasing in $\delta$, we have that 
\begin{equation}\label{eq:integral2r}
\int_0^1 \exp\left(n\cdot F_{\delta_n'}(m)\right)\,{\rm d}m \le \int_0^1 \exp\left(n\cdot \tilde{F}_{\delta_*}(m)\right)\,{\rm d}m.
\end{equation}
Note that $\tilde{F}_{\delta_*}(m)<0$ if and only if $F_{\delta_*}(m)<0$. Thus, by definition of $\delta_{\ell}$, we have that $\tilde{F}_{\delta_*}(m)<0$ for $m\neq 0$ when $n$ is sufficiently large. Furthermore, $\tilde{F}_{\delta_*}(0)=0$ and $\tilde{F}_{\delta_*}$ is a continuous function. As a result, by Lemma \ref{lemma:laplace}, the integral in \eqref{eq:integral2r} tends to $0$ as $n\to\infty$ and the claim immediately follows.

\section{Proof of Theorems \ref{th:upper} and \ref{th:upperr}: Spectral Upper Bound} \label{sec:spectralproof}

We will consider the complex case. The proof for the real case is essentially the same and it is briefly discussed in Remark \ref{rmk:realub} at the end of this section. 

A crucial ingredient of the proof consists in Lemma \ref{lemma:condub}, which is a generalization of Theorem 1 of \cite{lulispectral_arxiv}. Before stating this result, we need some definitions. Let $G\sim \cnormal(0, 1)$, $Y\sim p(\cdot \mid |G|)$, and $Z=\mathcal T(Y)$. Assume that $Z$ has bounded support and let $\tau$ be the supremum of this support, i.e., 
\begin{equation}\label{eq:deftau}
\tau = \inf\{ z : \mathbb P (Z\le z) = 1\}.
\end{equation}
For $\lambda\in (\tau, \infty)$ and $\delta\in (0, \infty)$, define
\begin{equation}\label{eq:defphi2}
\phi(\lambda) = \lambda \cdot {\mathbb E}\left\{\frac{Z\cdot|G|^2}{\lambda-Z}\right\},
\end{equation}
and 
\begin{equation}\label{eq:defpsi}
\psi_{\delta}(\lambda) = \lambda\left(\frac{1}{\delta}+{\mathbb E}\left\{\frac{Z}{\lambda-Z}\right\}\right).
\end{equation}
Note that $\phi(\lambda)$ is a monotone non-increasing function and that $\psi_{\delta}(\lambda)$ is a convex function. Let $\bar{\lambda}_{\delta}$ be the point at which $\psi_{\delta}$ attains its minimum, i.e.,
\begin{equation}\label{eq:minpsi}
\bar{\lambda}_\delta = \arg\min_{\lambda\ge \tau} \psi_{\delta}(\lambda).
\end{equation}
For $\lambda\in (\tau, \infty)$, define also
\begin{equation}\label{eq:defzeta}
\zeta_{\delta}(\lambda) = \psi_{\delta}(\max(\lambda, \bar{\lambda}_\delta)).
\end{equation}

\begin{lemma}[Generalization of Theorem 1 of \cite{lulispectral_arxiv}]\label{lemma:condub}
Let $\bx \sim \Unif(\Sphere_{\mathbb C}^{d-1})$, $\{\ba_i\}_{1\le i\le n}\sim_{i.i.d.}\cnormal({\bm 0}_d,\id_d)$, and $\by$ be distributed according to \eqref{eq:defy}. Let $n/d\to \delta$, $G\sim \cnormal(0, 1)$ and define $Z=\mathcal T(Y)$
for $Y\sim p(\,\cdot\,|\,|G|)$. Assume that $Z$ satisfies $\mathbb P(Z=0)<1$ and that it has bounded support. Let $\tau$ be defined in \eqref{eq:deftau}. Assume further that, as $\lambda$ approaches $\tau$ from the right, we have
\begin{equation}\label{eq:hplemmaub1}
\lim_{\lambda\to \tau^+}{\mathbb E}\left\{\frac{Z}{(\lambda-Z)^2}\right\}=\lim_{\lambda\to \tau^+}{\mathbb E}\left\{\frac{Z\cdot|G|^2}{\lambda-Z}\right\}=\infty.
\end{equation}
Let $\hat{\bx}$ be the principal eigenvector of the matrix $\bD_n$, defined as in \eqref{eq:defD}. Then, the following results hold:
\begin{enumerate}[(1)]

	\item The equation
	\begin{equation}
\zeta_{\delta}(\lambda) = \phi(\lambda) 	
	\end{equation}
	admits a unique solution, call it $\lambda_{\delta}^*$, for $\lambda > \tau$. 

	\item As $n \to \infty$,
	\begin{equation}\label{eq:numpred}
\frac{|\langle\hat{\bx}, \bx\rangle|^2}{\norm{\hat{\bx}}_2^2 \, \norm{\bx}_2^2} \stackrel{\mathclap{\mbox{\footnotesize a.s.}}}{\longrightarrow} \left\{\begin{array}{ll}
\vspace{1em}
0, & \mbox{ if }\psi_{\delta}'(\lambda_{\delta}^*)\le 0,\\
\displaystyle\frac{\psi_{\delta}'(\lambda_{\delta}^*)}{\psi_{\delta}'(\lambda_{\delta}^*)-\phi'(\lambda_{\delta}^*)}, & \mbox{ if }\psi_{\delta}'(\lambda_{\delta}^*)> 0,\\
\end{array}\right.
	\end{equation}		
where $\psi_{\delta}'$ and $\phi'$ denote the derivatives of these two functions.

	\item Let $\lambda_1^{\bD_n}\ge \lambda_2^{\bD_n}$ denote the two largest eigenvalues of $\bD_n$. Then, as $n\to \infty$,
	\begin{equation}
	\begin{split}
	\lambda_1^{\bD_n} &\stackrel{\mathclap{\mbox{\footnotesize a.s.}}}{\longrightarrow}\zeta_{\delta}(\lambda_{\delta}^*),\\ 
	\lambda_2^{\bD_n} &\stackrel{\mathclap{\mbox{\footnotesize a.s.}}}{\longrightarrow}\zeta_{\delta}(\bar{\lambda}_{\delta}).\\ 
	\end{split}	
	\end{equation}
\end{enumerate}
\end{lemma}

Before proceeding with the proof, we discuss these results in more detail and we describe in what sense Lemma \ref{lemma:condub} provides a generalization of Theorem 1 of \cite{lulispectral_arxiv}. 

\begin{remark}[Two different regimes]\label{rmk:diffreg}
The results of Lemma \ref{lemma:condub} imply that, according to the value of $\delta$, we can distinguish between two possible regimes.

On the one hand, suppose that $\phi(\bar{\lambda}_{\delta}) > \psi_{\delta}(\bar{\lambda}_{\delta})$. Recall that $\phi(\lambda)$ is non-increasing and that $\bar{\lambda}_{\delta}$ is the point in which $\psi_{\delta}(\lambda)$ attains its minimum. Thus, $\bar{\lambda}_{\delta} < \lambda_{\delta}^*$, which implies that $\psi_{\delta}'(\lambda_{\delta}^*)>0$ and that $\zeta_{\delta}(\lambda_{\delta}^*) > \zeta_{\delta}(\bar{\lambda}_{\delta})$. This means that the scalar product $|\langle\hat{\bx}, \bx\rangle|$ is bounded away from zero and that there is a strictly positive gap between the two largest eigenvalues of $\bD_n$. In this regime, the spectral method that outputs $\hat{\bx}$ solves the weak recovery problem and \eqref{eq:corrub} holds for some $\epsilon>0$.

On the other hand, suppose that $\phi(\bar{\lambda}_{\delta}) \le \psi_{\delta}(\bar{\lambda}_{\delta})$. Thus, $\bar{\lambda}_{\delta} \ge \lambda_{\delta}^*$, which implies that $\psi_{\delta}'(\lambda_{\delta}^*)\le 0$ and that $\zeta_{\delta}(\lambda_{\delta}^*) = \zeta_{\delta}(\bar{\lambda}_{\delta})$. In words, this means that the scalar product $|\langle\hat{\bx}, \bx\rangle|$ converges to zero and that there is no strictly positive gap between the two largest eigenvalues of $\bD_n$. In this regime, the spectral method that outputs $\hat{\bx}$ does not solve the weak recovery problem.
\end{remark}

\begin{remark}[Lemma \ref{lemma:condub} and Theorem 1 of \cite{lulispectral_arxiv}] 
Lemma \ref{lemma:condub} generalizes Theorem 1 of \cite{lulispectral_arxiv} in the following two regards:
\begin{itemize}
\item $\bx$ and $\{\ba_i\}_{1\le i\le n}$ are complex vectors, while Theorem 1 of \cite{lulispectral_arxiv} considers the real case;
\item $Z$ can also be negative, while Theorem 1 of \cite{lulispectral_arxiv} assumes that $Z\ge 0$.
\end{itemize}

The first generalization does not require additional work as the whole argument of \cite{lulispectral_arxiv} generalizes in the natural way to the complex case: Gaussian random variables become circularly-symmetric complex Gaussian random variables, transposes of vectors and matrices become conjugate transposes, squares become modulus squares, and so on. 

On the contrary, the second generalization is more challenging, as it requires the result of Lemma \ref{lemma:baiyaonopsd}, which is stated below and proved in Appendix \ref{app:baiyaonopsd}.

As a final observation, let us point out that Theorem 1 of \cite{lulispectral_arxiv} assumes also that ${\mathbb E}\left\{Z\cdot |G|^2\right\} > {\mathbb E}\left\{Z\right\}$. A careful check shows that this hypothesis is never used in the proof of that theorem, but it is required only in the proof of some additional results of \cite{lulispectral_arxiv}.   
\end{remark}

\begin{lemma}[Generalization of \cite{baiyaospike2012} to non-PSD matrices]\label{lemma:baiyaonopsd}
Consider the random matrix
\begin{equation}\label{eq:deflemma}
\bS_n = \frac{1}{n}\bU \bM_n\bU^*,
\end{equation}
where the entries of $\bU\in\mathbb C^{(d-1)\times n}$ are $\sim_{i.i.d.}\cnormal(0, 1)$, and $\bM_n\in\mathbb C^{n\times n}$ is independent of $\bU$. Let $\lambda_1^{\bM_n}$ denote the largest eigenvalue of $\bM_n$. Assume that the empirical spectral measure of the eigenvalues of $\bM_n$ almost surely converges weakly to the probability distribution $H$, where $H$ is the law of the random variable $Z$. Let $\Gamma_H$ be the support of $H$ and let $\tau$ be the supremum of $\Gamma_H$. Assume also that, as $n\to \infty$,
\begin{equation}
\begin{split}
\lambda_1^{\bM_n} &\stackrel{\mathclap{\mbox{\footnotesize a.s.}}}{\longrightarrow} \alpha_* \not \in \Gamma_H.\\
\end{split}
\end{equation}
Let $n/d\to\delta$, denote by $\lambda_1^{\bS_n}$ the largest eigenvalue of the matrix \eqref{eq:deflemma}, and define $\psi_{\delta}$ as in \eqref{eq:defpsi}. Then, as $n\to \infty$,
\begin{equation}\label{eq:baiyaonopsd}
\begin{split}
\lambda_1^{\bS_n} \stackrel{\mathclap{\mbox{\footnotesize a.s.}}}{\longrightarrow}\psi_{\delta}(\alpha_*), \quad \quad\hspace{0.4em}&\mbox{if } \psi_{\delta}'(\alpha_*) >0,\\
\lambda_1^{\bS_n} \stackrel{\mathclap{\mbox{\footnotesize a.s.}}}{\longrightarrow}\min_{\lambda>\tau}\psi_{\delta}(\lambda), \quad &\mbox{if } \psi_{\delta}'(\alpha_*) \le 0.
\end{split}
\end{equation}
\end{lemma}

\begin{proof}[Proof of Lemma \ref{lemma:condub}]
In this proof, we follow closely the approach detailed in Section III of \cite{lulispectral_arxiv}. First of all, let us write the matrix $\bD_n$ defined in \eqref{eq:defD} as
\begin{equation}\label{eq:Drewrite}
\bD_n = \frac{1}{n} \bA \bZ \bA^*,
\end{equation}
where $\bA = [\ba_1, \ldots, \ba_n]$, $\bZ$ is a diagonal matrix with entries $z_i = \mathcal T(y_i)$ for $i\in [n]$, the random variables $y_i$ are independent and distributed according to $p(\cdot \mid |g_i|)$, and $\{g_i\}_{1\le i \le n}\sim_{i.i.d.} \cnormal(0, 1)$. As the sensing vectors $\{\ba_i\}_{1\le i\le n}$ are drawn from the circularly-symmetric complex normal distribution, we can assume without loss of generality that $\bx= \be_1$, where $\be_1$ is the first element of the canonical basis of $\mathbb C^d$.

Consider a matrix $\bU \in \mathbb C^{(d-1)\times n}$ independent of $\{g_i\}_{1\le i \le n}$ and $\bZ$. Let the elements of $\bU$ be $\sim_{i.i.d.}\cnormal(0, 1)$. Define
\begin{equation}\label{eq:defP}
\bP_n = \frac{1}{n}\bU\bZ\bU^*, 
\end{equation}
and 
\begin{equation}\label{eq:defq}
\bq_n = \frac{1}{n} \bU\bv, 
\end{equation}
where $\bv = [z_1 g_1, \ldots, z_n g_n]^*$. Then, \eqref{eq:Drewrite} can be rewritten as
\begin{equation}\label{eq:Drewritebis}
\bD_n = \left[\begin{array}{cc}
a_n & \bq_n^* \\
\bq_n & \bP_n \\
\end{array}\right],
\end{equation}
where $a_n= \sum_{i=1}^n z_i|g_i|^2/n$ is a scalar that converges almost surely to ${\mathbb E}(Z\cdot |G|^2)$ as $n\to \infty$, with $G\sim \cnormal(0, 1)$. 

Next, consider a parametric family of matrices $\{\bP_n + \mu \bq_n\bq_n^*\}$ and let $L_n(\mu)$ denote their largest eigenvalues, i.e., 
\begin{equation*}
L_n(\mu) = \lambda_1(\bP_n + \mu \bq_n\bq_n^*).
\end{equation*}
The idea is to compute the largest eigenvalue of $\bD_n$, call it $\lambda_1^{\bD_n}$, and the scalar product between $\hat{\bX}$ and $\be_1$ via a fixed-point equation involving $L_n(\mu)$. 

To do so, we first need an intermediate result holding for any matrix $\bD$ that can be written in the form
\begin{equation*}
\bD = \left[\begin{array}{cc}
a & \bq^* \\
\bq & \bP \\
\end{array}\right],
\end{equation*}
where $a\in \mathbb R$, $\bP\in \mathbb C^{(d-1)\times (d-1)}$ is a Hermitian matrix and $\bq\in \mathbb C^{d-1}$ is such that $\norm{\bq}\neq 0$. Note that the matrix $\bD_n$ defined in \eqref{eq:defD} fulfills such requirements, since the matrix $\bP_n$ defined in \eqref{eq:defP} is Hermitian and $\bq_n$ defined in \eqref{eq:defq} is such that $\norm{\bq_n}\neq 0$ with high probability, as $\mathbb P(Z= 0)<1$.

Let $\lambda_1^{\bP}\ge \lambda_2^{\bP}\ge \cdots \ge \lambda_{d-1}^{\bP}$ be the set of eigenvalues of $\bP$, and let $\bw_1, \bw_2, \ldots, \bw_{d-1}$ be a corresponding set of eigenvectors. For $\lambda\in (\max\{\lambda_{i}^{\bP} : \langle\bq, \bw_i\rangle\neq 0\}, \infty)$, define
\begin{equation}
R(\lambda) = \bq^*(\bP-\lambda\bI)^{-1}\bq=\sum_{i=1}^{d-1}\frac{|\langle\bq, \bw_i\rangle|^2}{\lambda_i^{\bP}-\lambda}.
\end{equation} 
Note that $R(\lambda)$ increases monotonically from $-\infty$ to $0$. Hence, it admits an inverse, call it $R^{-1}(x)$, for $x<0$. Then, the maximum eigenvalue
$L(\mu)=\lambda_1(\bP+\mu\bq\bq^*)$ is given by
\begin{equation}\label{eq:equivlemma1}
L(\mu) = \max(R^{-1}(-1/\mu), \lambda_1^{\bP}).
\end{equation}
The proof of \eqref{eq:equivlemma1} is standard, cf. e.g. Lemma 1 in \cite{lulispectral_arxiv}. Note that $L(\mu)$ is a non-decreasing function such that $\lim_{\mu\to\infty}L(\mu)=\infty$. Indeed, by construction, $R^{-1}(-1/\mu)$ is strictly increasing and $$\lim_{\mu\to \infty}R^{-1}(-1/\mu)=\infty.$$ Furthermore, $L(\mu)$ is convex since it is the maximum of a set of linear functions, as
\begin{equation*}
L(\mu) = \lambda_1(\bP + \mu \bq\bq^*) = \max_{\bx : \norm{\bx}=1} \bx^*(\bP+\mu \bq \bq^*)\bx.
\end{equation*}

Let $\mu^*>0$ be the solution to the fixed-point equation
\begin{equation}\label{eq:fixedpoint}
\mu = (L(\mu)-a)^{-1}.
\end{equation}
This solution is unique, since $L(\mu)$ is a non-decreasing function with $\lim_{\mu\to\infty}L(\mu)=\infty$. Then, 
\begin{equation}\label{eq:prop2-1}
\lambda_1^{\bD} = L(\mu^*),
\end{equation}
and 
\begin{equation}\label{eq:prop2-2}
|\langle\hat{\bx}, \be_1\rangle|^2\in \left[\frac{\partial_{-}L(\mu^*)}{\partial_{-}L(\mu^*)+(1/\mu^*)^2}, \frac{\partial_{+}L(\mu^*)}{\partial_{+}L(\mu^*)+(1/\mu^*)^2}\right],
\end{equation}
where $\partial_{-}L(\mu^*)$ and $\partial_{+}L(\mu^*)$ denote the left and right derivative of $L(\mu)$, respectively. In particular, if $L(\mu)$ is differentiable at $\mu^*$, then
\begin{equation}\label{eq:prop2-3}
|\langle\hat{\bx}, \be_1\rangle|^2 =\frac{L'(\mu^*)}{L'(\mu^*)+(1/\mu^*)^2}.
\end{equation}
The proof of \eqref{eq:prop2-1}, \eqref{eq:prop2-2}, and \eqref{eq:prop2-3} uses the characterization \eqref{eq:equivlemma1} and it is analogous to the proof of Proposition 2 in \cite{lulispectral_arxiv}.

At this point, we need to compute $L_n(\mu)$ for the matrix $\bD_n$ defined in \eqref{eq:Drewritebis}. The eigenvalues of a low rank perturbation of a random matrix are studied in \cite{benaychpert2011}. However, we cannot apply those results, as $\bP_n$ and $\bq_n$ are dependent. Hence, we write
\begin{equation*}
\bP_n+\mu \bq_n \bq_n^* = \frac{1}{n} \bU \bM_n \bU^*,
\end{equation*}
where $\bM_n$ is independent of $\bU$ with
\begin{equation*}
\bM_n = \bZ + \frac{\mu}{n}\bv \bv^*.
\end{equation*}
We start by studying the spectrum of $\bM_n$. Let $\lambda_1^{\bM_n}\ge \lambda_2^{\bM_n}\ge \cdots \ge \lambda_{n}^{\bM_n}$ be the set of eigenvalues of $\bM_n$ and let 
\begin{equation*}
f^{\bM_n}=\frac{1}{n-1}\sum_{i=2}^{n} \delta_{\lambda_i^{\bM_n}}
\end{equation*}
be the empirical spectral measure of the last $n-1$ eigenvalues.

Then, standard interlacing theorems (see \cite[Section 4.3]{horn2012matrix}) yield that $f^{\bM_n}$ almost surely converges weakly to the probability law of $Z$. Furthermore, by using the characterization \eqref{eq:equivlemma1}, we can show that
\begin{equation}
\lambda_1^{\bM_n} \stackrel{\mathclap{\mbox{\footnotesize a.s.}}}{\longrightarrow}\lambda_\mu=Q^{-1}(1/\mu),
\end{equation}
where $Q^{-1}$ is the inverse of the function
\begin{equation*}
Q(\lambda) = {\mathbb E}\left\{\frac{Z^2\cdot |G|^2}{\lambda-Z}\right\}.
\end{equation*}
The proof of these results is the same as the proof of Proposition 3 in \cite{lulispectral_arxiv}. 

Note that $Q(\lambda)$ is defined for $\lambda\in(\tau, \infty)$, it is continuous and strictly decreasing with $Q(\infty)=0$. Furthermore, by hypothesis \eqref{eq:hplemmaub1}, we have that $\lim_{\lambda\to\tau^+}Q(\lambda)=\infty$. Thus, $Q(\lambda)$ admits an inverse and $Q^{-1}(1/\mu)$ is well-defined for all $\mu>0$.

Let us now consider the matrix $\frac{1}{n} \bU \bM_n \bU^*$. First, if $Z\ge 0$, then $\bM_n$ is positive semi-definite (PSD) and we can apply results from \cite{baiyaospike2012} to compute the limit of $L_n(\mu)$. If $\bM_n$ is not necessarily PSD, we use Lemma \ref{lemma:baiyaonopsd} with $\alpha_*=\lambda_{\mu}$ to conclude that 
\begin{equation}\label{eq:gen}
\begin{split}
L_n(\mu) \stackrel{\mathclap{\mbox{\footnotesize a.s.}}}{\longrightarrow}\psi_{\delta}(\lambda_\mu), \quad \quad\hspace{0.4em}&\mbox{if } \psi_{\delta}'(\lambda_\mu) >0,\\
L_n(\mu) \stackrel{\mathclap{\mbox{\footnotesize a.s.}}}{\longrightarrow}\min_{\lambda>\tau}\psi_{\delta}(\lambda), \quad &\mbox{if } \psi_{\delta}'(\lambda_\mu) \le 0.
\end{split}
\end{equation}

The remaining part of the proof follows the argument of Section III-D in \cite{lulispectral_arxiv}. For the sake of readability, we reproduce it below.

We start by proving the first claim of the lemma. For $n\ge 1$, let $\mu_n$ be the unique solution to the fixed-point equation \eqref{eq:fixedpoint}. Then,
\begin{equation*}
L_n(\mu_n)-1/\mu_n=a_n.
\end{equation*}
Now, fix any $\mu>0$. Then, by using the definition \eqref{eq:defzeta} and the fact that $\lambda_\mu=Q^{-1}(1/\mu)$, \eqref{eq:gen} immediately implies that, as $n\to\infty$, 
\begin{equation}\label{eq:asconv}
L_n(\mu) -1/\mu\stackrel{\mathclap{\mbox{\footnotesize a.s.}}}{\longrightarrow}\zeta_{\delta}(Q^{-1}(1/\mu))-1/\mu.
\end{equation}
Note that, as $n\to \infty$, $a_n \stackrel{\mathclap{\mbox{\footnotesize a.s.}}}{\longrightarrow}{\mathbb E}(Z\cdot |G|^2)$. Furthermore, as $L_n(\mu)$ and $\zeta_{\delta}(\mu)$ are non-decreasing, the two functions on both sides of \eqref{eq:asconv} are strictly increasing. Consequently, by Lemma 3 in Appendix E of \cite{lulispectral_arxiv}, we conclude that
\begin{equation}\label{eq:convmustar}
\mu_n \stackrel{\mathclap{\mbox{\footnotesize a.s.}}}{\longrightarrow} \mu^*,
\end{equation}
where $\mu^*$ is the unique fixed point such that
\begin{equation}\label{eq:zetafp}
\zeta_{\delta}(Q^{-1}(1/\mu^*)) = {\mathbb E}(Z\cdot |G|^2) + 1/\mu^*.
\end{equation}
Define
\begin{equation}\label{eq:lambdas}
\lambda^* = Q^{-1}(1/\mu^*).
\end{equation}
Then, \eqref{eq:zetafp} can be rewritten as
\begin{equation}\label{eq:claim1}
\zeta_{\delta}(\lambda^*) = {\mathbb E}(Z\cdot |G|^2) +Q(\lambda^*)=\phi(\lambda^*),
\end{equation}
where $\phi$ is defined in \eqref{eq:defphi2}. By construction, $\zeta_{\delta}(\lambda)$ is a non-decreasing continuous function on $(\tau, \infty)$ and $\phi(\lambda)$ is a strictly decreasing continuous function. Furthermore, by hypothesis \eqref{eq:hplemmaub1}, we have that $\lim_{\lambda\to \tau^+}\phi(\lambda)=\infty$. Hence, the existence and the uniqueness of $\lambda^*$ satisfying \eqref{eq:claim1} is guaranteed. This suffices to prove the first claim of the lemma. 

Let us now move on to the proof of the second claim of the lemma. Suppose that $\zeta_{\delta}(Q^{-1}(1/\mu))$ is differentiable at $\mu=\mu^*$. Then, as $L_n(\mu)$ is convex for any $n\ge 1$, by Lemma 4 in Appendix E of \cite{lulispectral_arxiv}, we have that
\begin{equation*}
\partial_{-}L_n(\mu_n)\stackrel{\mathclap{\mbox{\footnotesize a.s.}}}{\longrightarrow}\frac{{\rm d}\zeta_{\delta}(Q^{-1}(1/\mu))}{{\rm d}\mu}\bigg |_{\mu=\mu^*} = \frac{-\zeta_{\delta}'(Q^{-1}(1/\mu^*))}{Q'(Q^{-1}(1/\mu^*))\cdot (\mu^*)^2}.
\end{equation*}
Similarly, 
\begin{equation*}
\partial_{+}L_n(\mu_n)\stackrel{\mathclap{\mbox{\footnotesize a.s.}}}{\longrightarrow} \frac{-\zeta_{\delta}'(Q^{-1}(1/\mu^*))}{Q'(Q^{-1}(1/\mu^*))\cdot (\mu^*)^2}.
\end{equation*}
By using \eqref{eq:prop2-2}, we obtain that 
\begin{equation*}
|\langle\hat{\bx}, \be_1\rangle|^2\stackrel{\mathclap{\mbox{\footnotesize a.s.}}}{\longrightarrow} \frac{\zeta_{\delta}'(Q^{-1}(1/\mu^*))}{\zeta_{\delta}'(Q^{-1}(1/\mu^*))-Q'(Q^{-1}(1/\mu^*))}=\frac{\zeta_{\delta}'(\lambda^*)}{\zeta_{\delta}'(\lambda^*)-\phi'(\lambda^*)},
\end{equation*} 
where the equality follows from the definition \eqref{eq:lambdas} of $\lambda^*$ and from the fact that $Q'(\lambda)=\phi'(\lambda)$. In order to prove the second claim of the lemma, it suffices to note that, by its definition in \eqref{eq:defzeta}, $\zeta_{\delta}'(\lambda)=\psi_{\delta}'(\lambda)$ if $\psi_{\delta}'(\lambda)>0$, and $\zeta_{\delta}'(\lambda)=0$ if $\psi_{\delta}'(\lambda) < 0$.

Finally, let us prove the third claim of the lemma. By using \eqref{eq:prop2-1}, we immediately obtain that $\lambda_1^{\bD_n} = L_n(\mu_n)$. By applying \eqref{eq:convmustar} and Lemma 3 in Appendix E of \cite{lulispectral_arxiv}, we conclude that
\begin{equation*}
\lambda_1^{\bD_n} \stackrel{\mathclap{\mbox{\footnotesize a.s.}}}{\longrightarrow} \zeta_{\delta}(\lambda^*).
\end{equation*}
As $\bP_n$ is obtained by deleting the first row and column of $\bD_n$, by applying Cauchy interlacing theorem (see, e.g., \cite[Theorem 4.3.17]{horn2012matrix}), we also have that
\begin{equation*}
\lambda_2^{\bP_n}\le \lambda_2^{\bD_n}\le \lambda_1^{\bP_n}.
\end{equation*}
Furthermore, the upper edge of the support of the limiting spectral distribution of $\bP_n$ is given by \cite[Section 4]{silsang1995} and \cite[Lemma 3.1]{baiyaospike2012} $$\min_{\lambda > \tau}\psi_{\delta}(\lambda)=\zeta_{\delta}(\bar{\lambda}_{\delta}),$$ where $\bar{\lambda}_{\delta}$ is defined in \eqref{eq:minpsi}. Therefore, 
\begin{equation*}
\lambda_2^{\bD_n} \stackrel{\mathclap{\mbox{\footnotesize a.s.}}}{\longrightarrow} \zeta_{\delta}(\bar{\lambda}_{\delta}),
\end{equation*}
which concludes the proof. 
\end{proof}

At this point, we are ready to prove our spectral upper bound.
 
\begin{proof}[Proof of Theorem \ref{th:upper}]
Note that the normalization of $\bx$ and $\{\ba_i\}_{1\le i \le n}$ required in Lemma \ref{lemma:condub} is different from the normalization required in Theorem \ref{th:upper}. However, the scalar product $\langle \bx, \ba_i\rangle$ is the same and the data matrix $\bD_n$ changes by a factor $d$. Hence, the principal eigenvector $\hat{\bx}$ is not affected by this change in the normalization. 

Let $G\sim \cnormal(0, 1)$, $Y\sim p(\cdot \mid |G|)$ and  $Z = \mathcal T(Y)$, where $p$ is defined in \eqref{eq:defy} and $\mathcal T$ is some pre-processing function that we will choose later on. 
We will assume that the supremum $\tau$ of the support of $Z$ is strictly positive and that conditions \eqref{eq:hplemmaub1} 
are satified, and  will verify later that our choice of the function $\mathcal T$ satisfies these requirements.
Recall that the function $\psi_{\delta}(\lambda)$ defined in \eqref{eq:defpsi} is convex and that it attains its minimum at the point $\bar{\lambda}_{\delta}$. Since by condition \eqref{eq:hplemmaub1}  
$\psi_{\delta}(\lambda)\uparrow \infty$ as $\lambda\downarrow 0$, we have $\bar{\lambda}_{\delta}\in (\tau,\infty)$. Hence, $\psi'_{\delta}(\bar{\lambda}_{\delta}) =0$. 
By calculating the derivative of $\psi_{\delta}(\lambda)$ and setting it to $0$, we have
\begin{equation}\label{eq:condub1}
{\mathbb E}\left\{\frac{Z^2}{(\bar{\lambda}_\delta-Z)^2}\right\}=\frac{1}{\delta}.
\end{equation}
Furthermore, as pointed out in Remark \ref{rmk:diffreg}, \eqref{eq:corrub} holds for some $\epsilon > 0$ if  and only if 
\begin{equation}\label{eq:condub2}
\phi(\bar{\lambda}_{\delta}) > \psi_{\delta}(\bar{\lambda}_{\delta}). 
\end{equation}

As $\tau>0$, we also have that $\bar{\lambda}_{\delta}>0$. Consider now the matrix $\bD_n'=\bD_n/\alpha$ for some $\alpha>0$. Then, the principal eigenvector of $\bD_n'$ is equal to the principal eigenvector of $\bD_n$. Hence, we can assume without loss of generality that $\bar{\lambda}_{\delta}=1$. Consequently, the conditions \eqref{eq:condub1} and \eqref{eq:condub2} can be respectively rewritten as
\begin{equation}\label{eq:condubnew1}
{\mathbb E}\left\{\frac{Z^2}{(1-Z)^2}\right\}=\frac{1}{\delta},
\end{equation}
\begin{equation}\label{eq:condubnew2}
{\mathbb E}\left\{\frac{Z(|G|^2-1)}{1-Z}\right\} >\frac{1}{\delta}.
\end{equation}
Furthermore, as $Z=\mathcal T(Y)$, we also obtain that
\begin{equation}\label{eq:compint}
\begin{split}
{\mathbb E}\left\{\frac{Z^2}{(1-Z)^2}\right\} &=\int_{\mathbb R} \left(\frac{\mathcal T (y)}{1-\mathcal T (y)}\right)^2 {\mathbb E}_G \left\{p(y\mid |G|)\right\}\,{\rm d}y,\\
{\mathbb E}\left\{\frac{Z(|G|^2-1)}{1-Z}\right\} &= \int_{\mathbb R} \frac{\mathcal T (y)}{1-\mathcal T (y)} {\mathbb E}_G \left\{p(y\mid |G|)\cdot(|G|^2-1)\right\}\,{\rm d}y.\\
\end{split}
\end{equation}
Let $\mathcal T^*(y)$ be defined in \eqref{eq:defty}. Note that, if we substitute $\mathcal T(y)=\mathcal T^*(y)$ into the RHS of \eqref{eq:compint}, then 
\begin{equation*}
{\mathbb E}\left\{\frac{Z^2}{(1-Z)^2}\right\} = {\mathbb E}\left\{\frac{Z(|G|^2-1)}{1-Z}\right\} = \frac{1}{\delta_{\rm u}},
\end{equation*}
where $\delta_{\rm u}$ is defined in \eqref{eq:defdeltau}. Let $\mathcal T_\delta^*(y)$ be defined in \eqref{eq:deftydelta}. Then, 
\begin{equation*}
\frac{\mathcal T_\delta^*(y)}{1-\mathcal T_\delta^*(y)} = \sqrt{\frac{\delta_{\rm u}}{\delta}}\frac{\mathcal T^*(y)}{1-\mathcal T^*(y)},
\end{equation*}
which immediately implies that 
\begin{equation}\label{eq:expvalTdeltastar1}
{\mathbb E}\left\{\frac{(\mathcal T_\delta^*(Y))^2}{(1-\mathcal T_\delta^*(Y))^2}\right\}=\frac{1}{\delta},
\end{equation}
\begin{equation}\label{eq:expvalTdeltastar2}
{\mathbb E}\left\{\frac{\mathcal T_\delta^*(Y)(|G|^2-1)}{1-\mathcal T_\delta^*(Y)}\right\} = \frac{1}{\sqrt{\delta\cdot \delta_{\rm u}}} >\frac{1}{\delta}.
\end{equation}
 As a result, we need to show that the function $\mathcal T_\delta^*(y)$ fulfills the following requirements:
\begin{enumerate}[(1)]
\item $\mathcal  T_\delta^*(y)$ is bounded;
\item $\mathbb P(\mathcal  T_\delta^*(Y)=0)<1$;
\item the supremum $\tau$ of the support of $\mathcal  T_\delta^*(Y)$ is strictly positive;
\item the condition \eqref{eq:hplemmaub1} holds.
\end{enumerate}

Note that $\mathcal T_\delta^*(y)$ is bounded, as $\mathcal T^*(y)\le 1$. Furthermore, if
\begin{equation}\label{eq:badcond}
{\mathbb E}_{G}\left\{p(y\mid |G|)\right\}={\mathbb E}_{G}\left\{p(y\mid |G|)|G|^2\right\},
\end{equation}
identically, then $\delta_{\rm u}=\infty$ and the claim of Theorem \ref{th:upper} trivially holds. Hence, we can assume that \eqref{eq:badcond} does not hold, which implies that the function $\mathcal T^*$ is not equal to the constant value $0$. Consequently, $\mathbb P(\mathcal T_\delta^*(Y)=0)<1$. 
 
By definition \eqref{eq:defty} of $\mathcal T^*$, we have that 
\begin{equation}
{\mathbb E}_Y\left\{\frac{1}{1-\mathcal T^*(Y)}\right\} = \int_{\mathbb R} {\mathbb E}_{G}\left\{p(y\mid |G|)\cdot |G|^2\right\}\,{\rm d}y = {\mathbb E}_G\left\{|G|^2\right\}=1.
\end{equation} 
Hence, $\mathbb P(\mathcal T^*(Y)>0)>0$, which implies that $\mathbb P(\mathcal T_\delta^*(Y)>0)>0$. Consequently, the supremum $\tau$ of the support of $\mathcal T_\delta^*(Y)$ is strictly positive.

If $\mathbb P(\mathcal T_\delta^*(Y)=\tau)>0$, then the condition \eqref{eq:hplemmaub1} is satisfied. Suppose now that $\mathbb P(\mathcal T_\delta^*(Y)=\tau)=0$. Then, for any $\epsilon_1 >0$, there exists $\Delta_1(\epsilon_1)$ such that
\begin{equation}\label{eq:condDelta1}
0 < \mathbb P \bigl(\mathcal T_\delta^*(Y)\in (\tau-\Delta_1(\epsilon_1), \tau)\bigr)\le \epsilon_1.
\end{equation} 
Define
\begin{equation}\label{eq:defTyfin}
\mathcal T_{\delta}^*(y, \epsilon_1)=\left\{\begin{array}{ll}
\mathcal T_\delta^*(y), & \quad\mbox{if } \mathcal T_\delta^*(y)\le \tau-\Delta_1(\epsilon_1),\\
\\
\tau-\Delta_1(\epsilon_1), & \quad\mbox{otherwise}.\\
\end{array}\right.
\end{equation}
Clearly, the random variable $\mathcal T_{\delta}^*(Y, \epsilon_1)$ has a point mass, hence the condition \eqref{eq:hplemmaub1} is satisfied.

As a final step, we show that we can take $\epsilon_1=0$. Define
\begin{equation*}
\bD_n(\epsilon_1) = \frac{1}{n}\sum_{i=1}^n \mathcal T_{\delta}^*(y_i, \epsilon_1) \ba_i \ba_i^*.
\end{equation*}
Define also 
\begin{equation*}
\bD_n = \frac{1}{n}\sum_{i=1}^n \mathcal T_\delta^* (y_i) \ba_i \ba_i^*.
\end{equation*}
Let $\hat{\bx}(\epsilon_1)$ and $\hat{\bx}$ be the principal eigenvectors of $\bD_n(\epsilon_1)$ and of $\bD_n$, respectively. Then,
\begin{equation}
\norm{\bD_n(\epsilon_1)-\bD_n}_{\rm op} \le C_1 \cdot \Delta_1(\epsilon_1),
\end{equation}
where the constant $C_1$ depends only on $n/d$. By Lemma \ref{lemma:condub}, there is a strictly positive gap, call it $\theta$, between the first and the second eigenvalue of $\bD_n(\epsilon_1)$. Consequently, by the Davis-Kahan theorem \cite{davis1970rotation}, we conclude that 
\begin{equation}
\norm{\hat{\bx}(\epsilon_1)-\hat{\bx}}_2\le C_2  \cdot \Delta_1(\epsilon_1),
\end{equation}
where the constant $C_2$ depends only on $n/d$ and on $\theta$. In words, for any $n$, as $\epsilon_1$ tends to $0$, the principal eigenvector of $\bD_n(\epsilon_1)$ tends to the principal eigenvector of $\bD_n$. This means that we can set $\mathcal T = \mathcal T_\delta^*$ and have that, almost surely, \eqref{eq:corrub} holds.

In order to conclude the proof, it remains to show that $\delta_{\rm u}$ is the optimal threshold for the spectral method, namely, for any $\delta < \delta_{\rm u}$, there is no pre-processing function $\mathcal T$ such that, \eqref{eq:corrub} holds almost surely. To do so, note that,  \eqref{eq:corrub} holds almost surely if and only if \eqref{eq:condubnew1} and \eqref{eq:condubnew2} are satisfied. By setting $u(y) = \mathcal T(y)/(1-\mathcal T(y))$ and using \eqref{eq:compint}, we have that these conditions can be rewritten as
\begin{equation}\label{eq:compintbis}
\int_{\mathbb R} (u(y))^2 {\mathbb E}_G \left\{p(y\mid |G|)\right\}\,{\rm d}y=\frac{1}{\delta},
\end{equation}
\begin{equation}\label{eq:compintter}
{\mathbb E}\left\{\frac{Z(|G|^2-1)}{1-Z}\right\} = \int_{\mathbb R} u(y) \sqrt{{\mathbb E}_G \left\{p(y\mid |G|)\right\}}\frac{{\mathbb E}_G \left\{p(y\mid |G|)\cdot(|G|^2-1)\right\}}{\sqrt{{\mathbb E}_G \left\{p(y\mid |G|)\right\}}}\,{\rm d}y>\frac{1}{\delta}.
\end{equation}
By Cauchy-Schwarz inequality, we also have that
\begin{equation}\label{eq:compintter2}
\begin{split}
\int_{\mathbb R} & u(y) \sqrt{{\mathbb E}_G \left\{p(y\mid |G|)\right\}}\frac{{\mathbb E}_G \left\{p(y\mid |G|)\cdot(|G|^2-1)\right\}}{\sqrt{{\mathbb E}_G \left\{p(y\mid |G|)\right\}}}\,{\rm d}y \\
&\le \sqrt{\int_{\mathbb R} (u(y))^2 {\mathbb E}_G \left\{p(y\mid |G|)\right\}\,{\rm d}y}\,\sqrt{\int_{\mathbb R} \frac{\left({\mathbb E}_G \left\{p(y\mid |G|)\cdot(|G|^2-1)\right\}\right)^2}{{\mathbb E}_G \left\{p(y\mid |G|)\right\}}\,{\rm d}y}.
\end{split}
\end{equation}
 By combining \eqref{eq:compintbis}, \eqref{eq:compintter} and \eqref{eq:compintter2} with the definition \eqref{eq:defdeltau} of $\delta_{\rm u}$, we conclude that
\begin{equation}
\frac{1}{\sqrt{\delta_{\rm u}}}\frac{1}{\sqrt{\delta}}>\frac{1}{\delta},
\end{equation}
which implies that $\delta > \delta_{\rm u}$. Consequently, for 
$\delta \leq \delta_{\rm u}$, no pre-processing function achieves weak recovery and the  proof is complete.  
\end{proof}

\begin{remark}[Proof of Spectral Upper Bound for the Real Case]\label{rmk:realub}
First, we need to prove a result analogous to that of Lemma \ref{lemma:condub}, where $\bx \sim \Unif(\Sphere_{\mathbb R}^{d-1})$, $\{\ba_i\}_{1\le i\le n}\sim_{i.i.d.}\normal({\bm 0}_d,\id_d)$, $\by$ is distributed according to \eqref{eq:defyr}, and $G\sim \normal(0, 1)$. To do so, one can follow the proof of Theorem 1 of \cite{lulispectral_arxiv}. The technical difficulty consists in the fact that the matrix $\bM_n$ is not necessarily PSD. In order to solve this issue, we apply the version of Lemma \ref{lemma:baiyaonopsd} for the real case discussed in Remark \ref{rmk:baiyaonopsd} at the end of Appendix \ref{app:baiyaonopsd}. 
At this point, the proof of Theorem \ref{th:upperr} follows from the same argument as the proof of Theorem \ref{th:upper}. 
\end{remark}

\section{Comparison with Message Passing Algorithms}\label{sec:amp}

\subsection{Motivation and Background}

Message passing algorithms have proved successful in a broad range of statistical estimation problems, including high-dimensional regression \cite{BayatiMontanariLASSO},
robust regression \cite{highdimDonoho}, low-rank matrix estimation \cite{deshpande2013finding,krzakala2013phase,montanari2016non,kabashima2016phase}, 
and network structure estimation \cite{decelle2011asymptotic,mossel2014belief,mossel2016density}. 
A bold conjecture from statistical physics suggests that -- for these and other problems -- message passing approaches achieve optimal statistical performances among polynomial-time
algorithms. In view of this conjecture, it is interesting to compare our spectral approach to message passing algorithms. We will present two types of results
(with $\delta_{\rm u}$ the spectral threshold defined in \eqref{eq:defdeltaur}):
\begin{enumerate}
\item We prove that, for $\delta<\delta_{\rm u}$ (i.e. in the regime in which the spectral approach fails), message passing converges to an un-informative fixed point, even if initialized in a state
that is correlated with the true signal $\bx$.
\item Vice versa, for $\delta>\delta_{\rm u}$ (when the spectral algorithm achieves weak recovery), we consider a linearized message passing algorithm, and prove that the un-informative
fixed point is unstable. The proof of this fact builds on the analysis contained in the previous pages. 
\end{enumerate}
Let us point out that the techniques described in Section \ref{sec:spectralproof} to compute the spectral threshold $\delta_{\rm u}$ are different from those described in this section to analyze message passing algorithms. Hence, we find very interesting the fact that the spectral threshold is closely related to the performance of message passing. In particular, our findings suggest the conjecture that $\delta_{\rm u}$ represents the fundamental limit for all polynomial-time algorithms.

Note also that message passing often allows to further refine the spectral estimate, in order to provide an exact recovery of the signal. Hence, combining the analyses of message passing and of the spectral method to provide a threshold for exact recovery constitutes an interesting direction for future research (see \cite{montanari2017estimation}
for an example in which this program is carried out).

For the sake of simplicity, we will assume that the signal $\bx$ and the measurement matrix $\bA$ are real. Of particular interest for the present setting is approximate message passing (AMP) \cite{DMM09,BM-MPCS-2011}: this is a broad class of iterative methods 
that operates with dense random matrices (as the sensing matrix $\bA$ in the present case). 
In particular, in \cite{RanganGAMP} it was proposed a ``generalized approximate message passing'' (GAMP) scheme, which is an AMP algorithm for Bayesian estimation
in non-linear regression models. This approach was further developed in the context of phase retrieval in \cite{schniter2015compressive}.
 We will follow the same Bayesian formulation here, by considering an AMP algorithm that is equivalent to GAMP although somewhat simpler.

In order to minimize technical overhead, we assume throughout this section that the conditional density $p(y \mid g)$ is bounded and two times differentiable with respect to $g$.  
Denote by $\partial_{g}p(y\mid g)$ and $\partial^2_{g}p(y\mid g)$ the first and the second derivative of $p(y\mid g)$, respectively. Let $G\sim \normal(0, 1)$ and define the function
\begin{equation}\label{eq:FGdef}
\sF(x,y;\baq) = \frac{\mathbb E_G\{\partial_{g}p(y\mid \baq\, x+\sqrt{\baq}G )\}}{\mathbb E_G\{ p(y\mid  \baq\, x+\sqrt{\baq}G )\}}\, .
\end{equation}
We further define the following ``state evolution'' recursion:
\begin{equation}\label{eq:SE}
\begin{split}
\mu_{t+1} & = \delta\cdot h(q_t)\, ,\\
q_{t} & = \frac{\mu_t}{1+\mu_t}\, ,
\end{split}
\end{equation}
where 
\begin{equation}\label{eq:defhfun}
h(q) = \int_{\mathbb R} \mathbb E_{G_0}\left\{\frac{\big(\mathbb E_{G_1} \{\partial_{g} p(y\mid\sqrt{q}G_0+\sqrt{1-q}G_1)\}\big)^2}{\mathbb E_{G_1}\{p(y\mid\sqrt{q}G_0+\sqrt{1-q}G_1)\}}\right\}\, \de y\, ,
\end{equation}
with $G_0, G_1\sim_{i.i.d.} \normal(0, 1)$.

Given the sensing  matrix $\bA = (\ba_1, \ldots, \ba_n)^{\sT}\in \mathbb R^{n\times d}$, and the vector of measurements $\by = (y_1, \ldots, y_n)\in \mathbb R^n$,
the message passing algorithm updates iteratively the estimate $\bz^t\in\reals^d$ of the signal $\bx\in \reals^d$, with $\norm{\bx}_2=\sqrt{d}$, according to the iteration
\begin{equation}\label{eq:GAMP}
\begin{split}
\bz^{t+1} & = \bA^{\sT} f_t(\hbz^t;\by) - \sb_t \bz^t\, ,\\
\hbz^t & = \bA\bz^t- f_{t-1}(\hbz^{t-1};\by)\, .
\end{split}
\end{equation}
Here, the function $f_t(\hbz;\by) = (f_t(\hz_1;y_1),\dots,f_t(\hz_n;y_n))$ is understood to be applied component-wise to 
its arguments and $\sb_t$ it is defined as
\begin{equation}
f_t(\hz;y) =  \sF(\hz,y;1-q_t)\, ,\label{eq:GAMP_Fdef}
\end{equation}
and the ``Onsager coefficient'' $\sb_t$ is defined as
\begin{equation}
\sb_t = \delta \cdot \E\{ f'_t(\mu_t G_0+\sqrt{\mu_t} G_1;Y)\}\, ,\label{eq:OnsDef}
\end{equation}
where $f_t'(\hz;y)$ denotes the derivative of $f_t(\hz;y)$ with respect to $\hz$, and the expectation is with respect to $G_0,G_1\sim_{i.i.d.}\normal(0,1)$ and $Y\sim p(\,\cdot\,|G_0)$. 
The recursion (\ref{eq:GAMP}) is initialized with $\bz^{0}\in\reals^d$ and it is understood
that $f_{-1}(\cdot\, ;\cdot) = \bzero_n$.

State evolution precisely tracks the asymptotics of AMP. The next statement is a consequence of \cite{BM-MPCS-2011,javanmard2013state}.
We refer to Appendix  \ref{app:AMPfail} for its proof.
\begin{lemma}[State Evolution for AMP Iteration \eqref{eq:GAMP}]\label{lemma:SE}
Let $\bx\in \mathbb R^d$ denote the unknown signal such that $\norm{\bx}_2 = \sqrt{d}$, $\bA = (\ba_1, \ldots, \ba_n)^{\sT}\in \mathbb R^{n\times d}$ with $\{\ba_i\}_{1\le i\le n}\sim_{i.i.d.}\normal(\bzero_d,\id_d/d)$, 
and $\by = (y_1, \ldots, y_n)$ with $y_i\sim p(\cdot \mid \langle \bx, \ba_i \rangle)$. 
Consider the AMP iterates $\bz^t, \hbz^t$ defined in \eqref{eq:GAMP}, where $f_t(\hat{z}; y)$ and $\sb_t$ are given by \eqref{eq:GAMP_Fdef} and  \eqref{eq:OnsDef}, respectively. Assume that the initialization
$\bz^0$ is independent of $\bA$ and that, almost surely,
\begin{align}
\lim_{n\to\infty}\frac{1}{d}\<\bx,\bz^0\> = \mu_0\, , \;\;\; \lim_{n\to\infty}\frac{1}{d}\|\bz^0\|^2 = \mu_0^2+\mu_0\, .
\end{align}
Let the state evolution recursion $q_t,\mu_t$ be defined as in~(\ref{eq:SE}) with initialization $\mu_0$. Then, for any $t$, and for any function
$\psi:\reals^2\to\reals$ such that $|\psi(\bu)-\psi(\bv)|\le L(1+\|\bu\|_2+\|\bv\|_2)\|\bu-\bv\|_2$ for some $L\in \mathbb R$, we have that, almost surely,
\begin{align}
\lim_{n\to\infty}\frac{1}{n}\sum_{i=1}^n\psi(x_i,z^t_i) = \E\left\{\psi(X_0, \mu_t X_0+ \sqrt{\mu_t}G)\right\}\, ,
\end{align}
where the expectation is taken with respect to $X_0,G\sim_{i.i.d.}\normal(0,1)$.
\end{lemma}

Informally, this lemma states that $\bz^t$ is a noisy version of the signal $\bx$, namely $\bz^t\approx \mu_t \,\bx +\sqrt{\mu_t}\, \bg$, with 
$\bg\sim \normal(\bzero_d,\id _d)$, and that this approximation holds for empirical averages. 

\subsection{Results}

In order to obtain a non-vanishing weak recovery threshold, we assume that the observation model satisfies the condition
\begin{equation}\label{eq:condfixedpt}
\E_G\{\partial_{g}p(y\mid G)\}=0\, ,
\end{equation}
where the expectation is with respect to $G\sim\normal(0,1)$. Notice that this implies $h(0) = 0$, therefore $\mu_t=q_t=0$ is a fixed point of state evolution.
Furthermore, $\sF(0,y;1) = 0$, therefore $q_t=0$, $\bz^t=\bzero_d$ is a fixed point of the message passing algorithm. We will refer to this as to the ``un-informative fixed point''.
Note that the condition \eqref{eq:condfixedpt} holds -- among others -- for the phase retrieval problem.

Vice versa, if $\E_G\{\partial_{g}p(y\mid G)\}\neq 0$, then $\mu_1>0$ even if $\mu_0=0$, for any $\delta>0$. Thanks to Lemma \ref{lemma:SE}, this implies that
weak recovery is possible for all $\delta>0$. Hence, we will assume that the condition \eqref{eq:condfixedpt} holds.

The first result of this section establishes the following: for $\delta < \delta_{\rm u}$, the message passing algorithm fails even if the initial condition has a positive correlation with the unknown signal.
We refer to  Appendix \ref{app:AMPfail} for its proof.
\begin{theorem}[Message Passing Fails for $\delta < \delta_{\rm u}$]\label{th:spectralAMPfail}
Let $\bx\in \mathbb R^d$ denote the unknown signal such that $\norm{\bx}_2 = \sqrt{d}$. Let $\bA = (\ba_1, \ldots, \ba_n)^{\sT}\in \mathbb R^{n\times d}$ with $\{\ba_i\}_{1\le i\le n}\sim_{i.i.d.}\normal(\bzero_d,\id_d/d)$, 
and $\by = (y_1, \ldots, y_n)$ with $y_i\sim p(\cdot \mid \langle \bx, \ba_i \rangle)$. Let $n/d\to \delta$ and define $\delta_{\rm u}$ as in \eqref{eq:defdeltaur}. Let $G\sim \normal(0, 1)$ 
and assume that the condition~(\ref{eq:condfixedpt}) holds for any $y\in \mathbb R$.

Consider the AMP algorithm defined in~(\ref{eq:GAMP}), and assume that the initial condition $\bz^{0}$ is such that
\begin{equation}
\lim_{n\to \infty}\frac{\langle\bz^{0}, \bx\rangle}{\norm{\bz^0}_2 \norm{\bx}_2} = \epsilon.
\end{equation}
Then, for any $\delta < \delta_{\rm u}$, there exists $\epsilon_0(\delta)$ such that for any $\epsilon \in (0, \epsilon_0(\delta))$, almost surely,
\begin{equation}
\begin{split}
\lim_{t\to \infty}\lim_{n\to \infty}\frac{1}{d}\norm{\bz^{t}}_2 &=\b0_d.
\end{split}
\end{equation}
\end{theorem}

Next, we consider the case $\delta>\delta_{\rm u}$ and we linearize the iteration (\ref{eq:GAMP}) around the non-informative fixed point. 
\begin{lemma}[Linearized AMP Equations]\label{lemma:linGAMP}
Consider the one-step map defined in~(\ref{eq:GAMP}) and assume that condition \eqref{eq:condfixedpt} holds. Define $R_t = \|\bz^t\|_2+\|\hbz^{t-1}\|_2$.
Then, as $R_t\to 0$ and $q_t\to 0$, we obtain
\begin{equation}
\left(\begin{array}{c}
\bz^{t+1}\\
\hbz^t\\
\end{array}\right) =\bL_n \left(\begin{array}{c}
\bz^{t}\\
\hbz^{t-1}\\
\end{array}\right) +o(R_t)+ R_t\, o_{q_t}(1)\, ,
\end{equation}
where $\bL_n\in \mathbb R^{(n+d)\times (n+d)}$ is defined as
\begin{equation}\label{eq:Mlin}
\bL_n = \left(\begin{array}{cc}
\bA^{\sT} \bJ \bA & -\bA^{\sT} \bJ^2\\
\bA & -\bJ\\
\end{array}\right),
\end{equation}
and $\bJ\in \mathbb R^{n\times n}$ is a diagonal matrix with entries $j_{i} =\sF'(0,y_i;1)$ for $i\in [n]$, with $\sF'$ denoting the derivative of $\sF$ with respect to the first argument.
\end{lemma}

The second result of this section establishes the following: for $\delta > \delta_{\rm u}$, the un-informative fixed point is unstable for the iteration (\ref{eq:GAMP}), i.e., the matrix $\bL_n$ has an eigenvalue that is larger than 1 in modulus. To do so, we will relate the matrix $\bJ$ appearing in \eqref{eq:Mlin} to the optimal pre-processing function defined in \eqref{eq:deftyr} (see equation \eqref{eq:JGrel} in Appendix \ref{app:AMPsucc}). We refer to Appendix \ref{app:AMPsucc} for the complete proof.
\begin{theorem}[Message Passing Escapes from Un-informative Fixed Point for $\delta > \delta_{\rm u}$]\label{th:spectralAMPsucc}
Let $\bx\in \mathbb R^d$ denote the unknown signal such that $\norm{\bx}_2 = \sqrt{d}$. Let $\bA = (\ba_1, \ldots, \ba_n)^{\sT}\in \mathbb R^{n\times d}$ with $\{\ba_i\}_{1\le i\le n}\sim_{i.i.d.}\normal({\bm 0}_d,\id_d/d)$, and $\by= (y_1, \ldots, y_n)$ with $y_i\sim p(\cdot \mid \langle \bx, \ba_i \rangle)$. Let $n/d\to \delta$ and define $\delta_{\rm u}$ as in \eqref{eq:defdeltaur}. Furthermore, assume that \eqref{eq:condfixedpt} holds for any $y\in \mathbb R$. Define $\bL_n\in \mathbb R^{(n+d)\times (n+d)}$ as in \eqref{eq:Mlin}, where $\bJ\in \mathbb R^{n\times n}$ is a diagonal matrix with entries $j_i =\sG(0,y_i;1)$ for $i\in [n]$. Then, the eigenvalues of $\bL_n$ are real and the largest of them, call it $\lambda_1^{\bL_n}$, is such that, for any $\delta > \delta_{\rm u}$, 
\begin{equation}
\lim_{n\to\infty}\lambda_1^{\bL_n} > 1.
\end{equation}
\end{theorem}

\section{Numerical Experiments}\label{sec:num}

We focus on the phase retrieval problem and present some numerical results to illustrate the performance achieved by the proposed spectral method. First, we consider the case in which the unknown vector is chosen uniformly at random and the sensing vectors are Gaussian. Then, we consider the more practical scenario in which the unknown vector is an image and the sensing vectors come from a coded diffraction model.

\subsection{Gaussian Sensing Vectors for Synthetic Data} \label{subsec:gauss}

Let us consider the complex case. In our experiments, the vector $\bx$ is chosen uniformly at random on the $d$-dimensional complex sphere with radius $\sqrt{d}$, the sensing vectors $\{\ba_i\}_{1\le i\le n}$ are i.i.d. circularly-symmetric normal with variance $1/d$ and, for $i\in [n]$, the measurement $y_i$ is equal to $|\langle \bx, \ba_i\rangle|^2$. We take $d=4096$ and the numerical simulations are averaged over $n_{\rm sample} = 40$ independent trials. The results are plotted in Figure \ref{fig:compl}.

The red curve corresponds to the proposed pre-processing function given by 
\begin{equation}\label{eq:numpre}
\mathcal T(y) = \frac{y-1}{y+\sqrt{\tilde{\delta}}-1}.
\end{equation}
We pick $\tilde{\delta} = 1.001$ and, as shown by the figure, weak recovery is possible for values of $\delta$ very close to $1$.

The green curve corresponds to the pre-processing function given by 
\begin{equation}\label{eq:numprepos}
\mathcal T(y) = \max\left(\frac{y-1}{y+\sqrt{\tilde{\delta}}-1}, 0\right).
\end{equation}
We add this plot in order to show that, by enforcing non-negativity of the pre-processing function, we incur in a degradation of the performance of the spectral method. 

The black curve corresponds to the pre-processing function given by
\begin{equation}\label{eq:subset}
\mathcal T(y)= \left\{\begin{array}{ll}
\vspace{0.5em}
1, &  \quad\mbox{for }y>t,\\
0, &  \quad\mbox{otherwise}.\\
\end{array}\right.
\end{equation}
This choice was proposed in \cite{wang2016solving} and it is also considered in \cite{lulispectral_arxiv}, where the authors refer to it as the ``subset algorithm''. For each value of $t$, we can compute the smallest value of $\delta$, call it $\delta^*(t)$, that yields a strictly positive scalar product according to the result of Lemma \ref{lemma:condub}. Hence, we pick $t=2$ that corresponds to the smallest value of $\delta^*(t)$ over $t\in \{0.25, 0.5, 0.75, \ldots, 10\}$.

The blue curve corresponds to the pre-processing function given by
\begin{equation}\label{eq:trimming}
\mathcal T(y)= \left\{\begin{array}{ll}
\vspace{0.5em}
y, &  \quad\mbox{for }y\le t,\\
0, &  \quad\mbox{otherwise}.\\
\end{array}\right.
\end{equation}
This choice corresponds to the truncated spectral initialization proposed in \cite{chen2017solving} and it is also considered in \cite{lulispectral_arxiv}, where the authors refer to it as the ``trimming algorithm''. For each value of $t$, we can compute the smallest value of $\delta$, call it $\delta^*(t)$, that yields a strictly positive scalar product according to the result of Lemma \ref{lemma:condub}. Hence, we pick $t=5.25$ that corresponds to the smallest value of $\delta^*(t)$ over $t\in \{0.25, 0.5, 0.75, \ldots, 10\}$.

Note that the numerical simulations follow closely the theoretical prediction given by \eqref{eq:numpred}. Furthermore, the choice of the pre-processing function \eqref{eq:numpre} yields a remarkable performance gain with respect to both the subset algorithm and the trimming algorithm. 

Similar considerations apply to the real case. Here, the vector $\bx$ is chosen uniformly at random on the $d$-dimensional real sphere with radius $\sqrt{d}$ and the sensing vectors $\{\ba_i\}_{1\le i\le n}$ are i.i.d. normal with zero mean and variance $1/d$. We pick $d=4096$ and $n_{\rm sample}=40$. The results are plotted in Figure \ref{fig:real}. Again, the numerical simulations follow closely the theoretical prediction. The red curve corresponds to the pre-processing function given by \eqref{eq:numpre}, where we pick $\tilde{\delta} = 1.001$. Note that weak recovery is possible for values of $\delta$ very close to $1/2$. The green curve corresponds to the pre-processing function given by \eqref{eq:numprepos}. The blue curve corresponds to the pre-processing function given by \eqref{eq:subset}, where we pick $t=2$ which yields the smallest value of $\delta^*(t)$ over $t\in \{0.25, 0.5, 0.75, \ldots, 10\}$. The black curve corresponds to the pre-processing function given by \eqref{eq:trimming}, where we pick $t=7$ which yields the smallest value of $\delta^*(t)$ over $t\in \{0.25, 0.5, 0.75, \ldots, 10\}$.

\begin{figure}[p]
    \centering
    \subfloat[Complex case.\label{fig:compl}]{\includegraphics[width=.65\columnwidth]{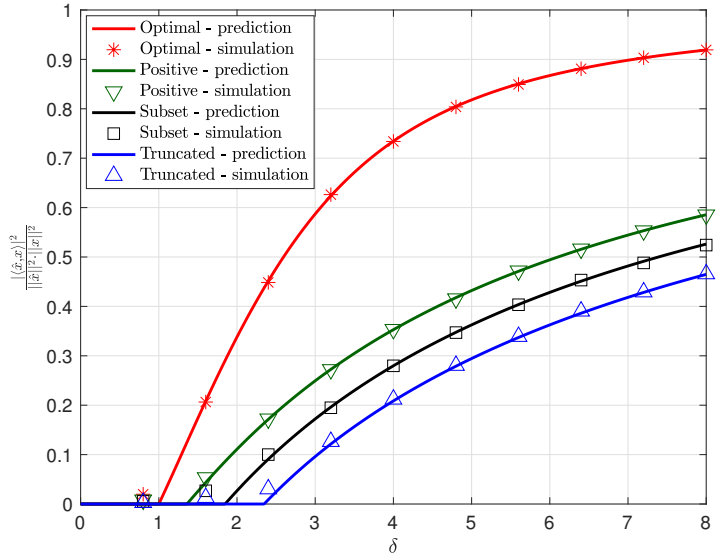}}\\
    \subfloat[Real case.\label{fig:real}]{\includegraphics[width=.65\columnwidth]{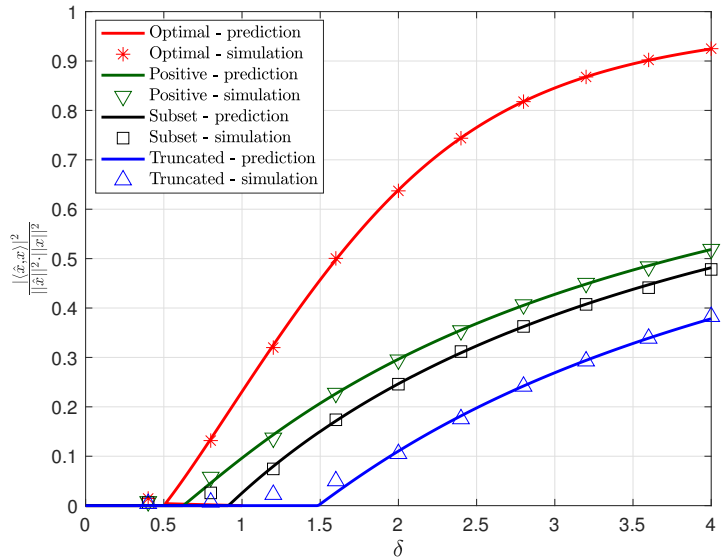}}
\caption{Performance of the spectral method for the phase retrieval problem where the unknown vector is uniformly random on the sphere and the sensing vectors are Gaussian. On the $x$-axis, we have the ratio $\delta$ between the number of samples and the dimension of the signal; on the $y$-axis, we have the square of the normalized scalar product between the unknown signal $\bx$ and the estimate $\hat{\bx}$. Note that the proposed choice of the pre-processing function (red curve) provides a significant performance improvement with respect to the subset algorithm considered in \cite{wang2016solving, lulispectral_arxiv} (black curve) and the truncated spectral initialization considered in \cite{chen2017solving, lulispectral_arxiv} (blue curve).}
\label{fig:gauss}
\end{figure}

\subsection{Coded Diffraction Model for Natural Images} \label{subsec:natimg}

\begin{figure}[t]
    \centering
\includegraphics[width=.8\columnwidth]{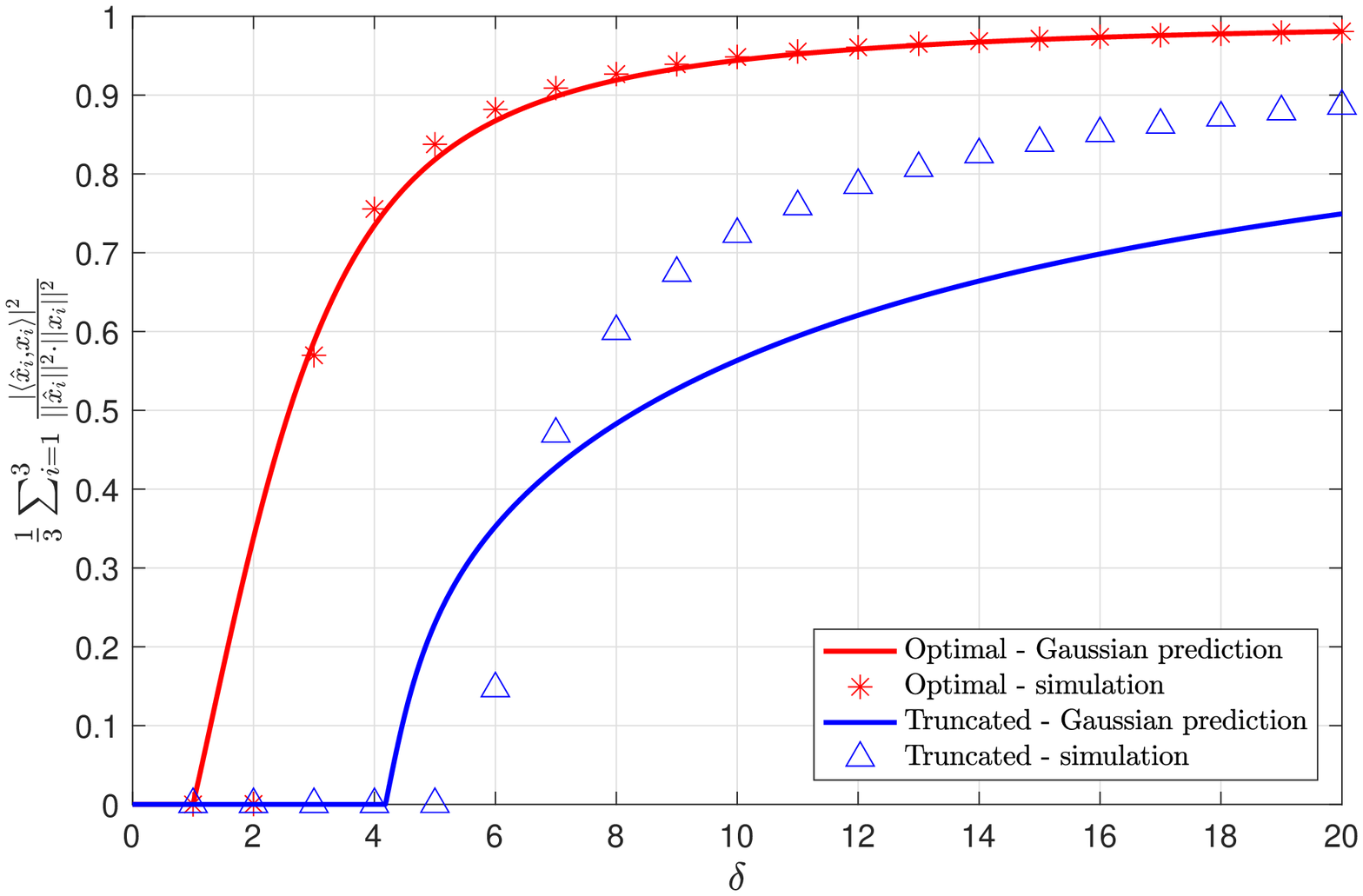}\\
\caption{Performance of the spectral method for the phase retrieval problem where the unknown vector is a digital photograph and the sensing vectors are obtained from a coded diffraction model. On the $x$-axis, we have the ratio $\delta$ between the number of samples and the dimension of the signal; on the $y$-axis, we have the square of the normalized scalar product between the unknown signal $\bx$ and the estimate $\hat{\bx}$ (averaged on the three RGB components of the image). Note that the proposed choice of the pre-processing function (red curve) provides a significant performance improvement with respect to the truncated spectral initialization considered in \cite{chen2017solving} (blue curve).}
\label{fig:venscal}
\end{figure}

We consider a model of coded diffraction patterns in which the sensing vectors $\{\ba_r\}_{1\le r \le n}$ are obtained as follows. For $t_1\in [d_1]$ and $t_2\in [d_2]$, denote by $a_r(t_1, t_2)$ the $(t_1, t_2)$-th component of the vector $\ba_r\in \mathbb C^{d}$, with $d=d_1\cdot d_2$. Then,
\begin{equation}
a_r(t_1, t_2) = d_{\ell}(t_1, t_2)\cdot e^{i2\pi k_1 t_1/d_1}\cdot e^{i2\pi k_2 t_2/d_2},
\end{equation}
where $i$ denotes the imaginary unit. The index $r\in [n]$ is associated to a pair $(\ell, k)$, with $\ell\in [L]$; the index $k\in [d]$ is associated to a pair $(k_1, k_2)$ with $k_1\in [d_1]$ and $k_2\in [d_2]$. As usual, the measurement $y_r$ of an unknown $d$-dimensional vector $\bx$ is equal to $|\langle \bx, \ba_r\rangle|^2$. As an immediate consequence, the number of measurements $n$ is equal to $L\cdot d$, therefore $\delta=L\in \mathbb N$. In words, for a fixed $\ell$, we collect the magnitude of the diffraction pattern of $\bx$ modulated by $\bd_{\ell}$. By varying $\ell$ and changing the modulation pattern $\bd_\ell$, we generate $L$ distinct views. The vectors $\{\bd_{\ell}\}_{1\le \ell\le L}$ are i.i.d. and their entries are also i.i.d. drawn uniformly from the set $\{1, -1, i, -i\}$. 

We test the spectral method on the digital photograph represented in Figure \ref{fig:ven}. Each color image can be viewed as a $d_1 \times d_2 \times 3$ array. We run the spectral algorithm separately on the vectors $\bx_j\in \mathbb R^{d}$, where $j\in\{1, 2, 3\}$. In our example, $d_1 = 820$ and $d_2=1280$. Let $\hat{\bx}_j$ be the estimate of $\bx_j$ provided by the spectral method. Then, we employ as a performance metric the average squared normalized scalar product
\begin{equation}
\frac{1}{3}\sum_{j=1}^3 \frac{|\langle\hat{\bx}_j, \bx_j\rangle|^2}{\norm{\hat{\bx}_j}^2 \norm{\bx_j}^2}.
\end{equation}
Note that the scalar product between the input and the measurement vectors can be interpreted as a 2-dimensional Fourier transform, hence it can be computed with an FFT algorithm. In order to evaluate the principal eigenvector of the data matrix, we use the power method with a random initialization, as described in Appendix B of \cite{candes2015wirt}. As a stopping criterion, we require that one of the following two conditions is fulfilled: either the number of iterations reaches the maximum value of $10000$, or the modulus of the scalar product between the estimate at the current iteration $T$ and at the iteration $T-10$ is larger than $1-10^{-7}$.  

The results are summarized in Figure \ref{fig:venscal}. The red curve corresponds to the proposed pre-processing function. In this case, the eigenvalues of the data matrix can be negative. Recall that the power method outputs the eigenvector associated to the largest eigenvalue \emph{in modulus}, while we are interested in the eigenvector associated to the largest eigenvalue. To address this issue, we add to the data matrix a multiple $\alpha$ of the identity matrix. However, as $\alpha$ grows, the convergence of the power method becomes slower and slower. In order to reduce the negative tail of the distribution of eigenvalues and, consequently, the value of $\alpha$, we pick the pre-processing function given by
\begin{equation}
\mathcal T_1(y) = \max(\mathcal T(y), -M),
\end{equation}
where $\mathcal T(y)$ is defined in \eqref{eq:numpre}, $\tilde{\delta} = 1.001$, and $M=40$. In this way, by taking $\alpha = 100$, the largest eigenvalue in modulus has positive sign.

The blue curve corresponds to the truncated spectral initialization in \cite{chen2017solving}, i.e., the pre-processing function is given by \eqref{eq:trimming} with $t=9$.

The numerical simulations for the optimal pre-processing function follow closely the theoretical predictions \eqref{eq:numpred} obtained for a Gaussian measurement matrix, with the exception of the point $\delta=2$. On the contrary, the numerical simulations for the truncated spectral initialization show a different behavior with respect to the Gaussian model. Our algorithm provides weak recovery of the original image for $\delta\ge 3$, while the truncated spectral initialization requires $\delta\ge 6$. Furthermore, for any value of $\delta$, the proposed choice of the pre-processing function yields a better performance than the choice in \cite{chen2017solving}. For a visual representation of these results, see Figure \ref{fig:venrec}. 

\section*{Acknowledgement}

M.~M. was supported by an Early Postdoc.Mobility fellowship from the Swiss National Science Foundation. A. M. was partially supported by grants NSF DMS-1613091 and NSF CCF-1714305. 

\appendix

\section{Proof of Corollary \ref{cor:lower}}\label{app:noiseless}

We start by providing in Lemma \ref{lemma:expl} a less compact, but more explicit form of the expression \eqref{eq:fnorm}. This more explicit expression is employed to prove Lemma \ref{lemma:deltalnoiseless}, which yields the value of $\delta_{\ell}$ for the case of phase retrieval. Finally, we provide the proof of Corollary \ref{cor:lower}.

\begin{lemma}[Explicit Formula for $f(m)$ - Complex Case]\label{lemma:expl}
Consider the function $f:[0, 1]\to \mathbb R$ defined in \eqref{eq:fnorm}. Then, $f(m)$ is given by the following expression:
\begin{equation}\label{eq:f}
f(m) = \bigints_{\!\!\mathbb R}\frac{\displaystyle\frac{1}{1-m}\int_0^{+\infty}\displaystyle\int_0^{+\infty}4r_1 r_2 \cdot p(y\mid r_1) p(y\mid r_2) \cdot  \exp\left(-\frac{r_1^2+r_2^2}{1-m}\right)\cdot I_0\left(\frac{2r_1r_2\sqrt{m}}{1-m}\right){\rm d}r_1{\rm d}r_2}{\displaystyle\int_0^{+\infty}2r \cdot p(y\mid r) \cdot  \exp\left(-r^2\right){\rm d}r} \,{\rm d}y,
\end{equation}
where $I_0$ denotes the modified Bessel function of the first kind, given by
\begin{equation}\label{eq:bessel}
I_0(x) = \frac{1}{\pi}\int_0^\pi \exp\left(x\cos \theta\right){\rm d}\theta.
\end{equation}
\end{lemma}

\begin{proof}
Let us rewrite $G$ as
\begin{equation*}
G = G^{(\rm R)}+ j G^{(\rm I)}, \quad \mbox{with }(G^{(\rm R)}, G^{(\rm I)})\sim \normal\left({\bm 0}_d, \frac{1}{2}\bI_2\right),
\end{equation*}
i.e., $G^{(\rm R)}$ and $G^{(\rm I)}$ are i.i.d. Gaussian random variables with mean 0 and variance $1/2$. Set
\begin{equation*}
R = \sqrt{\left(G^{(\rm R)}\right)^2+\left(G^{(\rm I)}\right)^2}.
\end{equation*}
Then, $R$ follows a Rayleigh distribution with scale parameter $1/\sqrt{2}$, hence 
\begin{equation}\label{eq:expA}
{\mathbb E}_{G}\left\{ p(y\mid |G|)\right\} ={\mathbb E}_{R}\left\{ p(y\mid R)\right\} = \int_0^{+\infty}2r \cdot p(y\mid r) \cdot  \exp\left(-r^2\right){\rm d}r.
\end{equation}

Let us rewrite $(G_{1}, G_{2})$ as
\begin{equation*}
(G_{1}, G_{2}) = (G_{1}^{(\rm R)}+ j G_{1}^{(\rm I)}, G_{2}^{(\rm R)}+ j G_{2}^{(\rm I)}),
\end{equation*}
with
\begin{equation*}
(G_{1}^{(\rm R)}, G_{2}^{(\rm R)}, G_{1}^{(\rm I)}, G_{2}^{(\rm I)})\sim \normal\left({\bm 0}_d, \frac{1}{2}\left[\begin{array}{llll}
1 & \Re{(c)} & 0 & -\Im{(c)} \\
\Re{(c)} & 1 & \Im{(c)} & 0 \\
0 & \Im{(c)} & 1 & \Re{(c)} \\
-\Im{(c)} & 0 & \Re{(c)} & 1 \\
\end{array}\right]\right),
\end{equation*}
and consider the following change of variables:
\begin{equation*}
\left\{\begin{array}{l}
G_{1}^{(\rm R)} = R_1 \cos \theta_1 \\
G_{2}^{(\rm R)} = R_2 \cos \theta_2 \\
G_{1}^{(\rm I)} = R_1 \sin \theta_1 \\
G_{2}^{(\rm I)} = R_2 \sin \theta_2 \\
\end{array} \right..
\end{equation*}
Then, after some algebra, we have that
\begin{equation}\label{eq:int}
\begin{split}
&{\mathbb E}_{G_{1}, G_{2}}\left\{ p(y\mid |G_{1}|)\cdot p(y\mid |G_{2}|)\right\} = \frac{1}{\pi^2(1-|c|^2)}\int_{0}^{+\infty}\int_{0}^{+\infty}\int_{0}^{2\pi}\int_{0}^{2\pi}r_1 r_2 \cdot p(y\mid r_1) p(y\mid r_2) \cdot  \\
&\hspace{7em}\exp\left(-\frac{r_1^2+r_2^2-2r_1r_2\left(\Re{(c)}\cos(\theta_2-\theta_1)-\Im{(c)}\sin(\theta_2-\theta_1)\right)}{1-|c|^2}\right)\,{\rm d}r_1\,{\rm d}r_2\,{\rm d}\theta_1\,{\rm d}\theta_2.
\end{split}
\end{equation}
By writing $(\Re{(c)}, \Im{(c)})=(|c|\cos\theta_c, |c|\sin\theta_c)$ and by using the definition \eqref{eq:bessel}, we can further simplify the RHS of \eqref{eq:int} as 
\begin{equation}\label{eq:int2}
\frac{1}{1-|c|^2}\int_0^{+\infty}\int_0^{+\infty}4r_1 r_2 \cdot p(y\mid r_1) p(y\mid r_2) \cdot  \exp\left(-\frac{r_1^2+r_2^2}{1-|c|^2}\right)\cdot I_0\left(\frac{2r_1r_2|c|}{1-|c|^2}\right){\rm d}r_1{\rm d}r_2.
\end{equation} 
From \eqref{eq:expA} and \eqref{eq:int2}, the claim easily follows.

\end{proof}

\begin{lemma}[Computation of $\delta_{\ell}$ for Phase Retrieval]\label{lemma:deltalnoiseless}
Let $\delta_{\ell}(\sigma^2)$ be defined as in \eqref{eq:defdelta} and assume that the distribution $p(\cdot\mid |G|)$ appearing in \eqref{eq:fnorm} is given by \eqref{eq:defypr}. Then,
\begin{equation}
\lim_{\sigma\to 0}\delta_{\ell}(\sigma^2)=1.
\end{equation}
\end{lemma}

\begin{proof}
For the special case of phase retrieval, it is possible to compute explicitly the function $f(m)$ defined in \eqref{eq:fnorm} and simplified in Lemma \ref{lemma:expl}. Indeed,  
\begin{equation}\label{eq:deltal1}
\begin{split}
\displaystyle\int_0^{+\infty}2r \cdot p_{\rm PR}(y\mid r) \cdot  \exp\left(-r^2\right)\,{\rm d}r & \stackrel{\mathclap{\mbox{\footnotesize(a)}}}{=} \displaystyle\int_0^{+\infty} p_{\rm PR}(y\mid \sqrt{z}) \cdot  \exp\left(-z\right)\,{\rm d}z\\
&\stackrel{\mathclap{\mbox{\footnotesize(b)}}}{=} \displaystyle\int_{-\infty}^{+\infty}\frac{1}{\sigma\sqrt{2\pi}}\exp\left(-\frac{(y-z)^2}{2\sigma^2}\right) \cdot  \exp\left(-z\right)\cdot H(z)\,{\rm d}z\\
&\stackrel{\mathclap{\mbox{\footnotesize(c)}}}{=} {\mathbb E}_Z \left\{\exp(-Z)H(Z)\right\},
\end{split}
\end{equation}
where in (a) we do the change of variables $z= r^2$; in (b) we use the definition \eqref{eq:defypr} and we define $H(x)=1$ if $x\ge 0$ and $H(x)=0$ otherwise; and in (c) we define $Z\sim \normal(y, \sigma^2)$. In the limit $\sigma^2\to 0$, we have that 
\begin{equation}\label{eq:deltal2}
{\mathbb E}_Z \left\{\exp(-Z)H(Z)\right\} \to \exp(-y)\cdot H(y),
\end{equation} 
by Lebesgue's dominated convergence theorem. Similarly,
\begin{equation*}
\begin{split}
\int_0^{+\infty}\displaystyle\int_0^{+\infty} & 4r_1 r_2 \cdot p_{\rm PR}(y\mid r_1) p_{\rm PR}(y\mid r_2)  \cdot  \exp\left(-\frac{r_1^2+r_2^2}{1-m}\right)\cdot I_0\left(\frac{2r_1r_2\sqrt{m}}{1-m}\right){\rm d}r_1{\rm d}r_2 \\
&\stackrel{\mathclap{\mbox{\footnotesize(a)}}}{=} \int_0^{+\infty}\displaystyle\int_0^{+\infty} p_{\rm PR}(y\mid \sqrt{z_1}) p_{\rm PR}(y\mid \sqrt{z_2}) \cdot  \exp\left(-\frac{z_1+z_2}{1-m}\right)\cdot I_0\left(\frac{2\sqrt{z_1 z_2}\sqrt{m}}{1-m}\right){\rm d}z_1{\rm d}z_2 \\
&\stackrel{\mathclap{\mbox{\footnotesize(b)}}}{=} \int_{-\infty}^{+\infty}\displaystyle\int_{-\infty}^{+\infty} \left(\frac{1}{\sigma\sqrt{2\pi}}\right)^2\exp\left(-\frac{(y-z_1)^2+(y-z_2)^2}{2\sigma^2}\right) \\
&\hspace{4em}\cdot  \exp\left(-\frac{z_1+z_2}{1-m}\right)\cdot I_0\left(\frac{2\sqrt{z_1 z_2}\sqrt{m}}{1-m}\right)\cdot H(z_1)H(z_2){\rm d}z_1{\rm d}z_2 \\
&\stackrel{\mathclap{\mbox{\footnotesize(c)}}}{=} {\mathbb E}_{Z_1, Z_2} \left\{\exp\left(-\frac{Z_1+Z_2}{1-m}\right)\cdot I_0\left(\frac{2\sqrt{Z_1 Z_2}\sqrt{m}}{1-m}\right)\cdot H(Z_1)H(Z_2)\right\},
\end{split}
\end{equation*}
where in (a) we do the change of variables $z_1= r_1^2$ and $z_2= r_2^2$; in (b) we use the definition \eqref{eq:defypr} and we define $H(x)=1$ if $x\ge 0$ and $H(x)=0$ otherwise; and in (c) we define $(Z_1, Z_2)\sim_{i.i.d.} \normal(y, \sigma^2)$. In the limit $\sigma^2\to 0$, we have that 
\begin{equation*}
\begin{split}
{\mathbb E}_{Z_1, Z_2} &\left\{\exp\left(-\frac{Z_1+Z_2}{1-m}\right)\cdot I_0\left(\frac{2\sqrt{Z_1 Z_2}\sqrt{m}}{1-m}\right)\cdot H(Z_1)H(Z_2)\right\}\\
&\hspace{12em}\to \exp\left(-\frac{2y}{1-m}\right)\cdot I_0\left(\frac{2y\sqrt{m}}{1-m}\right)\cdot H(y),
\end{split}
\end{equation*}
by Lebesgue's dominated convergence theorem.
As a result, by using \eqref{eq:bessel}, we obtain that 
\begin{equation*}
\begin{split}
f(m)  &\stackrel{\mathclap{\sigma^2\to 0}}{\longrightarrow} \frac{1}{1-m}\int_{0}^{+\infty} \exp\left(-\frac{2y}{1-m}\right)\cdot I_0\left(\frac{2y\sqrt{m}}{1-m}\right)\cdot \exp(y)\,{\rm d}y\\
& = \frac{1}{\pi(1-m)}\int_{0}^{\pi}\int_{0}^{+\infty}\exp\left(-y\left(\frac{1+m-2\sqrt{m}\cos\theta}{1-m}\right)\right)\,{\rm d}y\,{\rm d}\theta\\
&=\frac{1}{\pi}\int_{0}^{\pi}\frac{1}{1+m-2\sqrt{m}\cos\theta}\,{\rm d}\theta = \frac{1}{1-m}.
\end{split}
\end{equation*}
Consequently, 
\begin{equation*}
F_{\delta}(m) \stackrel{\mathclap{\sigma^2\to 0}}{\longrightarrow} (1-\delta)\log(1-m),
\end{equation*}
which implies the desired result.
\end{proof}

\begin{proof}[Proof of Corollary \ref{cor:lower}]
We follow the proof of Theorem \ref{th:lower} presented in Section \ref{sec:mainproof}. The first step is exactly the same, i.e., by applying Lemma \ref{lemma:ent}, we show that \eqref{eq:1step} holds. On the contrary, the second step requires some modifications, since the definition of the error metric is different. In particular, we will prove that 
\begin{equation}\label{eq:2stepnew}
\frac{1}{d^2} {\mathbb E}_{\bY_{1:n}, \bA_{1:n}}\left\{\norm{{\mathbb E}\left\{\bX \bX^*\right\} - {\mathbb E}\left\{\bX \bX^*\mid \bY_{1:n}, \bA_{1:n}\right\}}_{F}^2\right\}=o_n(1).
\end{equation}

Furthermore, we have that 
\begin{equation}\label{eq:3stepnew}
\begin{split}
 \norm{{\mathbb E}\left\{\bX \bX^*\right\} - {\mathbb E}\left\{\bX \bX^*\mid \bY_{1:n}, \bA_{1:n}\right\}}_{F}^2 &+ {\mathbb E}\left\{\norm{\bX \bX^* - {\mathbb E}\left\{\bX \bX^*\mid \bY_{1:n}, \bA_{1:n}\right\}}_{F}^2 \right\}\\
 &\hspace{3em}\stackrel{\mathclap{\mbox{\footnotesize(a)}}}{\ge} {\mathbb E}\left\{\norm{{\mathbb E}\left\{\bX \bX^*\right\} - \bX \bX^*}_{F}^2\right\}\\
&\hspace{3em}\stackrel{\mathclap{\mbox{\footnotesize(b)}}}{=}{\mathbb E}\left\{\norm{\bI_d - \bX \bX^*}_{F}^2\right\}\\
&\hspace{3em}\stackrel{\mathclap{\mbox{\footnotesize(c)}}}{=}{\mathbb E}\left\{{\rm trace}\left(\bI_d - 2\bX \bX^* + \bX \bX^*\bX \bX^*\right)\right\}\\
&\hspace{3em}\stackrel{\mathclap{\mbox{\footnotesize(d)}}}{=}d-2d+d^2=d^2-d,
\end{split}
\end{equation}
where in (a) we use the triangle inequality, in (b) we use that $\mathbb E \left\{\bX \bX^*\right\}=\bI_d$ by Lemma \ref{lemma:unit}, in (c) we use that, for any matrix $\bA$, $\norm{\bA}_F=\sqrt{{\rm trace}(\bA \bA^*)}$, and in (d) we use that ${\mathbb E}\left\{{\rm trace}\left(\bX \bX^*\bX \bX^*\right)\right\}={\mathbb E}\left\{\norm{\bX}^4\right\}=d^2$. By applying \eqref{eq:2stepnew} and \eqref{eq:3stepnew}, the proof of Corollary \ref{cor:lower} is complete. 

Let us now give the proof of \eqref{eq:2stepnew}. Similarly to \eqref{eq:distances}, we have that
\begin{equation}\label{eq:distancesnew}
\begin{split}
&I(Y_{n+1}; \bY_{1:n}, \bA_{1:n}| \bA_{n+1}) \\
&\ge \frac{1}{2K^2}\cdot{\mathbb E}_{\bY_{1:n}, \bA_{1:n+1}}\Biggl\{ \biggl\lvert\int_{\mathbb C^d}p(\bx\mid \bY_{1:n}, \bA_{1:n})\int_{\mathbb R}p_{\rm PR}(y_{n+1}\mid \bx, \bY_{1:n}, \bA_{1:n+1})\varphi_{\rm PR}(y_{n+1}){\rm d}y_{n+1}{\rm d}\bx\\
&\hspace{9em}-\int_{\mathbb C^d}p(\bx)\int_{\mathbb R}p_{\rm PR}(y_{n+1}\mid \bx, \bA_{n+1})\varphi_{\rm PR}(y_{n+1})\,{\rm d}y_{n+1}\,{\rm d}\bx\biggr\rvert^2\Biggr\},\\
\end{split}
\end{equation}
where we define $\varphi_{\rm PR}(x)=x$ for $|x|\le M$, and $\varphi_{\rm PR}(x)=M\cdot {\rm sign}(x)$ otherwise. Then,
\begin{equation}\label{eq:integ1}
\begin{split}
\int_{\mathbb C^d}  p(\bx\mid \bY_{1:n}, \bA_{1:n})&\int_{\mathbb R}p_{\rm PR}(y_{n+1}\mid \bx, \bY_{1:n}, \bA_{1:n+1})\cdot \varphi_{\rm PR}(y_{n+1})\,{\rm d}y_{n+1}\,{\rm d}\bx \\
&\stackrel{\mathclap{\mbox{\footnotesize(a)}}}{=}\int_{\mathbb C^d}p(\bx\mid \bY_{1:n}, \bA_{1:n})\int_{\mathbb R}p_{\rm PR}(y_{n+1}\mid \bx, \bY_{1:n}, \bA_{1:n+1})\cdot y_{n+1}\,{\rm d}y_{n+1}\,{\rm d}\bx + E_1 \\
&\stackrel{\mathclap{\mbox{\footnotesize(b)}}}{=} \int_{\mathbb C^d}p(\bx\mid \bY_{1:n}, \bA_{1:n})\cdot |\langle \bA_{n+1}, \bx \rangle|^2\,{\rm d}\bx+E_1\\
&\stackrel{\mathclap{\mbox{\footnotesize(c)}}}{=} \langle \bA_{n+1}, \left(\int_{\mathbb C^d}p(\bx\mid \bY_{1:n}, \bA_{1:n})\cdot \bx \bx^*\,{\rm d}\bx\right) \bA_{n+1}\rangle +E_1\\
&= \langle \bA_{n+1}, {\mathbb E}\left\{\bX \bX^*\mid \bY_{1:n}, \bA_{1:n}\right\} \bA_{n+1}\rangle +E_1,
\end{split}
\end{equation}
where in (a) we set
\begin{equation*}
E_1 = \int_{\mathbb C^d}p(\bx\mid \bY_{1:n}, \bA_{1:n})\int_{\mathbb R}p_{\rm PR}(y_{n+1}\mid \bx, \bY_{1:n}, \bA_{1:n+1})\cdot (\varphi_{\rm PR}(y_{n+1})-y_{n+1})\,{\rm d}y_{n+1}\,{\rm d}\bx,
\end{equation*}
in (b) we use the definition \eqref{eq:defypr}, and in (c) we use that  $|\langle \bA_{n+1}, \bx\rangle|^2=\langle \bA_{n+1}, \bx \bx^* \bA_{n+1}\rangle$. Similarly, we have that
\begin{equation}\label{eq:integ2}
\int_{\mathbb C^d}p(\bx)\int_{\mathbb R}p_{\rm PR}(y_{n+1}\mid \bx, \bA_{n+1})\varphi_{\rm PR}(y_{n+1})\,{\rm d}y_{n+1}\,{\rm d}\bx = \langle \bA_{n+1}, {\mathbb E}\left\{\bX \bX^*\right\} \bA_{n+1}\rangle +E_2,
\end{equation}
with
\begin{equation*}
E_2 = \int_{\mathbb C^d}p(\bx)\int_{\mathbb R}p_{\rm PR}(y_{n+1}\mid \bx, \bA_{n+1})\cdot (\varphi_{\rm PR}(y_{n+1})-y_{n+1})\,{\rm d}y_{n+1}\,{\rm d}\bx.
\end{equation*}

By applying \eqref{eq:integ1} and \eqref{eq:integ2}, we can rewrite the RHS of \eqref{eq:distancesnew} as
\begin{equation}\label{eq:cor1}
\begin{split}
\frac{1}{2K^2}&\cdot{\mathbb E}_{\bY_{1:n}, \bA_{1:n}}{\mathbb E}_{\bA_{n+1}}\left\{\left\lvert\langle \bA_{n+1}, \left({\mathbb E}\left\{\bX \bX^*\mid \bY_{1:n}, \bA_{1:n}\right\}-{\mathbb E}\left\{\bX \bX^*\right\}\right) \bA_{n+1}\rangle +E_1-E_2\right\rvert^2\right\}\\
&\ge \frac{1}{2K^2}\cdot{\mathbb E}_{\bY_{1:n}, \bA_{1:n}}\left({\mathbb E}_{\bA_{n+1}}\left\{\left\lvert\langle \bA_{n+1}, \bM \bA_{n+1}\rangle\right\rvert^2\right\}-{\mathbb E}_{\bA_{n+1}}\left\{|E_1|^2\right\}-{\mathbb E}_{\bA_{n+1}}\left\{|E_2|^2\right\}\right),
\end{split}
\end{equation}
where we define $\bM = {\mathbb E}\left\{\bX \bX^*\mid \bY_{1:n}, \bA_{1:n}\right\}-{\mathbb E}\left\{\bX \bX^*\right\}$. As $K$ goes large, ${\mathbb E}_{\bA_{n+1}}\left\{|E_i|^2\right\}$ tends to $0$, for $i\in \{1, 2\}$. Furthermore, we have that 
\begin{equation}
\begin{split}
{\mathbb E}_{\bA_{n+1}}\left\{\left\lvert\langle \bA_{n+1}, \bM \bA_{n+1}\rangle\right\rvert^2\right\} &\stackrel{\mathclap{\mbox{\footnotesize(a)}}}{=} \sum_{i, j, k, l=1}^d M_{ij} M^*_{kl}\cdot \frac{1}{d^2}(\delta_{ij}\cdot\delta_{kl}+\delta_{il}\cdot\delta_{jk}) \\
&= \frac{1}{d^2}\left(\lvert{\rm trace}(\bM)\rvert^2+\norm{\bM}_F^2\right)\\
&\stackrel{\mathclap{\mbox{\footnotesize(b)}}}{=} \frac{1}{d^2}\norm{\bM}_F^2,
\end{split}
\end{equation}
where in (a) we use the following definition of the Kronecker delta:
\begin{equation}
\delta_{ab} = \left\{\begin{array}{ll}
1, & \mbox{ if }a=b,\\
0, & \mbox{ otherwise},
\end{array}\right.
\end{equation} 
and in (b) we use that 
\begin{equation}
\begin{split}
{\rm trace}(\bM) &= \sum_{i=1}^d \left({\mathbb E}\left\{|X_i|^2 \mid \bY_{1:n}, \bA_{1:n}\right\}-{\mathbb E}\left\{|X_i|^2\right\}\right)\\
&= {\mathbb E}\left\{\sum_{i=1}^d|X_i|^2 \mid \bY_{1:n}, \bA_{1:n}\right\}-{\mathbb E}\left\{\sum_{i=1}^d|X_i|^2\right\}=0.
\end{split}
\end{equation} 
As a result, we conclude that \eqref{eq:2stepnew} holds. 
\end{proof}

\section{Proof of Corollary \ref{cor:upper}}\label{app:noiseless2}

First, we evaluate the RHS of \eqref{eq:defdeltau}, as well as the scaling between $\delta_{\rm u}$ and $\sigma^2$ when $\sigma^2\to 0$. Then, we give the proof of Corollary \ref{cor:upper}.

\begin{lemma}[Computation of $\delta_{\rm u}$ for Phase Retrieval]\label{lemma:deltaulnoiseless}
Let $\delta_{\rm u}(\sigma^2)$ be defined as in \eqref{eq:defdeltau} and assume that the distribution $p(\cdot\mid |g|)$ is given by \eqref{eq:defypr}. Then,
\begin{equation}\label{eq:deltaunoiseless}
\delta_{\rm u}(\sigma^2)=1+\sigma^2 + o(\sigma^2).
\end{equation}
\end{lemma}

\begin{proof}
The proof boils down to computing expected values and integrals. By using \eqref{eq:expA} and \eqref{eq:deltal1}, we immediately obtain that
\begin{equation*}
{\mathbb E}_{G}\left\{ p_{\rm PR}(y\mid |G|)\right\} = \int_0^{+\infty}2r \cdot p_{\rm PR}(y\mid r) \cdot  \exp\left(-r^2\right){\rm d}r= \exp(-y)\,{\mathbb E}_{X}\left\{\exp(-\sigma X)H(y+\sigma X)\right\},
\end{equation*}
where $X\sim \normal(0, 1)$. By computing explicitly the expectation, we deduce that 
\begin{equation}\label{eq:intexpcom1}
{\mathbb E}_{G}\left\{ p_{\rm PR}(y\mid |G|)\right\} =\frac{1}{2}\exp\left(-y+\frac{\sigma^2}{2}\right){\rm erfc}\left(\frac{1}{\sqrt{2}}\left(-\frac{y}{\sigma}+\sigma\right)\right),
\end{equation}
where $\rm erfc(\cdot)$ is the complimentary error function. Similarly, we have that
\begin{equation}\label{eq:intexpcom2}
\begin{split}
{\mathbb E}_{G}&\left\{ p_{\rm PR}(y\mid |G|)(|G|^2-1)\right\} = \int_0^{+\infty}2(r^3-r)  p_{\rm PR}(y\mid r) \cdot  \exp\left(-r^2\right){\rm d}r\\
&=\exp(-y)\,{\mathbb E}_{X}\left\{\exp(-\sigma X)H(y+\sigma X)(y-1+\sigma X)\right\}\\
&= \frac{\sigma}{\sqrt{2\pi}} \exp\left(-\frac{y^2}{2\sigma^2}\right)+\frac{1}{2}(y-1-\sigma^2)\exp\left(-y+\frac{\sigma^2}{2}\right){\rm erfc}\left(\frac{1}{\sqrt{2}}\left(-\frac{y}{\sigma}+\sigma\right)\right).
\end{split}
\end{equation}
Thus, by using \eqref{eq:intexpcom1} and \eqref{eq:intexpcom2}, after some manipulations, we obtain that 
\begin{equation}\label{eq:RHSint1}
\begin{split}
\frac{1}{\delta_{\rm u}}=\displaystyle\bigintssss_{\mathbb R}\frac{\left({\mathbb E}_{G}\left\{p_{\rm PR}(y\mid |G|)(|G|^2-1)\right\}\right)^2}{{\mathbb E}_{G}\left\{p(y_{\rm PR}\mid |G|)\right\}}& \,{\rm d}y = \int_{\mathbb R} \frac{\sigma^2}{2\pi} \exp\left(y-\frac{\sigma^2}{2}-\frac{y^2}{\sigma^2}\right)\frac{2}{{\rm erfc}\left(\frac{1}{\sqrt{2}}\left(-\frac{y}{\sigma}+\sigma\right)\right)}{\rm d}y\\
&\hspace{-3em}+\int_{\mathbb R} \frac{2\sigma}{\sqrt{2\pi}} \exp\left(-\frac{y^2}{2\sigma^2}\right)(y-1-\sigma^2){\rm d}y\\
&\hspace{-3em}+ \int_{\mathbb R} \frac{1}{2}\exp\left(-y+\frac{\sigma^2}{2}\right)(y-1-\sigma^2)^2{\rm erfc}\left(\frac{1}{\sqrt{2}}\left(-\frac{y}{\sigma}+\sigma\right)\right){\rm d}y.
\end{split}
\end{equation}
By performing the change of variables $t=-y/\sigma+\sigma$, we simplify the first integral in the RHS of \eqref{eq:RHSint1} as 
\begin{equation}\label{eq:RHSint2}
\int_{\mathbb R} \frac{\sigma^2}{2\pi} \exp\left(y-\frac{\sigma^2}{2}-\frac{y^2}{\sigma^2}\right)\frac{2}{{\rm erfc}\left(\frac{1}{\sqrt{2}}\left(-\frac{y}{\sigma}+\sigma\right)\right)}{\rm d}y = \frac{2\sigma^3}{2\pi}\int_{\mathbb R} \frac{\exp\left(-t^2\right)}{{\rm erfc}\left(\frac{t}{\sqrt{2}}\right)}\exp\left(\sigma t-\frac{\sigma^2}{2}\right)\,{\rm d}t = o(\sigma^2),
\end{equation}
where in the last equality we use that the integral 
\begin{equation*}
\int_{\mathbb R} \frac{\exp\left(-t^2\right)}{{\rm erfc}\left(\frac{t}{\sqrt{2}}\right)}\,{\rm d}t
\end{equation*}
is finite. The other two integrals in the RHS of \eqref{eq:RHSint1} can be expressed in closed form as
\begin{equation}\label{eq:RHSint3}
\int_{\mathbb R} \frac{2\sigma}{\sqrt{2\pi}} \exp\left(-\frac{y^2}{2\sigma^2}\right)(y-1-\sigma^2){\rm d}y = -2\sigma^2(1+\sigma^2),
\end{equation}
\begin{equation}\label{eq:RHSint4}
\begin{split}
\int_{\mathbb R} \frac{1}{2}\exp\left(-y+\frac{\sigma^2}{2}\right)&(y-1-\sigma^2)^2{\rm erfc}\left(\frac{1}{\sqrt{2}}\left(-\frac{y}{\sigma}+\sigma\right)\right){\rm d}y\\
&=\frac{\sigma}{2}\exp\left(-\frac{\sigma^2}{2}\right)\int_{\mathbb R} \exp\left(\sigma t\right)(\sigma t+1)^2{\rm erfc}\left(\frac{t}{\sqrt{2}}\right){\rm d}t=1+\sigma^2+\sigma^4.
\end{split}
\end{equation}
By combining \eqref{eq:RHSint1}, \eqref{eq:RHSint2}, \eqref{eq:RHSint3} and \eqref{eq:RHSint4}, the result follows. 
\end{proof}

\begin{proof}[Proof of Corollary \ref{cor:upper}]
Pick $\sigma$ sufficiently small. Let $G\sim \cnormal(0, 1)$, $Y\sim p_{\rm PR}(\cdot \mid |G|)$ and  $Z= \mathcal T(Y)$, where $p_{\rm PR}$ is defined in \eqref{eq:defypr} and $\mathcal T$ is a pre-processing function 
(possibly dependent on $\sigma$) that we will choose later on. Assume that
\begin{enumerate}[(1)]
\item $\mathcal T(y)$ is upper and lower bounded by constants independent of $\sigma$;
\item $\mathbb P (Z=0)\le c_1 <1$ and $c_1$ is independent of $\sigma$;
\item the condition \eqref{eq:hplemmaub1} holds.
\end{enumerate}
Then, by Lemma \ref{lemma:condub}, we have that, as $n\to \infty$,  
\begin{equation}\label{eq:scalprodfull}
\frac{|\langle\hat{\bx}, \bx\rangle|^2}{\norm{\hat{\bx}}_2^2 \norm{\bx}^2} \stackrel{\mathclap{\mbox{\footnotesize a.s.}}}{\longrightarrow}\rho = \left\{\begin{array}{ll}
\vspace{1em}
0, & \mbox{ if }\psi_{\delta}'(\lambda_{\delta}^*)\le 0,\\
\displaystyle\frac{\psi_{\delta}'(\lambda_{\delta}^*)}{\psi_{\delta}'(\lambda_{\delta}^*)-\phi'(\lambda_{\delta}^*)}, & \mbox{ if }\psi_{\delta}'(\lambda_{\delta}^*)> 0,\\
\end{array}\right.
\end{equation}
where $\lambda_{\delta}^*$ is the unique solution of the equation $\zeta_\delta(\lambda)=\phi(\lambda)$, and $\phi$, $\psi_\delta$ and $\zeta_\delta$ are defined in \eqref{eq:defphi2}, \eqref{eq:defpsi} and \eqref{eq:defzeta}, respectively. 

Let $\tau$ be the supremum of the support of $Z$. Assume also that, for $\bar{\lambda}_\delta< \lambda_{\delta}^*$,
\begin{enumerate}[(4)]
\item $\tau \ge c_2 > 0$ and $c_2$ is independent of $\sigma$;
\item[(5)] $\phi'(\lambda_{\delta}^*)$ is lower bounded by a constant independent of $\sigma$;
\item[(6)] $\displaystyle\min_{\lambda\in (\min(\bar{\lambda}_\delta, \lambda_{\delta}^*))}\psi_{\delta}''(\lambda)$ is lower bounded by a strictly positive constant independent of $\sigma$.
\end{enumerate}

Let $\bar{\lambda}_\delta$ be the point at which $\psi_{\delta}$ attains its minimum. Then,
\begin{equation}\label{eq:scalprodfull2}
\begin{split}
\phi(\bar{\lambda}_\delta) -\psi_\delta(\bar{\lambda}_\delta)& \stackrel{\mathclap{\mbox{\footnotesize (a)}}}{=} \phi(\bar{\lambda}_\delta) - \phi(\lambda_{\delta}^*) + \zeta_\delta(\lambda_{\delta}^*) - \psi_\delta(\bar{\lambda}_\delta)\\
&\stackrel{\mathclap{\mbox{\footnotesize (b)}}}{=} \phi(\bar{\lambda}_\delta) - \phi(\lambda_{\delta}^*) + \zeta_\delta(\lambda_{\delta}^*) - \zeta_\delta(\bar{\lambda}_\delta)\\
&\stackrel{\mathclap{\mbox{\footnotesize (c)}}}{=}\bigl(\zeta_\delta'(x_1)-\phi'(x_1)\bigr)\cdot(\lambda_{\delta}^* - \bar{\lambda}_\delta)\\
&\stackrel{\mathclap{\mbox{\footnotesize (d)}}}{=}\frac{\bigl(\zeta_\delta'(x_1)-\phi'(x_1)\bigr)}{\psi_{\delta}''(x_2)}\cdot\bigl(\psi_{\delta}'(\lambda_{\delta}^*) - \psi_{\delta}'(\bar{\lambda}_\delta)\bigr)\\
&\stackrel{\mathclap{\mbox{\footnotesize (e)}}}{=}\frac{\bigl(\zeta_\delta'(x_1)-\phi'(x_1)\bigr)}{\psi_{\delta}''(x_2)}\cdot\psi_{\delta}'(\lambda_{\delta}^*)\\
&\stackrel{\mathclap{\mbox{\footnotesize (f)}}}{\le} c_3\cdot\psi_{\delta}'(\lambda_{\delta}^*),
\end{split}
\end{equation}
where in (a) we use that $\zeta_\delta(\lambda_{\delta}^*)=\phi(\lambda_{\delta}^*)$, in (b) we use that $\zeta_\delta(\bar{\lambda}_\delta)=\psi_\delta(\bar{\lambda}_\delta)$, (c) holds for some $x_1 \in (\bar{\lambda}_\delta, \lambda_{\delta}^*)$ by the mean value theorem, (d) holds for some $x_2 \in (\bar{\lambda}_\delta, \lambda_{\delta}^*)$ by the mean value theorem, and in (e) we use that $\psi_{\delta}'(\bar{\lambda}_\delta)=0$. Note that (f) holds for some constant $c_3$ independent of $\sigma$, as $\zeta_\delta'(x_1)\ge 0$, $\psi_{\delta}''(x_2)$ is bounded, and $\phi'(x_2)<0$ since $\mathbb P (Z=0) <1$. 

As $\phi'(\lambda_{\delta}^*)$ is bounded, from \eqref{eq:scalprodfull} and \eqref{eq:scalprodfull2} we deduce that
\begin{equation}\label{eq:rho1}
\rho\ge c_4 \cdot\bigl(\phi(\bar{\lambda}_\delta) -\psi_\delta(\bar{\lambda}_\delta)\bigr),
\end{equation}
for some constant $c_4$ independent of $\sigma$. Notice that, if $\lambda_{\delta}^*\le \bar{\lambda}_{\delta}$, then the right hand side is non-positive and hence the lower bound still holds.

As $\tau>0$, we also have that $\bar{\lambda}_\delta>0$. Consider now the matrix $\bD_n'=\bD_n/\alpha$ for some $\alpha>0$. Then, the principal eigenvector of $\bD_n'$ is equal to the principal eigenvector of $\bD_n$. Hence, we can assume without loss of generality that $\bar{\lambda}_\delta=1$. This condition can be rewritten as
\begin{equation}\label{eq:condubnew1cor}
{\mathbb E}\left\{\frac{Z^2}{(1-Z)^2}\right\}=\frac{1}{\delta},
\end{equation}
and \eqref{eq:rho1} can be rewritten as
\begin{equation}\label{eq:condubnew2cor}
\rho \ge c_4 \cdot \left({\mathbb E}\left\{\frac{Z(|G|^2-1)}{1-Z}\right\} -\frac{1}{\delta}\right).
\end{equation}

We set 
\begin{equation}\label{eq:defsigma}
\mathcal T(y)= \mathcal T_\delta^*(y, \sigma) \triangleq \frac{y_+-1}{y_++\sqrt{\delta\, c(\sigma)}-1},
\end{equation}
where $y_+=\max(y,0)$ and $c(\sigma)$ is a function of $\sigma$ to be set as to satisfy Eq. ~\eqref{eq:condubnew1cor}.
By substituting \eqref{eq:defsigma} into \eqref{eq:condubnew1cor}, we get 
\begin{align}
\E \left\{\frac{Z^2}{(1-Z)^2}\right\}= \frac{1}{\delta c(\sigma)}\E\big\{(Y_+-1)^2\big\}\, .
\end{align}
Hence, Eq. ~\eqref{eq:condubnew1cor} is satisfied by
\begin{align}
c(\sigma) = \E\big\{(Y_+-1)^2\big\} = \E\Big\{\big((|G|^2+\sigma W)_+-1\big)^2\Big\}\, ,
\end{align}
where $W\sim\normal(0,1)$ is independent of $G$.
Therefore, $c(\sigma)$ is always well defined and, by dominated convergence, $c(\sigma)\to c(0) = 1$, as $\sigma\to 0$.
Furthermore,
\begin{align}
\E\left\{\frac{Z(|G|^2-1)}{1-Z}\right\} &= \frac{1}{\sqrt{\delta c(\sigma)} } \E\big\{(Y_+-1)(|G|^2-1)\big\}.
\end{align}
By applying again dominated convergence, we get 
\begin{align}
\lim_{\sigma\to 0}\E\left\{\frac{Z(|G|^2-1)}{1-Z}\right\} &= \frac{1}{\sqrt{\delta } } \E\big\{(|G|^2-1)^2\big\} = \frac{1}{\sqrt{\delta}}\, .
\end{align}
Hence, by using \eqref{eq:rho1}, we get that, for $\delta>1$ and $\sigma\le \sigma_1(\delta)$,
\begin{equation}
\liminf_{n\to\infty}\frac{|\langle\hat{\bx}_{\sigma}, \bx\rangle|^2}{\norm{\hat{\bx}_\sigma}_2^2 \norm{\bx}_2^2} \ge c_5\Big(\frac{1}{\sqrt{\delta}}-\frac{1}{\delta}\Big)>0\, ,
\end{equation}
where $\hat{\bx}_{\sigma}$ denotes the spectral estimator corresponding to the pre-processing function (\ref{eq:defsigma}).
Let us now verify that, by setting $\mathcal T=\mathcal T_\delta^*$, the requirements stated above are fulfilled. As $\delta>1$, the function $\mathcal T$ is bounded by constants independent of $\sigma$. It is also clear that the conditions (2) and (4) hold. Furthermore, the conditions (5) and (6) follow by showing that $\phi(\lambda)$, $\psi_{\delta}(\lambda)$ have well defined uniform limits as $\sigma\to 0$ that satisfy those conditions:
this can be proved by one more application of dominated convergence.

In order to show that the condition (3) holds, we follow the argument presented at the end of the proof of Lemma \ref{lemma:condub}. First, we add a point mass with associated probability at most $\epsilon_1$, which immediately implies that \eqref{eq:hplemmaub1} is satisfied. Then, by applying the Davis-Kahan theorem \cite{davis1970rotation}, we show that we can take $\epsilon_1 = 0$.

This proves the claim of the corollary for the pre-processing function $\cT_\delta^*(y, \sigma)$, defined in \eqref{eq:defsigma}. Let us now prove that the same conclusion holds for
$\cT^*_{\delta}(y)$ defined in \eqref{eq:Tstardeltapr}. 
Let
\begin{equation}
f_a(x) = \frac{x+1}{x+a}.
\end{equation}
Then, for any $x, a \in \reals_{\ge  0}$,
\begin{equation}
|f_a'(x)|= \frac{|a-1|}{(x+a)^2}\le \max\left(1,\frac{1}{x^2}\right).
\end{equation}
Therefore, since $\cT_\delta^*(y,\sigma) = 1-f_{y_+}(\sqrt{\delta c(\sigma)}-1)$, we have that
\begin{align}
\sup_{y\in\reals}\big|\cT_\delta^*(y, \sigma)-\cT^*_\delta(y)\big|\le \frac{1}{\bigl(\min(\sqrt{\delta c(\sigma)}-1, \sqrt{\delta}-1, 1)\bigr)^2}\,\sqrt{\delta}\cdot \big|\sqrt{c(\sigma)}-1\big|\, .
\end{align}
Denote by $\bD_n(\sigma)$ and $\bD_n$ the matrices constructed with the pre-processing functions $\cT_\delta^*(y, \sigma)$ and $\cT^*_\delta(y)$, respectively. It follows that, for any $\delta>1$, 
there exists a function $\Delta(\sigma)$ with $\Delta(\sigma)\to 0$ as $\sigma\to 0$ such that
\begin{align}
\norm{\bD_n(\sigma)-\bD_n}_{\rm op} \le \Delta(\sigma)\, .
\end{align}
Hence, by applying again the Davis-Kahan theorem,  we conclude that, for all $\delta>1$ and $\sigma\le \sigma_2(\delta)$,
\begin{equation}
\liminf_{n\to\infty}\frac{|\langle\hat{\bx}, \bx\rangle|^2}{\norm{\hat{\bx}}_2^2 \norm{\bx}_2^2} \ge c_5\Big(\frac{1}{\sqrt{\delta}}-\frac{1}{\delta}\Big)>0\, ,
\end{equation}
where $\hat{\bx}$ is the estimator corresponding to the pre-processing function $\cT^*_{\delta}(y)$.
\end{proof}

\section{Auxiliary Lemmas}\label{app:distribution}

\begin{lemma}[Distribution of Scalar Product of Two Unit Complex Vectors]\label{lemma:beta}
Let $\bx_1, \bx_2 \sim_{i.i.d.} \Unif(\Sphere_{\mathbb C}^{d-1})$ and define $M = |\langle \bx_1, \bx_2 \rangle|^2$. Then,
\begin{equation}
M\sim {\rm Beta}(1, d-1).
\end{equation}
\end{lemma}

\begin{proof}
Without loss of generality, we can pick $\bx_2$ to be the first element of the canonical base of $\mathbb C^d$. Thus, $M$ is equal to the squared modulus of the first component of $\bx_1$. Furthermore, we can think to $\bx_1$ as being chosen uniformly at random on the $2d$-dimensional real sphere with radius $1$. Note that, by taking a vector of i.i.d. standard Gaussian random variables and dividing it by its norm, we obtain a vector uniformly random on the sphere of radius $1$. Hence,
\begin{equation*}
M = \frac{U_1^2+U_2^2}{\sum_{i=1}^{2d} U_i^2}, \quad \mbox{with } \{U_i\}_{1\le i \le 2d} \sim_{i.i.d.} \normal(0, 1).
\end{equation*}
Set $A = U_1^2+U_2^2$ and $B = \sum_{i=3}^{2d} U_i^2$. Then, $A$ and $B$ are independent, $A$ follows a Gamma distribution with shape $1$ and scale $2$, i.e., $A\sim \Gamma(1, 2)$, and $B$ follows a Gamma distribution with shape $d-1$ and scale $2$, i.e., $B\sim \Gamma(d-1, 2)$. Thus, we conclude that
\begin{equation*}
M = \frac{A}{A+B}\sim{\rm Beta}(1, d-1),
\end{equation*} 
which proves the claim.
\end{proof}

\begin{lemma}[Distribution of Scalar Product of Two Unit Real Vectors]\label{lemma:betar}
Let $\bx_1, \bx_2 \sim_{i.i.d.} \Unif(\Sphere_{\mathbb R}^{d-1})$ and define $M = \langle \bx_1, \bx_2 \rangle$. Then, the distribution of $M$ is given by 
\begin{equation}
p(m) = \frac{\Gamma(\frac{d}{2})}{\sqrt{\pi}\Gamma(\frac{d-1}{2})}(1-m^2)^{\frac{d-3}{2}}, \qquad\qquad m\in [-1, 1].
\end{equation}
\end{lemma}

\begin{proof}
Without loss of generality, we can pick $\bx_2$ to be the first element of the canonical base of $\mathbb R^d$. Thus, $M$ is equal to the first component of $\bx_1$. Note that, by taking a vector of i.i.d. standard Gaussian random variables and dividing it by its norm, we obtain a vector uniformly random on the sphere of radius $1$. Hence,
\begin{equation*}
M^2 = \frac{U_1^2}{\sum_{i=1}^{d} U_i^2}, \quad \mbox{with } \{U_i\}_{1\le i \le d} \sim_{i.i.d.} \normal(0, 1).
\end{equation*}
Set $A = U_1^2$ and $B = \sum_{i=2}^{d} U_i^2$. Then, $A$ and $B$ are independent, $A$ follows a Gamma distribution with shape $1/2$ and scale $2$, i.e., $A\sim \Gamma(1/2, 2)$, and $B$ follows a Gamma distribution with shape $(d-1)/2$ and scale $2$, i.e., $B\sim \Gamma((d-1)/2, 2)$. Thus, we obtain that
\begin{equation*}
M^2 = \frac{A}{A+B}\sim{\rm Beta}(1/2, (d-1)/2).
\end{equation*} 
A change of variable and the observation that the distribution of $M$ is symmetric around $0$ immediately let us conclude that, for $m\in [-1, 1]$,
\begin{equation}
p(m) = c \cdot (1-m^2)^{\frac{d-3}{2}},
\end{equation}
where the normalization constant $c$ is given by
\begin{equation*}
c=\left(\int_{-1}^1(1-m^2)^{\frac{d-3}{2}}{\rm d}m\right)^{-1} = \frac{\Gamma(\frac{d}{2})}{\sqrt{\pi}\Gamma(\frac{d-1}{2})}.
\end{equation*}
\end{proof}

\begin{lemma}[Laplace's Method]\label{lemma:laplace}
Let $F:[0, 1]\to \mathbb R$ be such that
\begin{itemize}
	\item $F$ is continuous;
	\item $F(x)<0$ for $x\in (0, 1]$;
	\item $F(0)=0$.
\end{itemize}
Then,
\begin{equation}
\lim_{n\to +\infty}\int_0^1 \exp\left(n\cdot F(x)\right)\,{\rm d}x=0.
\end{equation}
\end{lemma}

\begin{proof}
Pick $\epsilon>0$ and separate the integral into two parts:
\begin{equation*}
\int_0^1\exp\left(n\cdot F(x)\right)\,{\rm d}x= \int_0^\epsilon\exp\left(n\cdot F(x)\right)\,{\rm d}x+\int_\epsilon^1\exp\left(n\cdot F(x)\right)\,{\rm d}x.
\end{equation*}
Now, the first integral is at most $\epsilon$ since $F(x)\le 0$ for any $x\in[0, 1]$, and the second integral tends to $0$ as $n\to +\infty$ since $F(x)<0$ for $x\in(0, 1]$. Thus, the claim immediately follows.
\end{proof}

\begin{lemma}[Second Moment of Uniform Vector on Complex Sphere]\label{lemma:unit}
Let $\bx\sim \Unif(\sqrt{d}\Sphere_{\mathbb C}^{d-1})$. Then,
\begin{equation}
\mathbb E \left\{\bX \bX^*\right\}=\bI_d.
\end{equation}
\end{lemma}

\begin{proof}
Let $\bz \sim \cnormal({\bm 0}_d, \bI_d)$ and note that, by taking a vector of i.i.d. standard complex normal random variables and dividing it by its norm, we obtain a vector uniformly random on the complex sphere of radius $1$. Then, $\bx = \sqrt{d}\bz/\norm{\bz}$.

For $i\in [d]$, denote by $x_i$ and by $z_i$ the $i$-th component of $\bx$ and $\bz$, respectively. Then, for $i\neq j$,
\begin{equation*}
\mathbb E \left\{X_i X_j^*\right\} =d\cdot \mathbb E \left\{\frac{Z_i Z_j^*}{\norm{Z}^2}\right\}=0,
\end{equation*}
where the last equality holds by symmetry. Furthermore,
\begin{equation*}
\mathbb E \left\{|X_i|^2\right\} =d\cdot \mathbb E \left\{\frac{|Z_i|^2}{\norm{Z}^2}\right\}=1,
\end{equation*}
as $|Z_i|^2/\norm{Z}^2 \sim{\rm Beta}(1, d-1)$ by the argument of Lemma \ref{lemma:beta}. As a result, the thesis is readily proved. 
\end{proof}

\section{Proof of Lemma \ref{lemma:baiyaonopsd}}\label{app:baiyaonopsd}

Before presenting the proof of the lemma, let us introduce some basic definitions and well-known results. Let $H$ be a probability measure on $[0, +\infty)$.  Denote by $\Gamma_H$ the support of $H$ and by $\tau$ the supremum of $\Gamma_H$. Let $s_H(g)$ denote the Stieltjes transform of $H$, which is defined as
\begin{equation}
s_H(g) = \int\frac{1}{t-g}\,{\rm d}H(t),
\end{equation}
and let $g_H(s)$ denote its inverse.

Consider a matrix
\begin{equation}\label{eq:diffnorm}
\bS_n = \frac{1}{d}\bU \bM_n \bU^*,
\end{equation}
and assume that
\begin{enumerate}[(1)]
\item $\bM_n$ is PSD for all $n\in \mathbb N$;
\item $\bU\in \mathbb C^{d\times n}$ is a random matrix whose entries $\{u_{i, j}\}_{1\le i \le d, 1\le j \le n}$ are i.i.d. such that ${\mathbb E}\left\{U_{i, j}\right\}=0$, ${\mathbb E}\left\{|U_{i, j}|^2\right\}=1$, and ${\mathbb E}\left\{|U_{i, j}|^4\right\}<\infty$ (this includes the cases in which the entries are $\sim_{i.i.d.}\cnormal(0, 1)$ or are $\sim_{i.i.d.}\normal(0, 1)$);
\item The sequence of empirical spectral distributions of $\bM_n\in \mathbb C^{n\times n}$ converges weakly to a probability distribution $H$, as $n\to +\infty$;
\item $n/d\to \delta\in (0, +\infty)$, as $n\to \infty$;
\item The sequence of spectral norms of $\bM_n$ is bounded.
\end{enumerate}
Note that the normalization of \eqref{eq:diffnorm} differs from the normalization of \eqref{eq:deflemma} by a factor of $\delta$. However, since the form \eqref{eq:diffnorm} is more common in the literature, we will stick to it for the rest of this section. In order to obtain the desired result for the matrix \eqref{eq:deflemma}, it suffices to incorporate a factor $1/\delta$ in the definition of the function $\psi_{\delta}$. 

Let $F_{\delta, H}$ be the probability measure on $[0, +\infty)$ such that the inverse $g_{F_{\delta, H}}$ of its Stieltjes transform $s_{F_{\delta, H}}$ is given by 
\begin{equation}\label{eq:invst}
g_{F_{\delta, H}}(s) = -\frac{1}{s}+\delta\int \frac{t}{1+ts}\,{\rm d}H(t), \quad s\in \{z\in \mathbb C:\Im{(z)}>0\}.
\end{equation}
Then, the sequence of empirical spectral distributions of $\bS_n$ converges weakly to $F_{\delta, H}$ \cite{Marchenko}, \cite[Chapter 4]{BaiSilverstein}.

For $\alpha\not \in \Gamma_H$ and $\alpha\neq 0$, let us also define
\begin{equation}
\psi_{F_{\delta, H}}(\alpha) = g_{F_{\delta, H}}\left(-\frac{1}{\alpha}\right).
\end{equation}
The function $\psi_{F_{\delta, H}}$ links the support of $F_{\delta, H}$ with the support of the generating measure $H$ (see \cite[Section 4]{silsang1995} and \cite[Lemma 3.1]{baiyaospike2012}). In particular, if $\lambda\not \in \Gamma_{F_{\delta, H}}$, then $s_{F_{\delta, H}}(\lambda)\neq 0$ and $\alpha = -1/s_{F_{\delta, H}}(\lambda)$ satisfies
\begin{enumerate}[(1)]
\item $\alpha\not\in \Gamma_H$ and $\alpha\neq 0$ (so that $\psi_{F_{\delta, H}}(\alpha)$ is well-defined);
\item $\psi_{F_{\delta, H}}'(\alpha)>0$.
\end{enumerate}
Conversely, if $\alpha$ satisfies (1) and (2), then $\lambda = \psi_{F_{\delta, H}}(\alpha)\not\in \Gamma_{F_{\delta, H}}$.

Let $\lambda_1^{\bM_n}$ denote the largest eigenvalue of $\bM_n$ and assume that, as $n\to \infty$,
\begin{equation}
\begin{split}
\lambda_1^{\bM_n} &\stackrel{\mathclap{\mbox{\footnotesize a.s.}}}{\longrightarrow} \alpha_* \not \in \Gamma_H.\\
\end{split}
\end{equation}
Denote by $\lambda_1^{\bS_n}$ the largest eigenvalue of $\bS_n$. Then, the results in \cite{baiyaospike2012} prove that
\begin{equation}\label{eq:equivnew}
\begin{split}
&\lambda_1^{\bS_n} \stackrel{\mathclap{\mbox{\footnotesize a.s.}}}{\longrightarrow}\lambda_*=\psi_{F_{\delta, H}}(\alpha_*), \quad  \mbox{if } \psi_{F_{\delta, H}}'(\alpha_*) >0,\\
&\lambda_1^{\bS_n} \stackrel{\mathclap{\mbox{\footnotesize a.s.}}}{\longrightarrow}\min_{\alpha>\tau}\psi_{F_{\delta, H}}(\alpha), \quad \quad\mbox{if } \psi_{F_{\delta, H}}'(\alpha_*) \le 0.
\end{split}
\end{equation}
Informally, the eigenvalue $\lambda_1^{\bM_n}$ is mapped into the point $\psi_{F_{\delta, H}}(\alpha_*)$, where $\alpha_* = -1/s_{F_{\delta, H}}(\lambda_*)$. This point emerges from the support of $F_{\delta, H}$ if and only if $\psi'_{F_{\delta, H}}(\alpha_*)>0$.

In what follows, we relax the first hypothesis, i.e., we consider the case in which the matrix $\bM_n$ is not PSD. We will show that \eqref{eq:equivnew} still holds, which implies the claim of Lemma \ref{lemma:baiyaonopsd}.

\begin{proof}[Proof of Lemma \ref{lemma:baiyaonopsd}]
As $\bU$ is drawn from a rotationally invariant distribution, we can assume without loss of generality that $\bM_n$ is diagonal. Then, we have that 
\begin{equation}
\begin{split}
\bS_n &= \left(\bU_+, \bU_-\right)\left(\begin{array}{cc}
\bM_n^+ & \b0_k\\
\b0_{n-k} & -\bM_n^ -\\
\end{array}\right) \left(\begin{array}{l}
\bU_+^*\\
\bU_-^*\\
\end{array}\right)\\
&= \frac{1}{d}\bU_+ \bM_n^+ \bU_+^* - \frac{1}{d}\bU_- \bM_n^- \bU_-^*,\\
\end{split}
\end{equation}
where $\bM_n^+\in \mathbb R^{k \times k}$ is the diagonal matrix containing the positive eigenvalues of $\bM_n$, $\bM_n^-\in \mathbb R^{(n-k) \times (n-k)}$ is the diagonal matrix containing the negative eigenvalues of $\bM_n$ with the sign changed, $\bU_+$ contains the first $k$ columns of $\bU$, and $\bU_-$ contains the remaining $n-k$ columns of $\bU$.

Note that $\bU_+$ and $\bU_-$ are independent. Furthermore, if $\bH$ is a unitary matrix, then $\bU_-$ and $\bH\bU_-$ have the same distribution. Hence, we can rewrite the matrix $\bS_n$ as
\begin{equation}\label{eq:decomp}
\bS_n= \frac{1}{d}\bU_1 \bM_n^+ \bU_1^* - \frac{1}{d}\bH\bU_2 \bM_n^- \bU_2^*\bH^*,
\end{equation}
where $\bU_1$ and $\bU_2$ are independent with entries $\sim_{i.i.d.}\cnormal(0, 1)$, and $\bH$ is a random unitary matrix distributed according to the Haar measure.

Recall that, by hypothesis, the sequence of empirical spectral distributions of $\bM_n$ converges weakly to the probability distribution $H$, where $H$ is the law of the random variable $Z$. Then, the sequence of empirical spectral distributions of $\bM_n^+$ converges weakly to the probability distribution $H^+$, where $H^+$ is the law of $Z^+ = \max(Z, 0)$. Let $F_{\delta, H^+}$ be the probability measure on $[0, +\infty)$ such that the inverse $g_{F_{\delta, H^+}}$ of its Stieltjes transform $s_{F_{\delta, H^+}}$ is given by 
\begin{equation}\label{eq:gint1}
g_{F_{\delta, H^+}}(s) = -\frac{1}{s}+\delta\int \frac{t}{1+ts}\,{\rm d}H^+(t).
\end{equation} 
Define $\bS_n^+=\frac{1}{d}\bU_1 \bM_n^+ \bU_1^*$. Then, as $\bM_n^+$ is PSD, the sequence of empirical spectral distributions of $\bS_n^+$ converges weakly to $F_{\delta, H^+}$ \cite{Marchenko}, \cite[Chapter 4]{BaiSilverstein}.

Similarly, the sequence of empirical spectral distributions of $\bM_n^-$ converges weakly to the probability distribution $H^-$, where $H^-$ is the law of $Z^- = -\min(Z, 0)$. Let $F_{\delta, H^-}$ be the probability measure on $[0, +\infty)$ such that the inverse $g_{F_{\delta, H^-}}$ of its Stieltjes transform $s_{F_{\delta, H^-}}$ is given by
\begin{equation}\label{eq:gint2}
g_{F_{\delta, H^-}}(s) = -\frac{1}{s}+\delta\int \frac{t}{1+ts}\,{\rm d}H^-(t).
\end{equation} 
Define $\bS_n^-=\frac{1}{d}\bU_2 \bM_n^- \bU_2^*$. Then, as $\bM_n^-$ is PSD, the sequence of empirical spectral distributions of $\bS_n^-$ converges weakly to $F_{\delta, H^-}$ \cite{Marchenko}, \cite[Chapter 4]{BaiSilverstein}. Furthermore, the sequence of empirical spectral distributions of $-\bS_n^-$ converges weakly to the probability measure $F_{\delta, H^-_{\rm inv}}$ such that
\begin{equation}\label{eq:gint3}
g_{F_{\delta, H^-_{\rm inv}}}(s)=-g_{F_{\delta, H^-}}(-s),
\end{equation}
where $g_{F_{\delta, H^-_{\rm inv}}}$ denotes the inverse of the Stieltjes transform $s_{F_{\delta, H^-_{\rm inv}}}$ of $F_{\delta, H^-_{\rm inv}}$.

Define 
\begin{equation}\label{eq:freeconv}
F_{\delta, H} = F_{\delta, H^+} \boxplus F_{\delta, H^-_{\rm inv}},
\end{equation}
where $\boxplus$ denotes the free additive convolution. Recall the decomposition \eqref{eq:decomp}. Then, the sequence of empirical spectral distributions of $\bS_n$ converges weakly to $F_{\delta, H}$ \cite{voi91, spe93}. Consequently, the inverse $g_{F_{\delta, H}}$ of the Stieltjes transform $s_{F_{\delta, H}}$ of $F_{\delta, H}$ can be computed as
\begin{equation}\label{eq:compgF}
\begin{split}
g_{F_{\delta, H}}(s) &\stackrel{\mathclap{\mbox{\footnotesize(a)}}}{=} g_{F_{\delta, H^+} \boxplus F_{\delta, H^-_{\rm inv}}}(s) \\
&\stackrel{\mathclap{\mbox{\footnotesize(b)}}}{=} g_{F_{\delta, H^+}}(s) + g_{F_{\delta, H^-_{\rm inv}}}(s) + \frac{1}{s} \\
&\stackrel{\mathclap{\mbox{\footnotesize(c)}}}{=} -\frac{1}{s}+\delta\int \frac{t}{1+ts}\,{\rm d}H^+(t)-\delta\int \frac{t}{1-ts}\,{\rm d}H^-(t) \\
&\stackrel{\mathclap{\mbox{\footnotesize(d)}}}{=} -\frac{1}{s}+\delta\int \frac{t}{1+ts}\,{\rm d}H^+(t)+\delta\int \frac{t}{1+ts}\,{\rm d}H^-(-t) \\
&\stackrel{\mathclap{\mbox{\footnotesize(e)}}}{=} -\frac{1}{s}+\delta\int \frac{t}{1+ts}\,{\rm d}H(t), \\
\end{split}
\end{equation}
where in (a) we use \eqref{eq:freeconv}, in (b) we use that the $\mathcal R$-transform of the free convolution is the sum of the $\mathcal R$-transforms of the addends, in (c) we use \eqref{eq:gint1}, \eqref{eq:gint2}, and \eqref{eq:gint3}, in (d) we perform the change of variable $t\to -t$ in the second integral; and in (e) we use the fact that $H^+(t)$ is the law of $\max(Z, 0)$, $H^-(-t)$ is the law of $\min(Z, 0)$, and that $t/(1+ts)=0$ for $t=0$.

By hypothesis, $\lambda_1^{\bM_n} \stackrel{\mathclap{\mbox{\footnotesize a.s.}}}{\longrightarrow} \alpha_* \not \in \Gamma_H$. First, we establish under what condition the largest eigenvalue of $\bS_n^+$, call it $\lambda_1^{\bS_n^+}$, converges to a point outside the support of $F_{\delta, H^+}$. To do so, define $\psi_{F_{\delta, H^+}}(\alpha) = g_{F_{\delta, H^+}}(-1/\alpha)$. Then, $\lambda_1^{\bS_n^+}\stackrel{\mathclap{\mbox{\footnotesize a.s.}}}{\longrightarrow} \psi_{F_{\delta, H^+}}(\alpha_*)$, if $\psi'_{F_{\delta, H^+}}(\alpha_*)>0$; and $\lambda_1^{\bS_n^+}$ converges almost surely to a point inside the support of $F_{\delta, H^+}$, otherwise \cite{baiyaospike2012}.

For the moment, assume that $\psi'_{F_{\delta, H^+}}(\alpha_*)>0$. We now establish under what condition the largest eigenvalue of $\bS_n$, call it $\lambda_1^{\bS_n}$, converges to a point outside the support of $F_{\delta, H}$. To do so, let $\omega_1$ and $\omega_2$ denote the subordination functions corresponding to the free convolution $F_{\delta, H^+} \boxplus F_{\delta, H^-_{\rm inv}}$. These functions satisfy the following analytic subordination property:
\begin{equation}\label{eq:subfun}
s_{F_{\delta, H^+} \boxplus F_{\delta, H^-_{\rm inv}}}(z) = s_{F_{\delta, H^+}}(\omega_1(z)) =s_{F_{\delta, H^-_{\rm inv}}}(\omega_2(z)). 
\end{equation} 
Then, by Theorem 2.1 of \cite{belinschi2015}, we have that the spike $\psi_{F_{\delta, H^+}}(\alpha_*)$ is mapped into $\omega_1^{-1}(\psi_{F_{\delta, H^+}}(\alpha_*))$. The Stieltjes transform at this point is given by
\begin{equation*}
\begin{split}
s_{F_{\delta, H^+} \boxplus F_{\delta, H^-_{\rm inv}}}(\omega_1^{-1}(\psi_{F_{\delta, H^+}}(\alpha_*))) &\stackrel{\mathclap{\mbox{\footnotesize(a)}}}{=}  s_{F_{\delta, H^+}}(\psi_{F_{\delta, H^+}}(\alpha_*)) \\
&\stackrel{\mathclap{\mbox{\footnotesize(b)}}}{=}  s_{F_{\delta, H^+}}(g_{F_{\delta, H^+}}(-1/\alpha_*)) \\
&\stackrel{\mathclap{\mbox{\footnotesize(c)}}}{=}-1/\alpha_*,
\end{split}
\end{equation*}
where in (a) we use \eqref{eq:subfun}, in (b) we use the definition of $\psi_{F_{\delta, H^+}}$, and in (c) we use that $g_{F_{\delta, H^+}}$ is the functional inverse of the Stieltjes transform $s_{F_{\delta, H^+}}$. As a result, by \cite[Section 4]{silsang1995}, we conclude that $\omega_1^{-1}(\psi_{F_{\delta, H^+}}(\alpha_*))\not \in \Gamma_{F_{\delta, H}}$ if and only if $\psi'_{F_{\delta, H}}(\alpha_*)>0$. Furthermore, the condition $\psi'_{F_{\delta, H}}(\alpha_*)>0$ is more restrictive than the condition $\psi'_{F_{\delta, H^+}}(\alpha_*)>0$ since 
\begin{equation*}
\psi'_{F_{\delta, H^+}}(\alpha_*) = 1-\delta\int\left(\frac{t}{\alpha_*-t}\right)^2\,{\rm d}H^+ \ge 1-\delta\int\left(\frac{t}{\alpha_*-t}\right)^2\,{\rm d}H = \psi'_{F_{\delta, H}}(\alpha_*).
\end{equation*}
Hence, $\lambda_1^{\bS_n}$ converges to a point outside the support of $F_{\delta, H}$ if and only if $\psi'_{F_{\delta, H}}(\alpha_*)>0$ and the proof is complete. 
\end{proof}

\begin{remark}[Lemma \ref{lemma:baiyaonopsd} for the Real Case]\label{rmk:baiyaonopsd}
Consider the random matrix $\frac{1}{n}\bU \bM_n\bU^{\sT}$, where $\bU\in \mathbb R^{(d-1)\times n}$ is a random matrix whose entries are $\sim_{i.i.d.}\normal(0, 1)$ and $\bM_n \in \mathbb R^{n\times n}$. Then, the claim of Lemma \ref{lemma:baiyaonopsd} still holds. Let us briefly explain why this is the case. 

If $\bM_n$ is PSD, then the results of \cite{baiyaospike2012} allow us to conclude. If $\bM_n$ is not PSD, we can write an expression analogous to \eqref{eq:decomp}:
\begin{equation}
\frac{1}{d}\bU \bM_n \bU^{\sT}=\frac{1}{d}\bU_1 \bM_n^+ \bU_1^{\sT} - \frac{1}{d}\bH\bU_2 \bM_n^- \bU_2^{\sT}\bH^*,
\end{equation} 
where $\bM_n^+$ is the diagonal matrix containing the positive eigenvalues of $\bM_n$, $\bM_n^-$ is the diagonal matrix containing the negative eigenvalues of $\bM_n$ with the sign changed, $\bU_1$ and $\bU_2$ are independent with entries $\sim_{i.i.d.}\normal(0, 1)$, $\bH$ is a random unitary matrix distributed according to the Haar measure, and we have used the fact that the eigenvalues of $\bU_2 \bM_n^- \bU_2^{\sT}$ are the same as the eigenvalues of $\bH\bU_2 \bM_n^- \bU_2^{\sT}\bH^*$ since $\bH$ is unitary. Hence, the proof follows from the same argument of Lemma \ref{lemma:baiyaonopsd}.
\end{remark}

\section{Proof of Lemma \ref{lemma:SE}  and Theorem \ref{th:spectralAMPfail}} \label{app:AMPfail}

We start by proving a result similar to Lemma \ref{lemma:SE} for a general AMP iteration, where the function $f_t(\hat{z}; y)$ is generic.

\begin{lemma}[State Evolution for General AMP Iteration]\label{lemma:SEgen}
Let $\bx\in \mathbb R^d$ denote the unknown signal such that $\norm{\bx}_2 = \sqrt{d}$, $\bA = (\ba_1, \ldots, \ba_n)^{\sT}\in \mathbb R^{n\times d}$ with $\{\ba_i\}_{1\le i\le n}\sim_{i.i.d.}\normal(\bzero_d,\id_d/d)$, 
and $\by = (y_1, \ldots, y_n)$ with $y_i\sim p(\cdot \mid \langle \bx, \ba_i \rangle)$. 
Consider the AMP iterates $\bz^t, \hbz^t$ defined in \eqref{eq:GAMP} for some function $f_t(\hat{z}; y)$, with $\sb_t$ given by 
\begin{equation}
\sb_t = \delta \cdot \E\{ f'_t(\mu_t G_0+\tau_t G_1;Y)\}\, ,\label{eq:OnsDefnew}
\end{equation}
where the expectation is with respect to $G_0,G_1\sim_{i.i.d.}\normal(0,1)$ and $Y\sim p(\,\cdot\,\mid G_0)$.
Assume that the initialization
$\bz^0$ is independent of $\bA$ and that, almost surely,
\begin{align}
\lim_{n\to\infty}\frac{1}{d}\<\bx,\bz^0\> = \mu_0\, , \;\;\; \lim_{n\to\infty}\frac{1}{d}\|\bz^0\|^2 = \mu_0^2+\tau_0^2\, .
\end{align}
Let the state evolution recursion $\tau_t,\mu_t$ be defined as
\begin{equation}\label{eq:SEnew}
\begin{split}
\mu_{t+1} & =\delta \int_{\mathbb R} \E\big\{\partial_g p(y\mid X_0) f_t(\mu_t X_0+\tau_t G;y)\big\} \, \de y\, ,\\
 \tau_{t+1}^2& = \delta \cdot \E\Big\{\big(f_t(\mu_t X_0+\tau_t G;Y)\big)^2\Big\} \, ,
\end{split}
\end{equation}
with initialization $\mu_0$ and $\tau_0$, where the expectation is taken with respect to $X_0,G\sim_{i.i.d.}\normal(0,1)$. Then, for any $t$, and for any function
$\psi:\reals^2\to\reals$ such that $|\psi(\bu)-\psi(\bv)|\le L(1+\|\bu\|_2+\|\bv\|_2)\|\bu-\bv\|_2$ for some $L\in \mathbb R$, we have that, almost surely,
\begin{align}
\lim_{n\to\infty}\frac{1}{n}\sum_{i=1}^n\psi(x_i,z^t_i) = \E\left\{\psi(X_0, \mu_t X_0+ \tau_t G)\right\}\, .
\end{align}
\end{lemma}

\begin{proof}
For $g\in \reals$, let $\cH(\,\cdot\, ;g):[0,1]\to \reals\cup\{+\infty,-\infty\}$ be the generalized inverse of $$\cF(y\mid g)\equiv \int_{-\infty}^y p(y'\mid g)\,\de y',$$ namely, 
\begin{align}
\cH(w;g) \equiv \inf\big\{y\in\reals :\;\; \cF(y\mid g)\ge w\big\}\, .
\end{align}
With this definition, the model $y_i\sim p(\,\cdot\,\mid \<\ba_i,\bx\>)$ is equivalent to $y_i= \cH(w_i;\<\ba_i,\bx\>)$ for  $\{w_i\}_{1\le i\le n}\sim_{i.i.d.} \Unif([0,1])$
independent of $\bA$ and $\bx$. Let $\bw = (w_1,\dots,w_n)\in\reals^n$ and denote by $[\bv_1\mid \cdots\mid \bv_k]\in\reals^{m\times k}$ the matrix obtained by stacking column vectors $\bv_1,\dots,\bv_k\in\reals^m$.

For $t\ge 0$, define $\br^t = \bzero_d$, $\hbr^t = \bA\bx$, and
introduce the extended state variables $\bs^t \in\reals^{d\times 2}$ and   $\hbs^t \in\reals^{n\times 2}$, defined as
\begin{equation}
\begin{split}
\bs^t &= [\bz^t\mid \br^t],\\
\hbs^t &= [\hbz^t\mid \hbr^t].
\end{split}
\end{equation}
We further define the functions $h_t =[h_{t, 1}\mid h_{t, 2}]:\reals^{2}\times \reals\to \reals^{2}$ and $\hh_t =[\hh_{t, 1}\mid \hh_{t, 2}]:\reals^{2}\times \reals\to \reals^{2}$ by
setting
\begin{equation}
\begin{split}
h_t(s_1, s_2 ;x) &\equiv [s_1 \mid x\,]\, ,\\
\hh_t(\hat{s}_1, \hat{s}_2;w) & \equiv [f_t(\hat{s}_1;\cH(w;\hat{s}_2)) \mid  \, 0 \,]\, .
\end{split}
\end{equation}
With these notations, the iteration (\ref{eq:GAMP}) is equivalent to
\begin{equation}\label{eq:ExtendedAMP1}
\begin{split}
\bs^{t+1} & = \bA^{\sT} \hh_t(\hbs^t;\bw)- h_t(\bs^t;\bx)\hsB_t\, ,\\
\hbs^t & =\bA h_t(\bs^t;\bx)- \hh_{t-1}(\hbs^{t-1};\bw)\sB_{t-1}\, ,
\end{split}
\end{equation}
where the functions $h_t(\bs^t;\bx)$ and $\hh_t(\hbs^t;\bw)$ are understood to be applied component-wise to their arguments and $\sB_t,\hsB_t\in\reals^{2\times 2}$ are defined by
\begin{equation}
\begin{split}
(\hsB_t)_{j,k} &= \delta \cdot \E\left\{ \frac{\partial\hh_{t, k}}{\partial \hat{s}_j}(\mu_t X_0+\tau_t G, X_0;W)\right\}\, ,\\
(\sB_t)_{j,k} &= \delta \cdot \E\left\{ \frac{\partial h_{t, k}}{\partial s_j}(\mu_t X_0+\tau_t G, 0;X_0)\right\}\, .
\end{split}
\end{equation}
The iteration (\ref{eq:ExtendedAMP1}) satisfies the assumptions of \cite{javanmard2013state}[Proposition 5]. By applying that result, the claim follows.
\end{proof}

At this point, first we present the proof of Lemma \ref{lemma:SE} and then of Theorem  \ref{th:spectralAMPfail}.

\begin{proof}[Proof of Lemma  \ref{lemma:SE}]
Consider the state evolution recursion defined in \eqref{eq:SE} with initialization $\mu_0$. Let $f_t$ be defined as in \eqref{eq:GAMP_Fdef} with $\sF$ given by \eqref{eq:FGdef}. Suppose that, for any $t$, \eqref{eq:SEnew} holds with $\mu_t = \tau_t^2$. Then, by Lemma \ref{lemma:SEgen}, the claim immediately follows.

The remaining part of the proof is devoted to show that \eqref{eq:SEnew} holds with $\mu_t = \tau_t^2$, for $t\ge 0$. First, we prove by induction that $\mu_t = \tau_t^2$, for $t\ge 0$. The basis of the induction, i.e., $\mu_0 =\tau_0^2$, is true by the hypothesis of the Lemma. Now, we assume that $\mu_t = \tau_t^2$ and we show that $\mu_{t+1} = \tau_{t+1}^2$. Set
\begin{equation}\label{eq:defZ}
Z = \mu_t X_0 + \tau_t G,
\end{equation}
and note that $Z\sim\normal(0, \mu_t^2+\tau_t^2)$. Then,  we can  re-write $X_0$ as
\begin{equation*}
X_0 = a Z + b \widetilde{G},
\end{equation*}
for some $a, b\in \mathbb R$, where $\widetilde{G}\sim \normal(0, 1)$ and independent from $Z$. In order to compute the coefficients $a$ and $b$, we evaluate $\mathbb E\{X_0^2\}$ and $\mathbb E\{X_0\cdot Z\}$, thus obtaining the equations
\begin{equation*}
\begin{split}
a^2(\mu_t^2+\tau_t^2) +b^2 &= 1,\\
a(\mu_t^2+\tau_t^2) &=\mu_t, 
\end{split}
\end{equation*}
which can be simplified as
\begin{equation*}
\begin{split}
a &= \frac{\mu_t}{\mu_t^2+\tau_t^2},\\
b &= \frac{\tau_t}{\sqrt{\mu_t^2+\tau_t^2}}.\\
\end{split}
\end{equation*} 
Furthermore, by using the inductive hypothesis $\mu_{t}=\tau_t^2$ and that $q_t = \mu_t/(1+\mu_t)$, we obtain that
\begin{equation}\label{eq:defX0}
X_0 = (1-q_t) \, Z + \sqrt{1-q_t}\, \widetilde{G}.
\end{equation}
Hence, the following chain of equalities holds:
\begin{equation}\label{eq:taucomp}
\begin{split}
 \tau_{t+1}^2 &\stackrel{\mathclap{\mbox{\footnotesize(a)}}}{=} \delta \int_{\mathbb R} \E\Big\{ p(y\mid X_0)\cdot \big(f_t(\mu_t X_0+\tau_t G;y)\big)^2\Big\} \, \de y\, \\
&\stackrel{\mathclap{\mbox{\footnotesize(b)}}}{=}\delta \int_{\mathbb R} \E\Big\{ p\big(y \,  | \, (1-q_t) \, Z + \sqrt{1-q_t}\, \widetilde{G}\big) \cdot \big(f_t(Z;y)\big)^2\Big\} \, \de y\, \\
&\stackrel{\mathclap{\mbox{\footnotesize(c)}}}{=}\delta \int_{\mathbb R} \E\Big\{ \big(f_t(Z;y)\big)^2 \cdot  \E\big \{ p\big(y \,  | \, (1-q_t) \, Z + \sqrt{1-q_t}\, \widetilde{G}\big)\,\big | \, Z  \big\} \Big\} \, \de y\, \\
&\stackrel{\mathclap{\mbox{\footnotesize(d)}}}{=}\delta \bigintsss_{\mathbb R} \E\Bigg\{ \bigg(\frac{\mathbb E\{\partial_{g}p(y\mid (1-q_t)\, Z+\sqrt{1-q_t}\, \widetilde{G} )\mid Z\}}{\mathbb E\{p(y\mid (1-q_t)\, Z+\sqrt{1-q_t}\, \widetilde{G} )\mid Z\}} \bigg)^2\cdot  \E\big \{ p\big(y \,  | \, (1-q_t) \, Z + \sqrt{1-q_t}\, \widetilde{G}\big)\,\big | \, Z  \big\} \Bigg\} \, \de y\, \\
&=\delta \bigintsss_{\mathbb R} \E\Bigg\{ \frac{\mathbb E\{\partial_{g}p(y\mid (1-q_t)\, Z+\sqrt{1-q_t}\, \widetilde{G} )\mid Z\}}{\mathbb E\{p(y\mid (1-q_t)\, Z+\sqrt{1-q_t}\, \widetilde{G} )\mid Z\}} \cdot  \E\big \{ \partial_g p\big(y \,  | \, (1-q_t) \, Z + \sqrt{1-q_t}\, \widetilde{G}\big)\,\big | \, Z  \big\} \Bigg\} \, \de y\, \\
&\stackrel{\mathclap{\mbox{\footnotesize(e)}}}{=}\delta \int_{\mathbb R} \E\Big\{ f_t(Z;y) \cdot  \E\big \{ \partial_g p\big(y \,  | \, (1-q_t) \, Z + \sqrt{1-q_t}\, \widetilde{G}\big)\,\big | \, Z  \big\} \Big\} \, \de y\, \\
& \stackrel{\mathclap{\mbox{\footnotesize(f)}}}{=}\delta \int_{\mathbb R} \E\Big\{ f_t(\mu_t X_0+\tau_t G;y) \cdot\partial_g p(y\mid X_0) \Big\} \, \de y = \mu_{t+1},
\end{split}
\end{equation}
where in (a) we use that $Y\sim p(\cdot\mid X_0)$, in (b) we use \eqref{eq:defZ} and \eqref{eq:defX0}, in (c)
 we condition with respect to $Z$, in (d) we use the definition \eqref{eq:GAMP_Fdef} of $f_t$, in (e) we use again the definition \eqref{eq:GAMP_Fdef} of $f_t$, and in (f) we use again \eqref{eq:defZ} and \eqref{eq:defX0}. 
 
Finally, we prove that $\mu_{t+1}$ satisfies \eqref{eq:SEnew}. Indeed, the following chain of equalities holds: 
\begin{equation}
\begin{split}
\mu_{t+1} &\stackrel{\mathclap{\mbox{\footnotesize(a)}}}{=}\delta \bigintsss_{\mathbb R} \E\Bigg\{ \frac{\big(\mathbb E\{\partial_{g}p(y\mid (1-q_t)\, Z+\sqrt{1-q_t}\, \widetilde{G} )\mid Z\}\big)^2}{\mathbb E\{p(y\mid (1-q_t)\, Z+\sqrt{1-q_t}\, \widetilde{G} )\mid Z\}} \Bigg\} \, \de y\, \\
&\stackrel{\mathclap{\mbox{\footnotesize(b)}}}{=}\delta \bigintsss_{\mathbb R} \E\Bigg\{ \frac{\big(\mathbb E\{\partial_{g}p(y\mid (1-q_t)\, \sqrt{\mu_t^2+\tau_t^2}\,G_0+\sqrt{1-q_t}\, G_1 )\mid G_0\}\big)^2}{\mathbb E\{p(y\mid (1-q_t)\, \sqrt{\mu_t^2+\tau_t^2}\,G_0+\sqrt{1-q_t}\, G_1 )\mid G_0\}} \Bigg\} \, \de y\, \\
&\stackrel{\mathclap{\mbox{\footnotesize(c)}}}{=}\delta \bigintsss_{\mathbb R} \E\Bigg\{ \frac{\big(\mathbb E\{\partial_{g}p(y\mid \sqrt{q_t}\,G_0+\sqrt{1-q_t}\, G_1 )\mid G_0\}\big)^2}{\mathbb E\{p(y\mid \sqrt{q_t}\,G_0+\sqrt{1-q_t}\, G_1 )\mid G_0\}} \Bigg\} \, \de y\, =\delta\cdot h(q_t),\\
\end{split}
\end{equation}
where in (a) we use \eqref{eq:taucomp}, in (b) we set $G_1 = \widetilde{G}$ and $G_0=Z/\sqrt{\mu_t^2+\tau_t^2}$, and in (c) we use that $\mu_{t}=\tau_t^2$ and that $q_t = \mu_t/(1+\mu_t)$.
\end{proof}

\begin{proof}[Proof of Theorem  \ref{th:spectralAMPfail}]

In view of Lemma \ref{lemma:SE}, it is sufficient to show that $(q, \mu) = (0, 0)$ is an attractive fixed point of the recursion \eqref{eq:SE}. 

First of all, let us check that $(q, \mu) = (0, 0)$ is a fixed point. This happens if and only if
\begin{equation}
h(0) = \int_{\mathbb R} \frac{\big(\mathbb E_{G_1} \{\partial_{g} p(y\mid G_1)\}\big)^2}{\mathbb E_{G_1}\{p(y\mid G_1)\}}\, \de y = 0,
\end{equation}
which holds because of the condition \eqref{eq:condfixedpt}. 

Let us now prove that this fixed point is stable. We start by re-writing the function $h(q)$ defined in \eqref{eq:defhfun} as
\begin{equation}
h(q) = \int_{\mathbb R} \mathbb E_{G_0}\left\{\frac{\left(h_{\rm num}(\sqrt{q}, y)\right)^2}{h_{\rm den}(\sqrt{q}, y)}\right\}\, \de y\, ,
\end{equation} 
where
\begin{equation}
\begin{split}
h_{\rm num}(x, y) &= \mathbb E_{G_1} \left\{\partial_{g} p(y\mid x\cdot G_0+\sqrt{1-x^2}\, G_1)\right\}, \\
h_{\rm den}(x, y) &=\mathbb E_{G_1}\left\{p(y\mid x\cdot G_0+\sqrt{1-x^2}\, G_1)\right\}.
\end{split}
\end{equation}
Note that $h_{\rm num}(0, y)=0$ by assumption \eqref{eq:condfixedpt}. Then,
\begin{equation}\label{eq:extra1}
\begin{split}
h_{\rm num}(\sqrt{q}, y) &= \sqrt{q}\frac{\partial h_{\rm num}(x, y)}{\partial x}\bigg \rvert_{x=0} + \frac{q}{2}\frac{\partial^2 h_{\rm num}(x, y)}{\partial^2 x}\bigg \rvert_{x=x_1}, \\
h_{\rm den}(x, y) &=h_{\rm den}(0, y) + \sqrt{q}\frac{\partial h_{\rm den}(x, y)}{\partial x}\bigg \rvert_{x=x_2},
\end{split}
\end{equation}
for some $x_1, x_2 \in [0, \sqrt{q}]$. Furthermore, by applying Stein's lemma, we have that
\begin{equation}\label{eq:extra2}
h_{\rm num}(x, y) = \frac{1}{\sqrt{1-x^2}}\mathbb E_{G_1} \left\{G_1\cdot  p(y\mid x\cdot G_0+\sqrt{1-x^2}\, G_1)\right\}.
\end{equation}
By using \eqref{eq:extra2}, we can re-write \eqref{eq:extra1} as
\begin{equation*}
\begin{split}
h_{\rm num}(\sqrt{q}, y) &= \sqrt{q}\, G_0 \cdot \mathbb E_{G_1} \{G_1\cdot \partial_g p(y\mid G_1)\} + \frac{q}{2}\frac{1}{(1-x_1^2)^{5/2}} \mathbb E_{G_1}\left\{f_{\rm num}(G_0, G_1, x_1)\right\}, \\
h_{\rm den}(x, y) &=\mathbb E_{G_1}\{p(y\mid G_1)\}+\sqrt{q} \,\mathbb E_{G_1} \left\{f_{\rm den}(G_0, G_1, x_2)\right\},
\end{split}
\end{equation*}
where
\begin{equation}
\begin{split}
f_{\rm num} & (G_0, G_1, x_1) = G_1 \Biggl( (1+2x_1^2)\, p(y\mid x_1\cdot G_0+\sqrt{1-x_1^2}\, G_1)\\
&-\big(2G_0 \cdot x_1 (x_1^2-1)+G_1\sqrt{1-x_1^2}\,(1+2x_1^2)\big) \partial_g p(y\mid x_1\cdot G_0+\sqrt{1-x_1^2}\, G_1) \\
&+ (x_1^2-1)(G_0^2\,(x_1^2-1)-G_1^2 \,x_1^2+2G_0 G_1 x_1\sqrt{1-x_1^2})\partial^2_g p(y\mid x_1\cdot G_0+\sqrt{1-x_1^2}\, G_1)\Biggr),\\
f_{\rm den} & (G_0, G_1, x_2) =\left(G_0-\frac{x_2}{\sqrt{1-x_2^2}}G_1\right)\cdot \partial_{g} p(y\mid x_2\cdot G_0+\sqrt{1-x_2^2}\, G_1).
\end{split}
\end{equation}
By applying again Stein's lemma and by using that the conditional density $p(y\mid g)$ is bounded, we note that $\mathbb E_{G_1}\left\{f_{\rm num}(G_0, G_1, x_1)\right\}$ and $\mathbb E_{G_1} \left\{f_{\rm den}(G_0, G_1, x_2)\right\}$ are bounded. Hence, by dominated convergence, we obtain that
\begin{equation*}
h(q) = q\cdot \int_{\mathbb R} \frac{\big(\mathbb E_{G_1} \{\partial_g^2 p(y\mid G_1)\}\big)^2}{\mathbb E_{G_1}\{p(y\mid G_1)\}}\, \de y + o(q).
\end{equation*}
Therefore, in a neighborhood of the fixed point we have
\begin{equation}\label{eq:hcomp1}
\begin{split}
q_t & = \mu_t + o(\mu_t)\, ,\\
\mu_{t+1} & = \delta \cdot q_t\cdot \int_{\mathbb R} \frac{\big(\mathbb E_{G_1} \{\partial_g^2 p(y\mid G_1)\}\big)^2}{\mathbb E_{G_1}\{p(y\mid G_1)\}}\, \de y + o(q_t)\,.
\end{split}
\end{equation}
Furthermore, by applying twice Stein's lemma, we also have that
\begin{equation}\label{eq:hcomp2}
\mathbb E_{G_1} \{\partial_g^2 p(y\mid G_1)\} = \mathbb E_{G_1} \{p(y\mid G_1)(G_1^2-1)\}.
\end{equation}
By using \eqref{eq:hcomp1}, \eqref{eq:hcomp2} and by recalling the definition \eqref{eq:defdeltaur} of $\delta_{\rm u}$, we conclude that
\begin{equation}
\begin{split}
q_t & = \mu_t + o(\mu_t)\, ,\\
\mu_{t+1} &= \frac{\delta}{\delta_{\rm u}}\, q_t + o(q_t).
\end{split}
\end{equation}
As $\delta < \delta_{\rm u}$, the fixed point is stable. 
\end{proof}

\section{Proof of Lemma \ref{lemma:linGAMP} and Theorem \ref{th:spectralAMPsucc}} \label{app:AMPsucc}

For the proofs in this section, it is convenient to introduce the function
\begin{equation}\label{eq:Gdef}
\sG(x,y;\baq) = \frac{\mathbb E_G\{\partial_g^2 p(y\mid \baq\, x+\sqrt{\baq}G )\}}{\mathbb E_G\{ p(y\mid  \baq\, x+\sqrt{\baq}G )\}}
-\left(\frac{\mathbb E_G\{\partial_{g}p(y\mid \baq\, x+\sqrt{\baq}G )\}}{\mathbb E_G\{ p(y\mid  \baq\, x+\sqrt{\baq}G )\}}\right)^2
\, .
\end{equation}
First, we present the proof of Lemma \ref{lemma:linGAMP} and then of Theorem \ref{th:spectralAMPsucc}. 

\begin{proof}[Proof of Lemma \ref{lemma:linGAMP}]

The condition \eqref{eq:condfixedpt} implies that
\begin{equation}\label{eq:compF0}
\sF(0, y;1)=0.
\end{equation}
Furthermore, we have that
\begin{equation}
\begin{split}
\mathbb E_Y \{\sG(0, Y;1)\}&\stackrel{\mathclap{\mbox{\footnotesize(a)}}}{=} \mathbb E_Y \left\{\frac{\mathbb E_G\{\partial_g^2 p(Y\mid G)\}}{\mathbb E_G\{ p(Y\mid G)\}}\right\}\\
& \stackrel{\mathclap{\mbox{\footnotesize(b)}}}{=} \int_{\mathbb R} \mathbb E_G\{\partial_g^2 p(y\mid G)\}\, \de y\, \\
& = \mathbb E_G\left\{\partial_g^2\int_{\mathbb R} p(y\mid G)\, \de y\right\}=0,
\end{split}
\end{equation}
where in (a) we use \eqref{eq:compF0} and the definition \eqref{eq:Gdef} of $\sG(0, y;1)$, and in (b) we use the fact that $y$ has density $\mathbb E_G\{ p(y\mid G)\}$. 

Denote by $\sF'(x,y;\baq)$ the derivative of $\sF$ with respect to its first argument. Then, we have
\begin{align}
\sF'(x,y;\baq) = \baq \sG(x,y;\baq) \, . \label{eq:linformula1}
\end{align}
Hence, 
\begin{equation}\label{eq:linformula2}
\sb_t = \delta \cdot  (1-q_t)\cdot \E\{ \sG(\mu_t G_0+\sqrt{\mu_t} G_1, Y; 1-q_t)\} = \delta\cdot \E\{ \sG(0;Y;1)\}  + o_{q_t}(1) = o_{q_t}(1)\, .
\end{equation}
By using \eqref{eq:linformula1} and \eqref{eq:linformula2}, we linearize the recursion \eqref{eq:GAMP} around the fixed point  $\bz^t=\b0_d$ and  $\hbz^{t-1}=\b0_n$ as 
\begin{align}
\bz^{t+1} &= \bA^{\sT} \bJ \hbz^t +o_{q_t}(1) (\|\bz^t\|_2+\|\hbz^t\|_2) +o (\|\hbz^t\|_2)\, ,\label{eq:GAMPlin1}\\
\hbz^t  &= \bA\bz^t- \bJ \hbz^{t-1} +o_{q_t}(1) \, \|\hbz^{t-1}\|_2+o (\|\hbz^{t-1}\|_2)\, , \label{eq:GAMPlin2}
\end{align}
where $\bJ\in \mathbb R^{n\times n}$ is a diagonal matrix with entries $j_{i} =\sF'(0,y_i;1)$ for $i\in [n]$. By substituting the expression \eqref{eq:GAMPlin2} for $\hbz^t$ into the RHS of \eqref{eq:GAMPlin1}, the result
 follows. 

\end{proof}

\begin{proof}[Proof of Theorem \ref{th:spectralAMPsucc}]
By definition, $\alpha$ is an eigenvalue of $\bL_n$ if and only if
\begin{equation}\label{eq:deteq}
{\rm det}(\bL_n -\alpha \bI_{n+d})=0.
\end{equation}
Recall that, when $\bD$ is invertible, 
\begin{equation}
{\rm det}\left(\begin{array}{cc}
\bA & \bB \\
\bC & \bD \\
\end{array}\right) = {\rm det}(\bD) \cdot {\rm det}(\bA - \bB \bD^{-1}\bC).
\end{equation}
Then, after some calculations, we obtain that \eqref{eq:deteq} is equivalent to 
\begin{equation}\label{eq:deteq3}
\alpha^d \cdot {\rm det}(-\bJ-\alpha \bI_n)\cdot {\rm det}(\bI_d-\bA^{\sT} ( \bI_n +\alpha \bJ^{-1})^{-1} \bA)=0.
\end{equation}

From \eqref{eq:deteq3}, we immediately deduce that the eigenvalues of $\bL_n$ are real if and only if all the solutions to 
\begin{equation}\label{eq:deteq2}
{\rm det}(\bI_d-\bA^{\sT} ( \bI_n +\alpha \bJ^{-1})^{-1} \bA)=0
\end{equation}
are real. We will prove that in fact this equation does not have any solution for $\alpha\in\complex\setminus \reals$.

 Let $\bU \bSigma \bV^{\sT}$ be the SVD of $\bA$. Then, \eqref{eq:deteq2} is equivalent to 
\begin{equation*}
{\rm det}(\bU\bSigma^{-2}-( \bI_n +\alpha \bJ^{-1})^{-1} \bU)=0.
\end{equation*}
Using the  fact that $\det(\bSigma)\neq 0$, and $\det( \bI_n +\alpha \bJ^{-1})\neq 0$ for  $\alpha\in\complex\setminus \reals$,
Eq.~\eqref{eq:deteq2} is equivalent to 
\begin{equation*}
\det( (\bI_n +\alpha \bJ^{-1})\bU- \bU\bSigma^2)=0\, ,
\end{equation*}
or equivalently
\begin{equation*}
\det( \bI_n +\alpha \bJ^{-1}- \bA\bA^{\sT})=\det( \bI_n +\alpha \bJ^{-1}- \bU\bSigma^2\bU^{\sT})=0\, .
\end{equation*}
Given that the solutions of this equations are generalized eigenvalues for
the pairs of symmetric matrices $\bA\bA^{\sT}-\bI_n$ and $\bJ^{-1}$, they must be real.
We conclude that the eigenvalues of $\bL_n$ are real. 

Note that 
\begin{equation}\label{eq:JGrel}
\begin{split}
\sG(0,y;1)&\stackrel{\mathclap{\mbox{\footnotesize(a)}}}{=} \frac{\mathbb E_G\{\partial_g^2 p(y\mid G)\}}{\mathbb E_G\{ p(y\mid G)\}}\\
& \stackrel{\mathclap{\mbox{\footnotesize(b)}}}{=} \frac{\mathbb E_G\{p(y\mid G)(G^2-1)\}}{\mathbb E_G\{ p(y\mid G)\}} \\
& \stackrel{\mathclap{\mbox{\footnotesize(c)}}}{=} \frac{\mathcal T^*(y)}{1-\mathcal T^*(y)},
\end{split}
\end{equation}
where in (a) we use that $\sF(0, 1; y)=0$ as \eqref{eq:condfixedpt} holds, in (b) we apply twice Stein's lemma, and in (c) we use the definition \eqref{eq:deftyr} of $\mathcal T^*$. Then, \eqref{eq:deteq2} can be re-written as 
\begin{equation}\label{eq:detequiv}
{\rm det}\left(\bI_d - \sum_{i=1}^n \frac{\mathcal T^*(y_i)}{\mathcal T^*(y_i) + \alpha (1-\mathcal T^*(y_i))} \ba_i \ba_i^{\sT}\right)=0.
\end{equation}

Let $\lambda_1^{\bD^*_n}(\alpha)$ be the largest eigenvalue of the matrix $\bD^*_n(\alpha)$ defined as
\begin{equation}
\bD^*_n(\alpha) = \sum_{i=1}^n \frac{\mathcal T^*(y_i)}{\mathcal T^*(y_i) + \alpha (1-\mathcal T^*(y_i))} \ba_i \ba_i^{\sT}.
\end{equation}
Note that, as $\alpha\to +\infty$, the entries of $\bD^*_n(\alpha)$ tend to $0$ with high probability. Since the eigenvalues of a matrix are continuous functions of the elements of the matrix, we also obtain that 
\begin{equation*}
\lim_{\alpha\to+\infty}\lambda_1^{\bD^*_n}(\alpha) = 0.
\end{equation*}
Hence, if there exists $\bar{\alpha}>1$ such that $\lambda_1^{\bD^*_n}(\bar{\alpha})>1$, 
then there exists also $\bar{\alpha}_0> \bar{\alpha} > 1$ such that $\lambda_1^{\bD^*_n}(\bar{\alpha}_0)=1$. Consequently, there exists $\alpha > 1$ that satisfies \eqref{eq:detequiv}, which implies the result of the theorem.

The rest of the proof consists in showing that $\bar{\alpha}= \sqrt{\delta/\delta_{\rm u}}$ satisfies the desired requirements. First of all, note that $\sqrt{\delta/\delta_{\rm u}}>1$, as $\delta>\delta_{\rm u}$. Furthermore, we have that 
\begin{equation}
\bD^*_n(\bar{\alpha}) = \sum_{i=1}^n \mathcal T_{\delta}^*(y_i) \ba_i \ba_i^{\sT},
\end{equation}
where $\mathcal T_{\delta}^*$ is defined in \eqref{eq:deftydeltar}. Recall that, by hypothesis, $\bx$ is such that $\norm{\bx}_2 = \sqrt{d}$ and $\{\ba_i\}_{1\le i\le n}\sim_{i.i.d.}\normal({\bm 0}_d,\id_d/d)$. Let $\tilde{\bx} = \bx/\sqrt{d}$ and $\tilde{\ba_i} =\sqrt{d}\cdot \ba_i$. Then, $\langle \bx, \ba_i \rangle = \langle \tilde{\bx}, \tilde{\ba}_i \rangle$. Let $\lambda_1^{\widetilde{\bD}_n}$ be the largest eigenvalue of the matrix $\widetilde{\bD}_n$ defined as
\begin{equation}
\widetilde{\bD}_n = \frac{1}{n}\sum_{i=1}^n \mathcal T_\delta^*(y_i) \tilde{\ba}_i \tilde{\ba}_i^{\sT}.
\end{equation} 
Since $\widetilde{\bD}_n = \bD^*_n(\bar{\alpha})/\delta$, it remains to prove that 
\begin{equation}\label{eq:limiteig}
\lambda_1^{\widetilde{\bD}_n} \stackrel{\mathclap{\mbox{\footnotesize a.s.}}}{\longrightarrow}\tilde{\lambda} > \frac{1}{\delta}.
\end{equation}

To do so, we apply a result analogous to that of Lemma \ref{lemma:condub} for the real case with $\mathcal T=\mathcal T_\delta^*$. For the moment, assume that $\mathcal T_\delta^*$ fulfills the hypotheses of Lemma \ref{lemma:condub} (we will prove later that this is the case). Then, $\lambda_1^{\widetilde{\bD}_n}$ converges almost surely to $\zeta_{\delta}(\lambda_{\delta}^*)$.

Recall that
\begin{equation*}
\zeta_{\delta}(\lambda) = \psi_{\delta}(\max(\lambda, \bar{\lambda}_\delta)),
\end{equation*}
where $\bar{\lambda}_\delta$ is the point of minimum of the convex function $\psi_{\delta}(\lambda)$ defined as
\begin{equation*}
\psi_{\delta}(\lambda) = \lambda\left(\frac{1}{\delta}+{\mathbb E}\left\{\frac{\mathcal T_\delta^*(Y)}{\lambda-\mathcal T_\delta^*(Y)}\right\}\right).
\end{equation*}
Notice also that this minimum is the unique local minimizer since $\psi_{\delta}$ is convex and analytic.

Furthermore, $\lambda_{\delta}^*$ is the unique solution to the equation $\zeta_{\delta}(\lambda_{\delta}^*) = \phi(\lambda_{\delta}^*)$, where $\phi(\lambda)$ is defined as
\begin{equation*}
\phi(\lambda) = \lambda \cdot {\mathbb E}\left\{\frac{\mathcal T_\delta^*(Y)\cdot G^2}{\lambda-\mathcal T_\delta^*(Y)}\right\}.
\end{equation*}

By setting the derivative of $\psi_{\delta}(\lambda)$ to $0$, we have that
\begin{equation*}
{\mathbb E}\left\{\frac{(\mathcal T_\delta^*(Y))^2}{(\bar{\lambda}_\delta-\mathcal T_\delta^*(Y))^2}\right\}=\frac{1}{\delta}.
\end{equation*}
By using the definition \eqref{eq:deftydeltar} of $\mathcal T^*_\delta$ and the definition \eqref{eq:deftyr} of $\mathcal T^*$, we verify that
\begin{equation}\label{eq:tratio}
\frac{\mathcal T_\delta^*(Y)}{1-\mathcal T_\delta^*(Y)} =\sqrt{ \frac{\delta_{\rm u}}{\delta}}\, \frac{\mathbb E_G\{p(y\mid G)(G^2-1)\}}{\mathbb E_G\{ p(y\mid G)\}}.
\end{equation}
Hence, by using the definition \eqref{eq:defdeltaur} of $\delta_{\rm u}$, we obtain that
\begin{equation*}
{\mathbb E}\left\{\frac{(\mathcal T_\delta^*(Y))^2}{(1-\mathcal T_\delta^*(Y))^2}\right\}= \frac{\delta_{\rm u}}{\delta}\int_{\mathbb R} \frac{\left(\mathbb E_G\{p(y\mid G)(G^2-1)\}\right)^2}{\mathbb E_G\{ p(y\mid G)\}}\,{\rm d}y = \frac{1}{\delta},
\end{equation*}
which immediately implies that
\begin{equation}\label{eq:calcTstardelta1}
\bar{\lambda}_\delta = 1.
\end{equation}
By using \eqref{eq:tratio}, one also obtains that
\begin{equation*}
{\mathbb E}\left\{\frac{\mathcal T_\delta^*(Y)}{1-\mathcal T_\delta^*(Y)}\right\}=\sqrt{\frac{\delta_{\rm u}}{\delta}}\int_{\mathbb R}\mathbb E_G\{p(y\mid G)(G^2-1)\}\,{\rm d}y=\sqrt{\frac{\delta_{\rm u}}{\delta}}\,\, \mathbb E_G\{G^2-1\}=0,
\end{equation*}
which implies that
\begin{equation}\label{eq:calcTstardelta2}
\psi_{\delta}(1) = \frac{1}{\delta}.
\end{equation}
Furthermore, we have that
\begin{equation*}
{\mathbb E}\left\{\frac{\mathcal T_\delta^*(Y)(G^2-1)}{1-\mathcal T_\delta^*(Y)}\right\} = \sqrt{\frac{\delta_{\rm u}}{\delta}}\int_{\mathbb R} \frac{\left(\mathbb E_G\{p(y\mid G)(G^2-1)\}\right)^2}{\mathbb E_G\{ p(y\mid G)\}}\,{\rm d}y = \frac{1}{\sqrt{\delta\cdot \delta_{\rm u}}} >\frac{1}{\delta},
\end{equation*}
which implies that
\begin{equation}\label{eq:calcTstardelta3}
\phi(1) = \frac{1}{\sqrt{\delta\cdot \delta_{\rm u}}} > \frac{1}{\delta},
\end{equation}
as $\delta > \delta_{\rm u}$. By putting \eqref{eq:calcTstardelta1}, \eqref{eq:calcTstardelta2}, and \eqref{eq:calcTstardelta3} together, we obtain that
\begin{equation}\label{eq:compphizeta}
\phi(\bar{\lambda}_\delta)>\zeta_\delta(\bar{\lambda}_\delta).
\end{equation} 
Recall that $\zeta_\delta(\lambda)$ is monotone non-decreasing and $\phi(\lambda)$ is monotone non-increasing. Consequently, \eqref{eq:compphizeta} implies that $\lambda_{\delta}^* > \bar{\lambda}_\delta$. Thus, we conclude that
\begin{equation}\label{eq:philambdadelta}
\begin{split}
\lim_{n\to\infty} \lambda_1^{\widetilde{\bD}_n} & = \zeta_{\delta}(\lambda^*_{\delta}) =\psi_{\delta}(\lambda_{\delta}^*) \\
&>\psi_{\delta}(\bar{\lambda}_{\delta}) =\psi_{\delta}(1)=\frac{1}{\delta}\, .
\end{split}
\end{equation}

Now, we show that $\mathcal T_\delta^*$ fulfills the hypotheses of Lemma \ref{lemma:condub} by using arguments similar to those at the end of the proof of Theorem \ref{th:upper}. First of all, since $\mathcal T^*(y)\le 1$, we have that $\mathcal T_\delta^*(y)$ is bounded. Furthermore, if $\mathcal T_\delta^*(y)$ is equal to the constant value $0$, then $\delta_{\rm u}=\infty$ and the claim of Theorem \ref{th:spectralAMPsucc} trivially holds. Hence, we can assume that $\mathbb P(\mathcal T_\delta^*(Y)=0)<1$. Let $\tau$ be the supremum of the support of $\mathcal T_\delta^*(Y)$. If $\mathbb P(\mathcal T_\delta^*(Y)=\tau)>0$, then the condition \eqref{eq:hplemmaub1} is satisfied and the proof is complete. Otherwise, for any $\epsilon_1 >0$, there exists $\Delta_1(\epsilon_1)$ such that Eq.~\eqref{eq:condDelta1} holds. Define $\mathcal T_{\delta}^*(y, \epsilon_1)$ as in \eqref{eq:defTyfin}. Clearly, the random variable $\mathcal T_{\delta}^*(Y, \epsilon_1)$ has a point mass at $\delta$, hence the condition \eqref{eq:hplemmaub1} is satisfied. As a final step, we show that we can take $\epsilon_1 \downarrow 0$. Define
\begin{equation*}
\widetilde{\bD}_n(\epsilon_1) = \frac{1}{n}\sum_{i=1}^n \mathcal T_{\delta}^*(y_i, \epsilon_1) \tilde{\ba}_i \tilde{\ba}_i^*.
\end{equation*}
Then,
\begin{equation}\label{eq:boundopnorm}
\norm{\widetilde{\bD}_n(\epsilon_1)-\widetilde{\bD}_n}_{\rm op} \le C_1 \cdot \Delta_1(\epsilon_1),
\end{equation}
where the constant $C_1$ depends only on $n/d$. Consequently, by using \eqref{eq:boundopnorm} and Weyl's inequality, we conclude that 
\begin{equation}
|\lambda_1^{\widetilde{\bD}_n(\epsilon_1)}-\lambda_1^{\widetilde{\bD}_n}|\le C_1  \cdot \Delta_1(\epsilon_1).
\end{equation}
Hence, for any $n$, as $\epsilon_1$ tends to $0$, the largest eigenvalue of $\widetilde{\bD}_n(\epsilon_1)$ tends to the largest eigenvalue of $\widetilde{\bD}_n$, which concludes the proof.

\end{proof}

\bibliographystyle{amsalpha}
\bibliography{all-bibliography}

\end{document}